\documentclass{article}

\usepackage{microtype}
\usepackage{graphicx}
\usepackage{subfigure}
\usepackage{booktabs} 
\usepackage{cite}
\usepackage{hyperref}

\usepackage{xcolor}
\usepackage[accepted]{icml2024}

\usepackage{amsmath}
\usepackage{amssymb}
\usepackage{mathtools}
\usepackage{amsthm}

\usepackage{sidecap,multirow,helvet,booktabs,courier,bm,cite}
\usepackage[capitalize,noabbrev]{cleveref}
\usepackage{tikz}
\usepackage{comment}
\usepackage{amsmath,amssymb} 
\usepackage{color}
\usepackage{diagbox}
\usepackage{booktabs}  
\usepackage{amsmath}  
\usepackage{multirow,helvet,booktabs,courier,bm,cite}
\usepackage{amssymb}
\usepackage{diagbox}
\usepackage{courier}  
\usepackage{caption} 
\usepackage{colortbl}  

\theoremstyle{plain}
\newtheorem{theorem}{Theorem}[section]
\newtheorem{proposition}[theorem]{Proposition}

\theoremstyle{definition}

\theoremstyle{remark}

\usepackage[textsize=tiny]{todonotes}
\icmltitlerunning{Residual-Conditioned Optimal Transport}
\begin{document}
	\twocolumn[
	\icmltitle{Residual-Conditioned Optimal Transport: Towards\\ Structure-Preserving Unpaired and Paired Image Restoration}
	
	
	\icmlsetsymbol{equal}{*}
	\begin{icmlauthorlist}
		\icmlauthor{Xiaole Tang}{xjtu}
		\icmlauthor{Xin Hu}{xjtu}
		\icmlauthor{Xiang Gu}{xjtu}
		\icmlauthor{Jian Sun}{xjtu}
	\end{icmlauthorlist}
	
	\icmlaffiliation{xjtu}{School of Mathematics and Statistics, Xi’an Jiaotong University, Xi’an, China}
	
	\icmlcorrespondingauthor{Jian Sun}{jiansun@xjtu.edu.cn}

	\icmlkeywords{Machine Learning, ICML}
	
	\vskip 0.3in
	]
	
	
	
	\printAffiliationsAndNotice{} 

\begin{abstract}
		Deep learning-based image restoration methods generally struggle with faithfully preserving the structures of the original image.  In this work, we propose a novel Residual-Conditioned Optimal Transport (RCOT) approach, which models image restoration as an optimal transport (OT) problem for both unpaired and paired settings,  introducing the transport residual as a unique degradation-specific cue for both the transport cost and the transport map. Specifically, we first formalize a Fourier residual-guided OT objective by incorporating the degradation-specific information of the residual into the transport cost. We further design the transport map as a two-pass RCOT map that comprises a base model and a refinement process, in which the transport residual is computed by the base model in the first pass and then encoded as a degradation-specific embedding to condition the second-pass restoration. By duality, the RCOT problem is transformed into a minimax optimization problem, which can be solved by adversarially training neural networks. Extensive experiments on multiple restoration tasks show that RCOT achieves competitive performance in terms of both distortion measures and perceptual quality, restoring images with more faithful structures as compared with state-of-the-art methods.
\end{abstract}
\section{Introduction}
\begin{figure}[!t]
	\centering
	\includegraphics[width=1\linewidth]{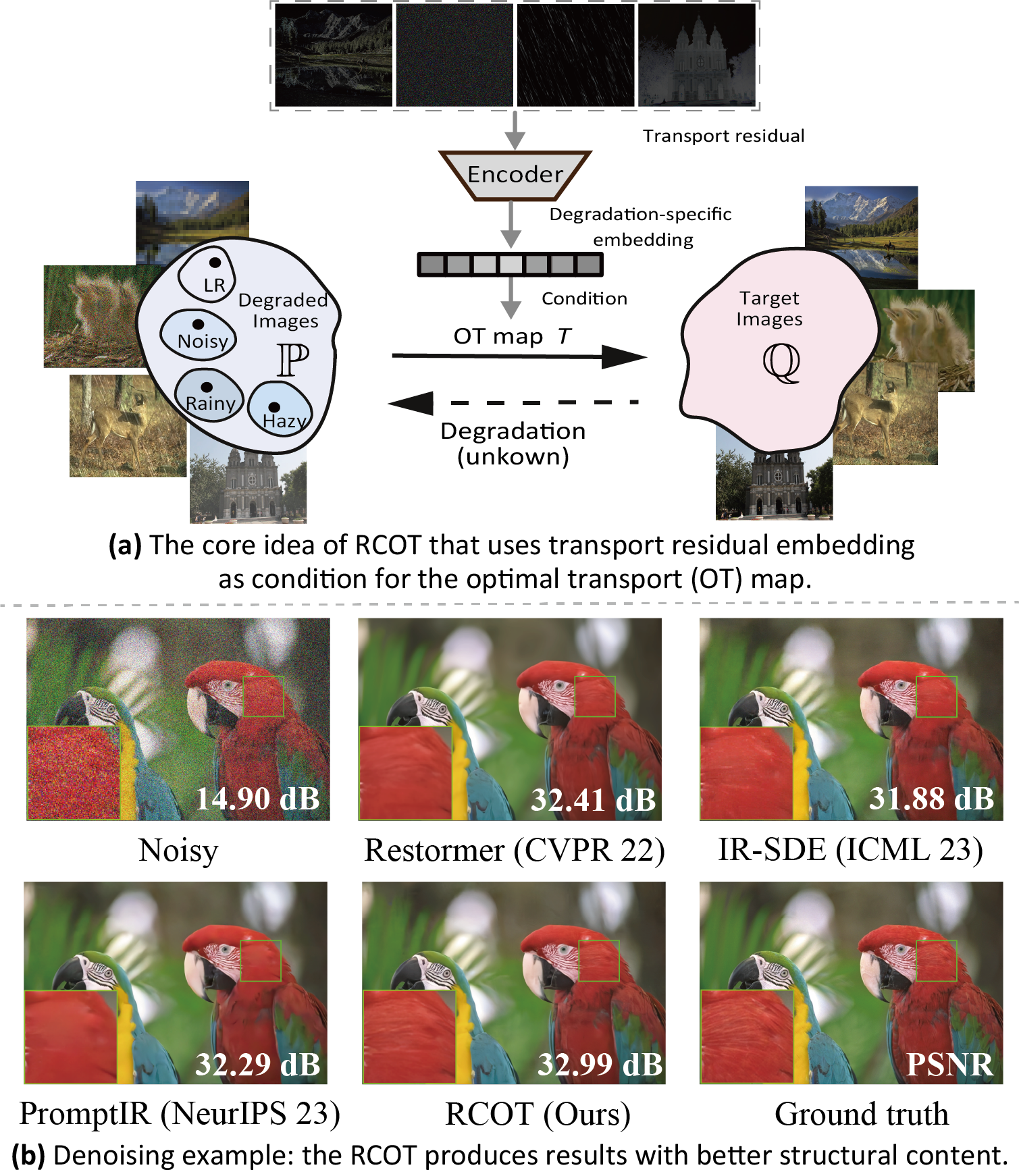}
	\vspace{-0.5cm}
	\caption{(a) The core idea of the RCOT framework that uses the transport residual embedding  (representation of some degradation-specific knowledge) as a condition for the OT map.  (b) A denoising demo under noise level $\sigma=50$. RCOT produces a noise-free image with better structures.}
	\label{demo}
	\vspace{-0.3cm}
\end{figure}
Image restoration is a fundamental low-level task of removing the degradation (e.g., noise, down-scaling, rain, haze, blur, etc.) from a degraded image. Traditional methods \cite{he2010single, sun2008image, zheng2023unsupervised, tang2022uncertainty, tang2023uncertainty} focus on designing an optimization problem that exploits suitable priors of the clear image. 
Recently, significant advancements in image restoration have been witnessed. They are primarily driven by the sophisticated network architectures  \cite{liang2021swinir, Zamir2021MPRNet, Zamir2021Restormer, wang2022uformer, zhou2023fourmer, wang2023promptrestorer, chen2023learning, ren2016single, cui2023irnext} and efficient generative models, especially Generative Adversarial Networks (GANs) \cite{li2018single, ledig2017photo, zhao2020large, pan2020physics, pan2021exploiting} and diffusion probabilistic models (DPMs) \cite{kawar2022denoising, luo2023image, zhu2023denoising, murata2023gibbsddrm,saharia2022image,saharia2022palette,choi2023restoration}. The mainstream of the generation-based methods focuses on how to effectively condition the generator on degraded images for high-fidelity restoration. 

Classic restoration methods \cite{Zamir2021MPRNet, Zamir2021Restormer,potlapalli2023promptir} utilize sophisticated networks to fit regression models that yield deterministic results, by minimizing distortion measures (e.g., MSE, SSIM) w.r.t. the ground-truth. However, due to the ill-posedness of image restoration, they often capture the ``mean'' of the posterior distribution of the high-quality data given the degraded data, sampling results with excessive smoothness and compromised structural details, which may deviate from human perception (e.g., Restormer \cite{Zamir2021Restormer}, PromptIR \cite{potlapalli2023promptir} in Figure \ref{demo} (b)). Differently, generation-based methods \cite{saharia2022image,gao2023implicit,luo2023image} regard the restoration task as a conditional distribution modeling problem, which generally results in visually appealing results. However, they often use the degraded image as a conditional input without including specific information about the degradation, which can result in outputs with remaining distortions and inaccurate structural details (e.g., IR-SDE \cite{luo2023image} in Figure \ref{demo} (b)). Therefore, how to faithfully preserve the structures while minimizing the distortion remains challenging.

  To tackle the challenge, we propose to model image restoration as an optimal transport (OT) problem, in which we introduce a novel transport residual as a degradation-specific cue for both the transport cost and transport map. Specifically, we present a two-pass Residual-Conditioned Optimal Transport (RCOT) approach to realize a degradation-aware and structure-preserving OT map, which applies to both unpaired and paired data settings. The key idea is to incorporate the degradation-specific knowledge (from the residual or its embedding) into the transport cost, and more importantly, into the transport map via a two-pass process, in which the transport residual is computed by the base model in the first pass and then encoded as a degradation-specific embedding to condition the second-pass restoration.  This conditioning mechanism enables the transport map to adjust its behaviors for multiple restoration tasks and restore images with better structural content (Figure \ref{demo} (b)). 

In summary, our contributions mainly include:
\begin{itemize}
	\vspace{-0.2cm}
    \item  We model image restoration as an OT problem, in which we introduce a Fourier residual-guided OT objective, allowing us to incorporate degradation-specific knowledge into the transport cost.  We further deduce a minimax dual formulation for the OT model.
 
		\vspace{-0.2cm}
   \item We propose a two-pass RCOT approach, which conditions the transport map on the residual embedding. This conditioning mechanism dynamically injects degradation-specific from the residual embedding into the restoration operator, i.e., the RCOT map, enhancing its capability to preserve the image structure.
   

  \item Extensive experiments on multiple tasks, e.g., image denoising, super-resolution,  deraining,  and dehazing on both synthetic and real-world datasets show the effectiveness of our method in terms of both distortion measures and perceptual quality. In particular, our method restores images with more faithful structural details, compared with existing approaches.
\end{itemize}
\begin{figure*}[!t]
	\centering
	\includegraphics[scale=1.9]{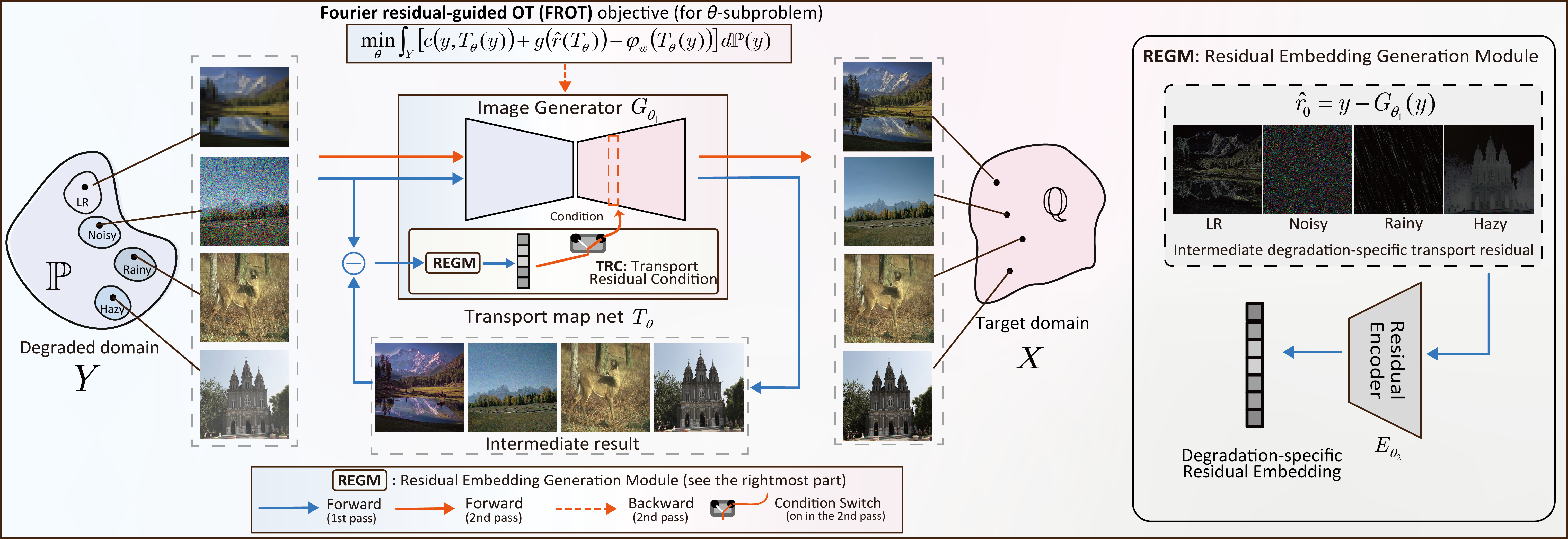}
	\caption{\textbf{Overview of the two-pass RCOT framework for structure-preserving restoration.}   RCOT integrates the transport residual into the transport cost, leading to a FROT objective; and more crucially, into the transport map via a two-pass conditioning process. The first pass unconditionally generates an intermediate result along with the estimated residual $\hat r_0$. The second pass restores the refined result conditioned on the residual embedding $E_{\theta_2}(\hat r_0)$. }
	\label{model}
	\vspace{-0.5cm}
\end{figure*}
\section{Background}
\textbf{Image restoration.}  Classic state-of-the-art restoration methods focus on minimizing distortion (e.g., MSE). They are primarily driven by efficient network architectures \cite{liang2021swinir, Zamir2021MPRNet, Zamir2021Restormer, wang2022uformer, zhou2023fourmer} and model the restoration as regression models that produce deterministic results. However, they rely on large amounts of paired training data and are prone to ``mean'' results with compromised perceptual quality.

Differently, deep generative models have made much progress that can generate perceptually realistic images. The early mainstream methods are built on the conditional generative adversarial network (CGAN) \cite{mirza2014conditional}, in which the restoration problem is treated as a conditional generation problem. RoCGAN \cite{chrysos2020rocgan} utilizes high-quality data for supervised training of restoration from degraded inputs. AmbientGAN \cite{bora2018ambientgan} generates clean images from noisy input, assuming the degradation satisfies certain conditions. Recently, a promising avenue is the adoption of diffusion models. \citet{saharia2022image} utilize the degraded images as conditions to train diffusion models. \citet{kawar2021snips,kawar2022denoising} and \citet{chung2023diffusion} operate under the assumption that the degradation and its parameters are known at test time. \citet{gao2023implicit} propose a scale-adaptive condition on the LR image for high-fidelity super-resolution. \citet{luo2023image} propose a maximum likelihood-based loss function to train a mean-reverting score-based model. However, these methods more or less omit the degradation-specific knowledge, which limits their ability to produce faithful structures while minimizing the distortion.


Intuitively, the generation-based methods are seeking an efficient map $T$ to transform the distribution of the low-quality image distribution $\mathbb P$  into the high-quality image distribution $\mathbb Q$. The key issue lies in the ambiguity that there may exist infinite maps satisfying this constraint. Prior knowledge is required to determine which is the optimal map. In this paper, we model image restoration from an OT perspective, seeking the most efficient map that transports $\mathbb P$ to $\mathbb Q$ with a minimal transport cost.

\textbf{Optimal transport and its applications in restoration.}
OT problem seeks the optimal \textit{transport map} (Monge Problem \cite{monge1781memoire}, a.k.a., MP) or \textit{transport plan} (Kantorovich problem \cite{kantorovich1942translocation}, a.k.a., KP) to transform a distribution $\mathbb P\in \mathcal P(Y)$ to another distribution $\mathbb Q\in \mathcal P(X)$  with the minimal transport cost. $\mathcal P(X)$  and $\mathcal P(Y)$ respectively represent the sets of probability distributions on the Polish spaces $X$ and $Y$. Formally, the Monge and Kantorovich problems can be stated as follows
\begin{align}
	\label{monge}
	C_{mp}(\mathbb{P},\mathbb{Q})\triangleq\inf_{T_\#\mathbb{P}= \mathbb{Q}}\int_{Y}c\big(y,T(y)\big)d\mathbb{P}(y),\\
	C_{kp}(\mathbb{P},\mathbb{Q})\triangleq\inf_{\pi\in\Pi(\mathbb P,\mathbb Q)}\int_{Y\times X}c(y,x)d\pi(y,x),
	\label{Kon}
\end{align}
where $c: Y\times X\rightarrow\mathbb R_+$ measures the transport cost between two samples. In the MP (\ref{monge}), the map $T^*$ attaining the infimum is called the \textit{optimal transport map}, which is taken over all the transport maps $T: Y\rightarrow X$. The constraint $T_\#\mathbb P=\mathbb Q$ means that $T$ pushes forward the probabilistic mass of $\mathbb P$ to $\mathbb Q$, where $T_\#$ is the push-forward operator. In the KP (\ref{Kon}), the coupling $\pi^*$ attaining the minimum is called the \textit{optimal transport plan}, which is taken over all the transport plans $\pi$ on $X\times Y$ whose marginals are $\mathbb P$ and $\mathbb Q$.

Recently, many attempts have been made to construct the translation/restoration as an OT problem. In this context, $\mathbb P, \mathbb Q$ represent the degraded distribution and target distribution respectively.  \citet{Gu2023optimal} propose to use optimal transport to guide the training of the conditional score-based diffusion model for super-resolution and translation. \citet{korotin2023neural, korotin2023kernel}  compute optimal transport maps and plans using neural networks under the duality framework and apply their method to unpaired image translation. They use $\ell_2$ regularizer for the transport cost $c(x,y)$. \citet{wang2022optimal} relax the transport constraint in Monge formulation with a Wasserstein-1 discrepancy penalty $W_1(\mathbb Q, T_\# \mathbb P)$ between the target distribution $\mathbb Q$ and the push-forward distribution $T_\# \mathbb P$. Likewise, they empirically use $\ell_2$ regularizer for $c(x,y)$ for the denoising task. These methods pioneer the way of modeling image translation/restoration problems as an OT problem, but their performances in multiple restoration problems are limited without proper prior knowledge about the correspondence between $\mathbb P$ and $\mathbb Q$. 

Different from the aforementioned studies, our RCOT customizes image restoration as an OT problem, crafting the transport cost and map through the integration of transport residual. This innovation leads to a degradation-aware and structure-preserving transport map, i.e., the RCOT map.
\vspace{-0.3cm}
\section{Method}

The key idea of RCOT is to introduce the transport residual as a degradation-specific cue for both the transport cost and transport map. We first model image restoration as an OT problem, exploiting the frequency knowledge of the residual, yielding the Fourier residual-guided OT (FROT) objective (section \ref{fs}). Secondly, and most crucially, we integrate the degradation-specific knowledge from the residual embedding in the transport map via a two-pass process (section 3.2), in which the transport residual is computed by the base model in the first pass and then encoded as a residual embedding to condition the second-pass restoration (see Figure \ref{model}). In section 3.3, we present the learning algorithm for the two-pass RCOT map by adversarially training two neural networks to solve the minimax optimization problem in both unpaired and paired settings.
\vspace{-0.2cm}
\subsection{Residual-guided OT Formulation for Restoration}
\begin{figure}[!t]
	\centering
	\includegraphics[scale=0.5]{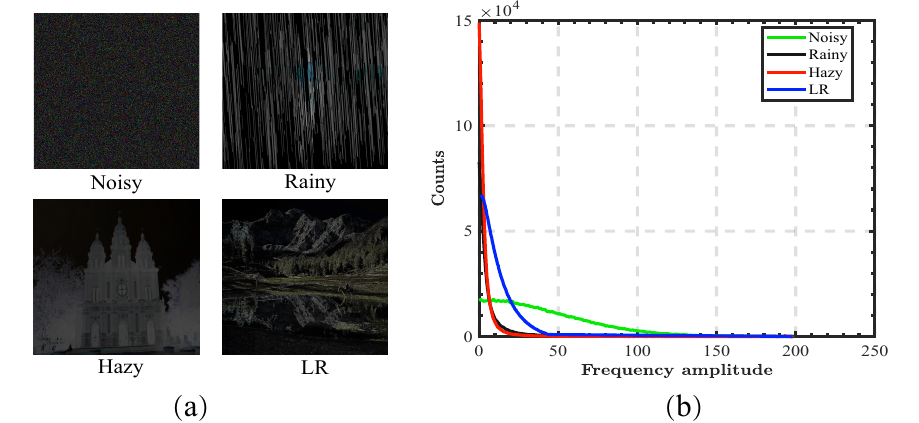}
	\caption{(a) Visual examples of the transport residual $r$. (b) Counts of the frequency amplitude of residuals for four types of degradation. For the draining, dehazing, and super-resolution tasks, the residuals are generally sparse in the frequency domain. The curves are averaged with 40 images. }
	\label{anl}
	\vspace{-0.3cm}
\end{figure}
\label{fs}
We first discuss formulizing image restoration as an OT problem, which can be applied to both unpaired and paired data. For the convenience of understanding, we elaborate our method for the unpaired data setting in sections 3.1 and 3.2.
We will specify and extend to the paired case in section 3.3.  We represent the domains of degraded images and target images by $Y$ and $X$, whose distributions are $\mathbb P$ and $\mathbb Q$, respectively. Then the Kantorovich \cite{kantorovich1942translocation} form of  OT cost can be defined by (\ref{Kon}).  However, this formulation does not consider the prior knowledge of the degradation, limiting its applicability to multiple restoration tasks. To this end, we suggest the FROT objective by introducing a penalty term $g(\cdot)$ on the degradation domain gap (i.e., the transport residual $r=y-x$) in the transport cost, leading to $\tilde{c}(y,x)=c(x,y)+g(r)$. The FROT objective is then  defined as 
\begin{align}
	\label{FROT}
	\text{FROT}(\mathbb P, \mathbb Q)  \triangleq \inf_{\pi\in {\rm \Pi}(\mathbb P,\mathbb Q)}\int_{X\times Y}\tilde{c}(y,x)d\pi(y,x).
\end{align}
Figure \ref{anl} presents the images of transport residuals for four types of degradation (noise, rain, haze, and low-resolution) along with their Fourier signal histograms. These histograms indicate that for degradations like rain, haze, and low-resolution, the residuals tend to be sparse in the frequency domain. In the case of noise, the histograms exhibit a smoother profile, resembling a Gaussian distribution. Based on the observation, we formalize the $\ell_1$ regularizer on the Fourier residuals for the deraining, dehazing, and super-resolution tasks, i.e., $g(\cdot)=||\mathcal F(\cdot)||_1$. For the denoising task, we specify the $\ell_2$ regularizer on the Fourier residual. Our objective is to find the corresponding OT map $T^*$ that attains the infimum of Monge's formulation (\ref{monge}) under cost $\tilde{c}$. The duality of (\ref{FROT}) will lead to a more manageable approach.  According to~\cite{villani2009optimal}, (\ref{FROT}) takes the following dual form:
\begin{align}
	\nonumber\text{FROT}(\mathbb P, \mathbb Q)=\sup_{\varphi}\int_Y\varphi^{\tilde{c}}(y)d\mathbb P(y)+\int_X\varphi(x)d\mathbb Q(x), 
\end{align}
where $\displaystyle\varphi^{\tilde{c}}(y)=\inf_{x\in X}\left[c(x,y)+g(r)-\varphi(x)\right]$ is the $c$-transform of $\varphi$. Replacing the optimization of the first term over target $x\in X$ with an equivalent optimization (Rockafellar interchange theorem \cite{rockafellar1976integral}, Theorem 3A) over the map of interest $T: X\rightarrow Y$,  we obtain the minimax reformulation of dual form:
\begin{align}
		\label{minimax}
\nonumber&\text{FROT}(\mathbb P, \mathbb Q)=\sup_{\varphi}\inf_{T}\bigg\{\mathcal L(T,\varphi)\triangleq\int_X\varphi(x)d\mathbb Q(x)\\
&+\int_{Y}\left[c(T(y),y)+g(\hat r(T))-\varphi(T(y))\right]d\mathbb P(y)\bigg\},
\end{align}
where $\hat r(T)=y-T(y)$ represents the transport degradation domain gap (termed as transport residual). Now we show that tackling this minimax problem provides the OT map. 
\begin{proposition}
	\label{pro}	\text{(Saddle points of FROT provide OT maps).} For any optimal potential function $\varphi^*\in\arg\sup_\varphi\mathcal L(T,\varphi)$, it holds for the Monge OT map $T^*$ that
	\begin{align} T^*\in \mathop{\arg\min}_T\mathcal L(T,\varphi^*).\end{align}
\end{proposition}
The proof is given in Appendix \ref{proof}. Proposition \ref{pro} affirms the feasibility of solving the minimax problem (\ref{minimax}) to acquire an optimal pair, constituting an OT map from $\mathbb P$ to $\mathbb Q$.  For general $\mathbb P$ and 
$\mathbb Q$, given some optimal potential $\varphi^*$, the set $\arg\inf_T$   may encompass not only the OT map $T^*$ but also other saddle points, which are capable of delivering decent performance as in experiments (section \ref{exp}). To tackle the minimax problem (\ref{minimax}), we can approximate the map $T$ and potential $\varphi$ with neural networks $T_\theta$ and $\varphi_\omega$. 

\subsection{Two-pass Residual-Conditioned OT Map}
\label{RCOT}
After parameterizing the transport map with a neural network $T_\theta$, to better preserve image structures for restoration under different degradations, we are motivated to incorporate degradation-specific and structural knowledge into the map $T_\theta$. Intuitively, the transport residual may contain degradation-specific knowledge (e.g., degradation type and level) and structural information (as empirically shown in Appendix \ref{res}). We then suggest utilizing the transport residual as an additional condition for the transport map $T_\theta$. To achieve this, we develop a two-pass RCOT map (Figure \ref{model}) based on a transport residual condition (TRC) module,  in which the transport residual is computed by the base model in the first pass and then encoded as a degradation-specific embedding to condition the second-pass restoration.  Correspondingly, we introduce an image generator $G_{\theta_1}$ to generate the restored image, and a residual encoder $E_{\theta_2}$ to control the restoration with the residual embedding as a condition. $G_{\theta_1}$ and $E_{\theta_2}$ constitute the map $T_\theta$.

\textbf{Two-pass transport map.} Given a degraded image $y$, the first pass unconditionally generates an intermediate restored result $G_{\theta_1}(y)$ via the image generator and calculates its corresponding intermediate transport residual $\hat r_0 = y-G_{\theta_1}(y)$. The second pass extracts the residual embedding $E_{\theta_2}(\hat r_0)$ via the residual encoder and uses it as a condition for the image generator, which then restores a refined result $G_{\theta_1}(y|E_{\theta_2}(\hat r_0))$. This transport process is summarized as
\begin{align}
	 \hat r_0=y-G_{\theta_1}(y),~~~
	T_\theta(y)=G_{\theta_1}\left(y|E_{\theta_2}(\hat r_0)\right).\label{2p}
\end{align}
\textbf{Transport residual condition (TRC).} The TRC module is intended to compensate for the degradation-specific information that can be lost through the degradation.
Specifically, the TRC module consists of two key components: a residual embedding generation module (REGM) and a condition integration module. In  REGM, the estimated transport residual in the first pass is projected by the residual encoder to a residual embedding $E_{\theta_2}(\hat r_0)$, which is then utilized as a degradation-specific condition for the second-pass restoration. In the condition integration module, we employ the cross-stage feature fusion \cite{Zamir2021MPRNet} to integrate the features from the image generator $G_{\theta_1}$ with the degradation-specific embedding $E_{\theta_2}(\hat r_0)$ for structure-preserving restoration.

\subsection{Overall Training}
With the parameterization of $T_\theta$ and $\varphi_\omega$, the optimization objective function of \eqref{minimax} can be written as 
\begin{align}\label{eq:loss_unpair}
    &\mathcal{L}_{\rm FROT}(\omega,\theta)=\mathbb{E}_{x\sim \mathbb{Q}}\left[\varphi_\omega(x)\right] \nonumber+ \\\mathbb{E}_{y\sim \mathbb{P}}&\left[ c(T_\theta(y),y)+g(\hat r(T_\theta))-\varphi_\omega(T_\theta(y)) \right].
\end{align}
For unpaired setting,  we train the networks $T_\theta$ and $\varphi_\omega$ by respectively minimizing and maximizing $\mathcal{L}_{\rm FROT}(\omega,\theta)$, i.e., $\max_{\omega}\min_{\theta} \mathcal{L}_{\rm FROT}(\omega,\theta)$. This can be achieved by adversarially training $T_\theta$ and $\varphi_\omega$, in which we estimate the expectation using mini-batch data in each training step.

\textbf{(Partially) Paired Setting.}  In practice, most datasets may include a fraction of available paired samples. For this partially paired setting, we can leverage the paired samples to enforce $T_\theta(y)$ to approximate the target $x$ for any pair $(y,x)\in P$ with a squared $\ell_2$ loss (where $P$ denotes the paired subsets of $X\times Y$):
\begin{equation}
    \mathcal L_{\text{paired}}(\theta) = \frac{\gamma}{|P|}\sum_{(y,x)\in P}\|T_\theta(y)-x\|^2.
\end{equation}
Consequently, the training objective for the partially paired setting is $\max_{\omega}\min_{\theta} \{ \mathcal{L}_{\rm FROT}(\omega,\theta) + \mathcal L_{\text{paired}}(\theta)\}$. The overall algorithm is detailed in Algorithm \ref{algo} in the Appendix \ref{algorithm}.


\section{Experiments}
\label{exp}
We evaluate the proposed RCOT on benchmark datasets on four representative image restoration tasks: image denoising, deraining,  dehazing, and super-resolution (SR). In tables, the best and second-best quality metrics (PSNR/SSIM for measuring pixel/structure similarity, and LPIPS \cite{zhang2018unreasonable}/FID \cite{heusel2017gans} for perceptual deviation measuring) are \textbf{highlighted} and \underline{underlined}. The \textbf{implementation details and selected compared methods} are introduced in Appendix \ref{details}. Extra evaluations (e.g., parameter quantity, cross-dataset and preliminary multiple-in-one comparisons) are included in Appendix \ref{aexp}.

\begin{table*}[!t]
	\centering
	\caption{Denoising results (PSNR/SSIM/LPIPS/FID) on Kodak24 \cite{franzen1999kodak} and CBSD68~\cite{martin2001database} datasets. RCOT achieves competitive qualitative performance. (*) indicates the method in an unpaired setting. }
	\label{denoising}
	\setlength{\tabcolsep}{10pt}
	\renewcommand{\arraystretch}{1.2}
	\resizebox{\textwidth}{!}{
		\begin{tabular}{@{}ccccc@{}}
			\toprule
			\multirow{2}{*}{Method} &\multicolumn{2}{c}{Kodak24~\cite{franzen1999kodak}} & \multicolumn{2}{c}{CBSD68~\cite{martin2001database}}\\ \cmidrule(lr){2-3} \cmidrule(lr){4-5}
			&   $\sigma = 25$ & $\sigma = 50$ &  $\sigma = 25$ & $\sigma = 50$ \\
			\midrule
			NOT$^*$ \cite{korotin2023neural}&  29.13/0.786/0.131/70.17 & 27.12/0.725/0.227/104.5 & 29.76/0.802/0.119/73.68 & 26.82/0.723/0.212/113.3 \\
			OTUR$^*$  \cite{wang2022optimal}& 31.05/0.848/0.104/55.74 & 28.03/0.744/0.178/83.27&  30.27/0.833/0.095/67.74 & 27.36/0.733/0.169/93.17 \\
			MPRNet \cite{Zamir2021MPRNet}&31.96/0.868/0.112/43.98&28.36/0.785/0.185/73.26&30.89/0.880/0.103/59.23&27.56/0.779/0.163/86.42\\
	Restormer\cite{Zamir2021Restormer}&32.13/0.880/0.097/40.22&\underline{29.25/0.799}/0.156/64.26&31.20/0.887/0.090/55.28&27.90/0.794/0.149/66.12\\
	IR-SDE \cite{luo2023image}&31.40/0.842/\underline{0.080}/45.56&28.03/0.721/\underline{0.134}/83.66&30.46/0.856/\underline{0.075}/57.30&26.98/0.737/\underline{0.138}/96.40\\
    RCD\cite{Zhang_2023_CVPR}&32.18/0.880/0.089/39.85&29.22/0.795/0.147/63.26&31.28/0.886/0.089/50.96&28.01/0.796/0.149/69.78\\
	PromptIR \cite{potlapalli2023promptir}& \underline{32.25/0.883}/0.091/\underline{30.91} & 29.19/\underline{0.799}/0.154/\underline{60.42}& \underline{31.31/0.888}/0.085/\underline{45.45}&\underline{28.03/0.797}/0.143/\underline{63.98}\\
			\midrule
			
			RCOT$^*$ & 31.84/0.860/0.085/35.28 & 28.64/0.792/0.152/63.27& 30.77/0.849/0.079/47.21&27.69/0.775/0.145/73.43 \\ 
			
			RCOT&  \textbf{32.64/0.885/0.070/19.33} & \textbf{29.53/0.828/0.121/33.13} & \textbf{31.52/0.889/0.062/34.24} & \textbf{28.25/0.799/0.118/56.60} \\ 
			\bottomrule
	\end{tabular}}
\end{table*}
\begin{figure*}[!t]
	\setlength\tabcolsep{1pt}
	\renewcommand{\arraystretch}{0.5} 
	\centering
	\begin{tabular}{cccccccc}
		\includegraphics[width=0.120\linewidth]{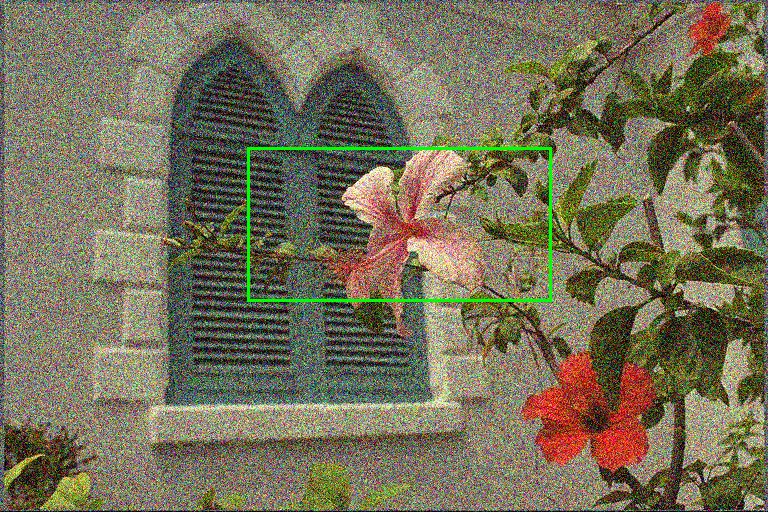}&
		\includegraphics[width=0.120\linewidth]{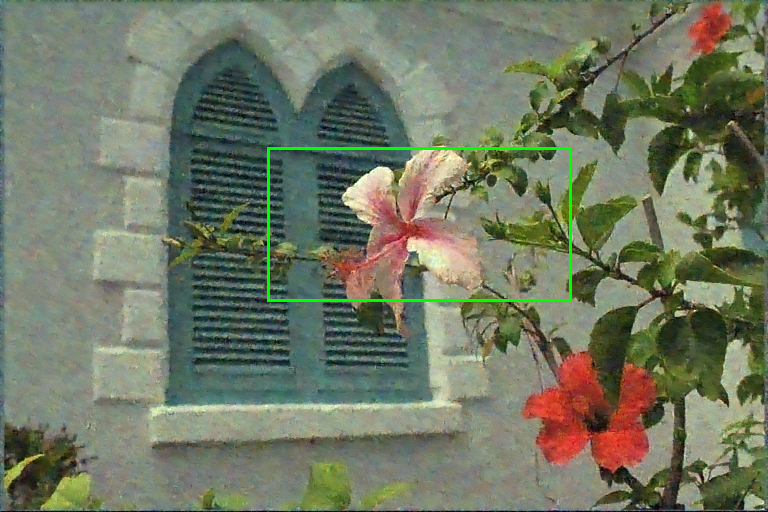}&
		\includegraphics[width=0.120\linewidth]{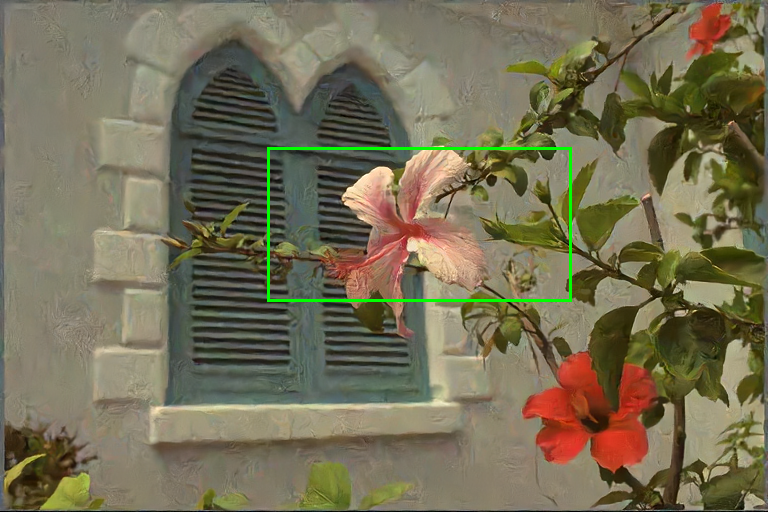}&
		\includegraphics[width=0.120\linewidth]{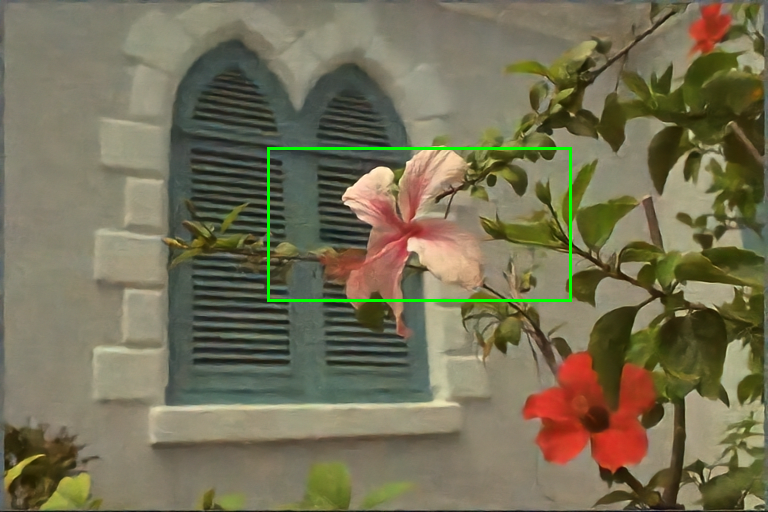}&
		\includegraphics[width=0.120\linewidth]{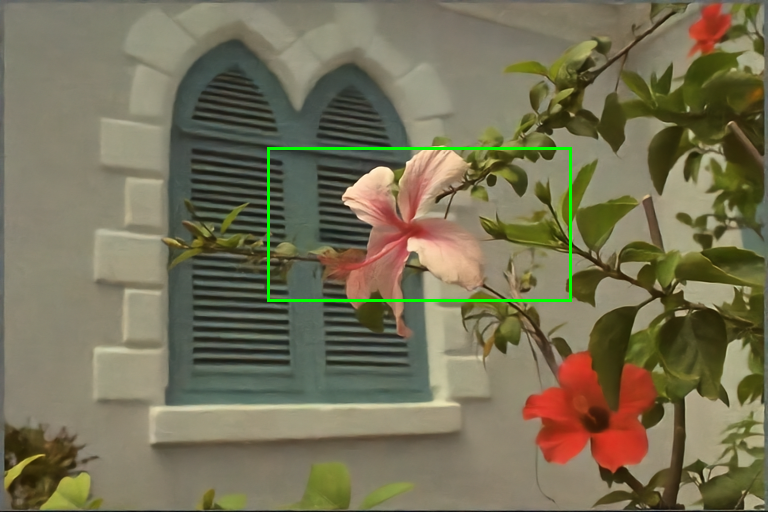}&
		\includegraphics[width=0.120\linewidth]{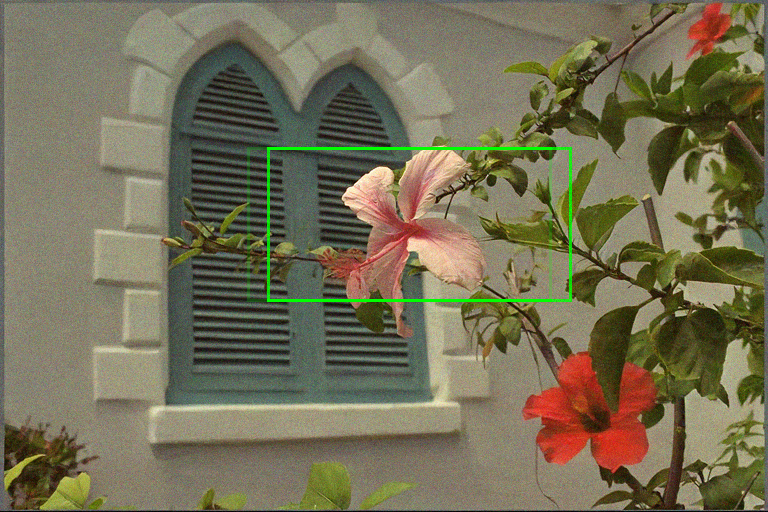}&
		\includegraphics[width=0.120\linewidth]{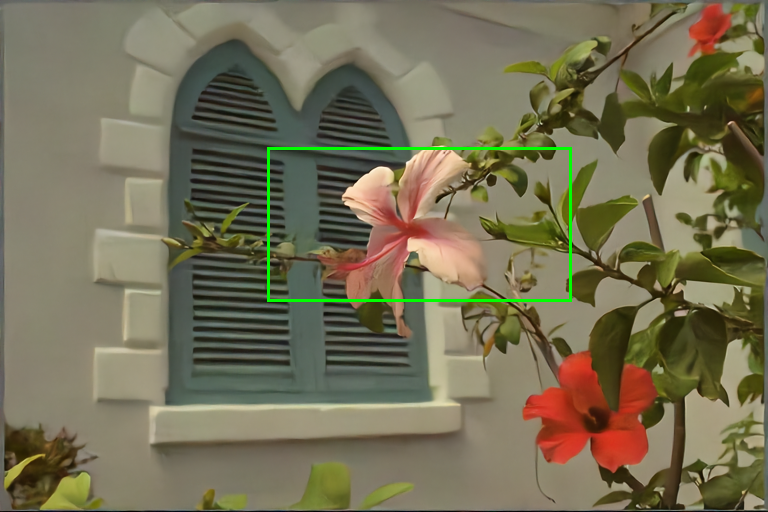}&
		\includegraphics[width=0.120\linewidth]{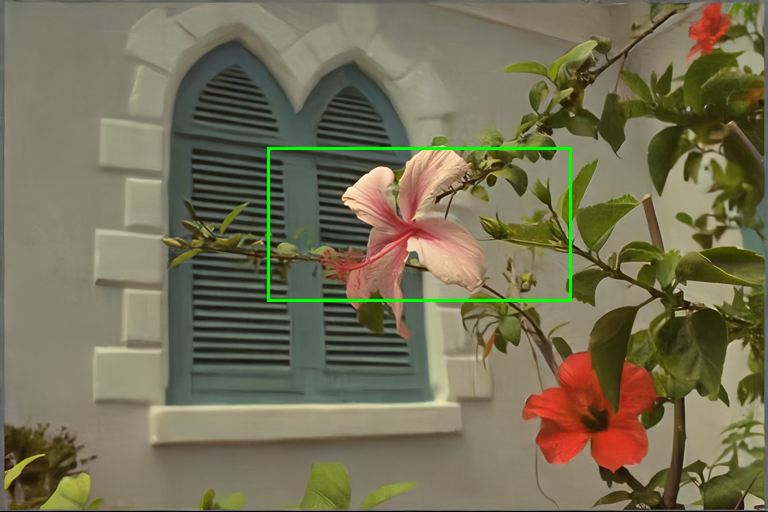}\\
		\includegraphics[width=0.120\linewidth]{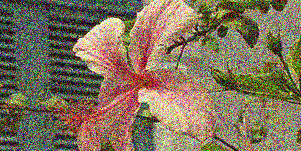}&
		\includegraphics[width=0.120\linewidth]{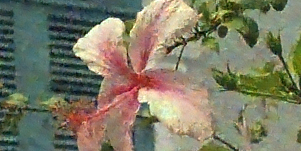}&
		\includegraphics[width=0.120\linewidth]{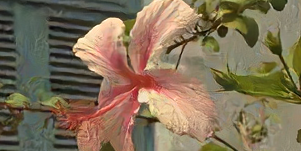}&
		\includegraphics[width=0.120\linewidth]{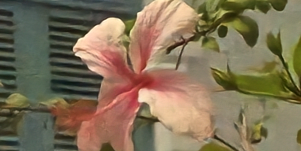}&
		\includegraphics[width=0.120\linewidth]{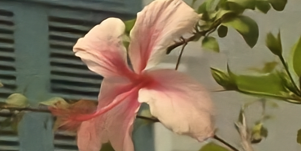}&
		\includegraphics[width=0.120\linewidth]{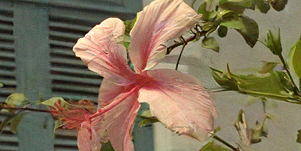}&
		\includegraphics[width=0.120\linewidth]{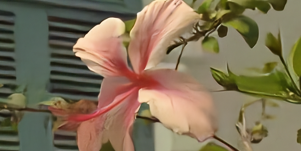}&
		\includegraphics[width=0.120\linewidth]{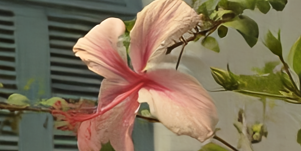}\\
		14.60/0.256&26.60/0.822&29.29/0.921&29.77/0.940&\textbf{31.32}/0.955&29.52/0.921&30.98/0.951&31.26/\textbf{0.961}
		\\ 
		Noisy&NOT &OTUR&MPRNet&Restormer&IR-SDE&PromptIR&RCOT
	\end{tabular}
	\caption{Visual comparison of denoising on Kodak24 \cite{franzen1999kodak} with $\sigma=50$. Our RCOT produces a noise-free image with clear textural details.}
	\label{noise}

\end{figure*}
\subsection{Results}
\paragraph{Gaussian Image Denoising.}
For Gaussian image denoising, we train the model on a combination of BSD400 \cite{arbelaez2010contour} and WED \cite{ma2016waterloo} datasets. The BSD400 dataset comprises 400 training images, while the WED dataset consists of 4,744 images. Gaussian noise with level $\sigma\in\{25,50\}$ is separately added to generate noisy images for training. We evaluate RCOT on the Kodak24 \cite{franzen1999kodak} and CBSD68 \cite{martin2001database} datasets under noise levels $\sigma\in\{25,50\}$. Table \ref{denoising} reports the PSNR/SSIM/LPIPS/FID scores of the compared methods. The RCOT achieves the best performance under all metrics. Particularly, the LPIPS and FID values of RCOT's results are remarkably better than those of other methods. Figure \ref{noise} and Figure \ref{anoisy} in the Appendix \ref{more} display visual examples of $\sigma=50$. MPRNet, Restormer, and PromptIR \cite{Zamir2021MPRNet, Zamir2021Restormer,potlapalli2023promptir}, directly minimizing the $\ell_1$ loss, can perform well in distortion measures (PSNR and SSIM), but they produce results with smoothed structures.  The generation-based methods IR-SDE \cite{luo2023image} and OTUR \cite{wang2022optimal} seem to produce realistic structures but are limited in removing the heavy noise completely. As a comparison, our RCOT reproduces a clean and sharp image with faithful textures. 
\begin{table*}[!t]
	\centering
	\caption{Deraining results on synetheic dataset Rain100L~\cite{fan2019general} and real-world dataset SPANet \cite{Wang_2019_CVPR}.  ($*$) indicates the method in an unpaired setting.}
	\label{derain}
	\setlength{\tabcolsep}{15pt}
	\renewcommand{\arraystretch}{1.2}
	\resizebox{\textwidth}{!}{
		\begin{tabular}{@{}ccccccccc@{}}
			\toprule
			\multirow{2}{*}{Method}&\multicolumn{4}{c}{Synthetic Rain100L~\cite{fan2019general}} & \multicolumn{4}{c}{Real-world SPANet~\cite{Wang_2019_CVPR}}\\ \cmidrule(lr){2-5}  \cmidrule(lr){6-9} 
			& PSNR ($\uparrow$) & SSIM ($\uparrow$) & LPIPS ($\downarrow$) & FID ($\downarrow$) & PSNR ($\uparrow$) & SSIM ($\uparrow$) & LPIPS ($\downarrow$) & FID ($\downarrow$) \\
			\midrule
			NOT$^*$ \cite{korotin2023neural}& 29.29 & 0.911 & 0.030 & 54.12 & 32.55 & 0.901& 0.028 &53.28 \\
			OTUR$^*$  \cite{wang2022optimal}& 33.71 & 0.954 & 0.027 & 36.64 & 39.23& 0.961& 0.017 &29.64 \\
			MPRNet \cite{Zamir2021MPRNet}& 34.95 & 0.964 & 0.039 & 21.61 & 39.52 & 0.967& 0.021 &28.13\\
			Restormer \cite{Zamir2021Restormer}& 36.74 & 0.978 & 0.026 & 13.29 & 41.39& 0.981& 0.013 &19.67\\
               SFNet \cite{cui2023selective}& 36.56 & 0.974 & 0.023 & 13.12 & 41.02 & 0.980& 0.015 &21.52\\
			IR-SDE \cite{luo2023image}& 36.94 & 0.978& \underline{0.014} & \underline{9.52} & \underline{42.56} & \underline{0.987}& \underline{0.009} &\underline{16.25}\\
			PromptIR \cite{potlapalli2023promptir}& \underline{37.09} & \underline{0.979} & 0.025 &  10.21 & 39.17 & 0.965 & 0.015 &  28.93\\
			\midrule
			RCOT$^*$ & 36.22 & 0.972 & 0.019 & 12.59 & 41.05 & 0.977& 0.014 &18.93\\
			
			RCOT& \textbf{37.27} & \textbf{0.980} & \textbf{0.015} & \textbf{7.97} & \textbf{43.77} & \textbf{0.993}& \textbf{0.008} &\textbf{9.52} \\ 
			\bottomrule
	\end{tabular}}
\vspace{-0.3cm}
\end{table*}
\begin{figure*}[!h]
	\setlength\tabcolsep{1pt}
	\renewcommand{\arraystretch}{0.5} 
	\centering
	\begin{tabular}{cccccccc}
 	\includegraphics[width=0.120\linewidth]{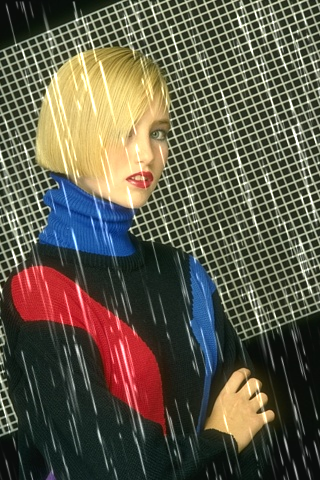}&
		\includegraphics[width=0.120\linewidth]{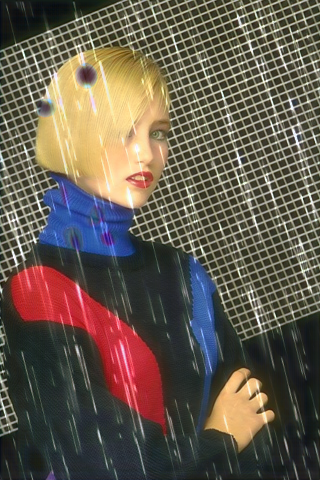}&
		\includegraphics[width=0.120\linewidth]{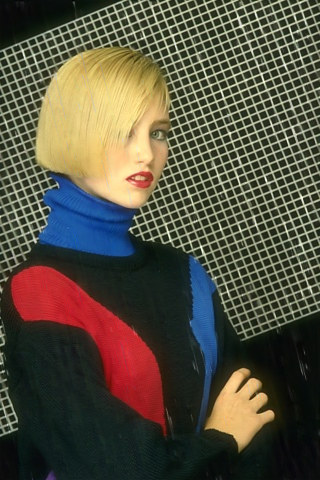}&
		\includegraphics[width=0.120\linewidth]{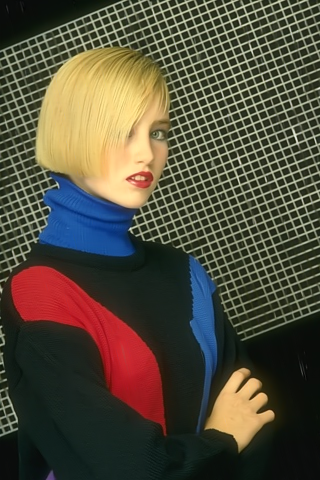}&
		\includegraphics[width=0.120\linewidth]{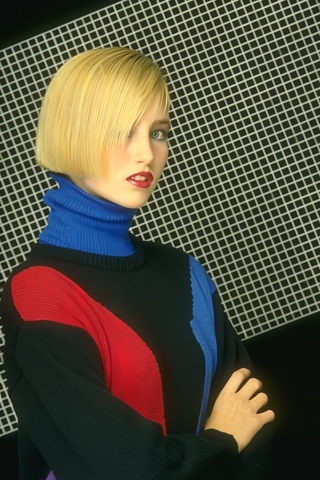}&
		\includegraphics[width=0.120\linewidth]{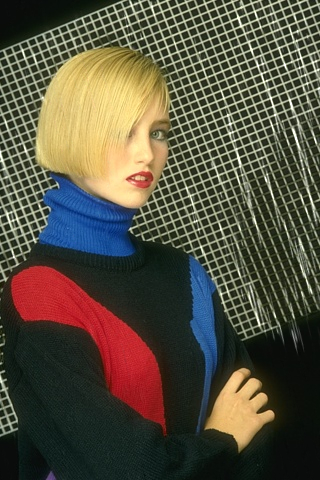}&
		\includegraphics[width=0.120\linewidth]{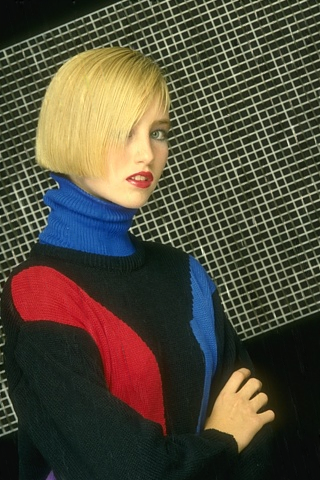}&
		\includegraphics[width=0.120\linewidth]{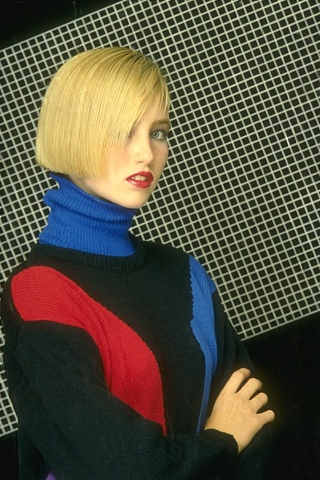}\\
19.97/0.829&24.87/0.876&29.23/0.966&27.66/0.952&35.15/0.981&27.03/0.941&30.69/0.937&\textbf{35.40/0.989}\\

		\includegraphics[width=0.120\linewidth]{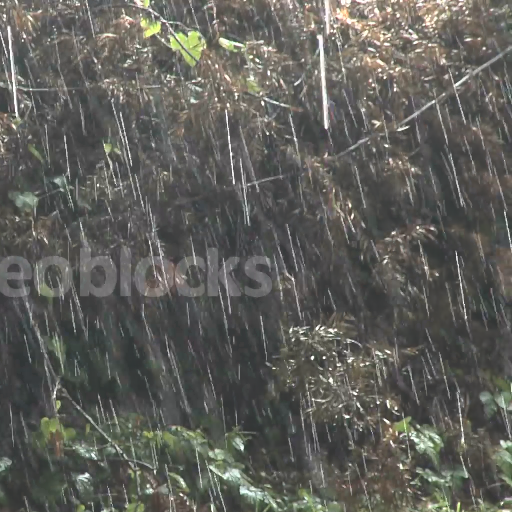}&
		\includegraphics[width=0.120\linewidth]{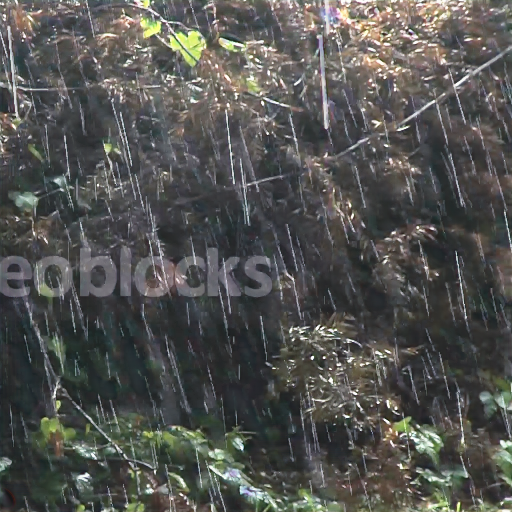}&
		\includegraphics[width=0.120\linewidth]{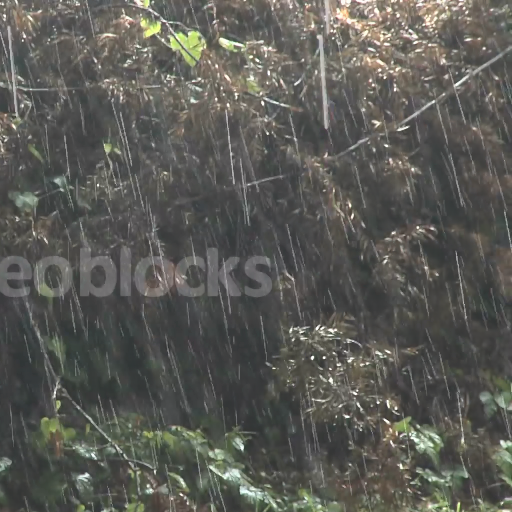}&
		\includegraphics[width=0.120\linewidth]{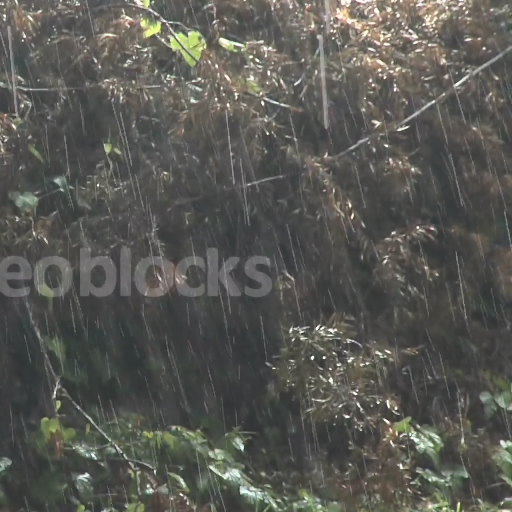}&
		\includegraphics[width=0.120\linewidth]{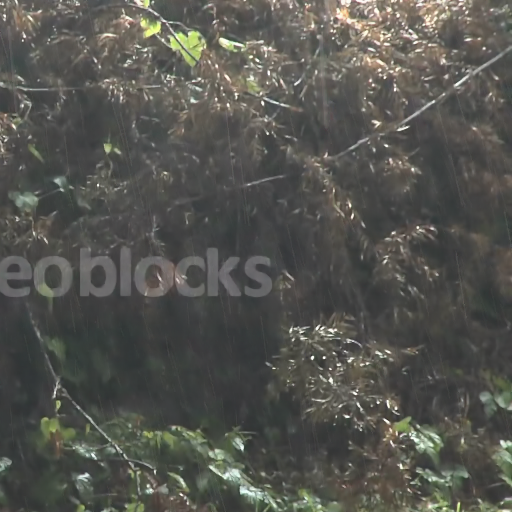}&
		\includegraphics[width=0.120\linewidth]{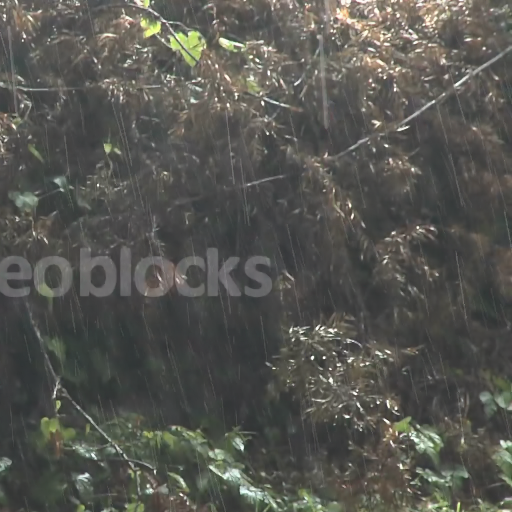}&
		\includegraphics[width=0.120\linewidth]{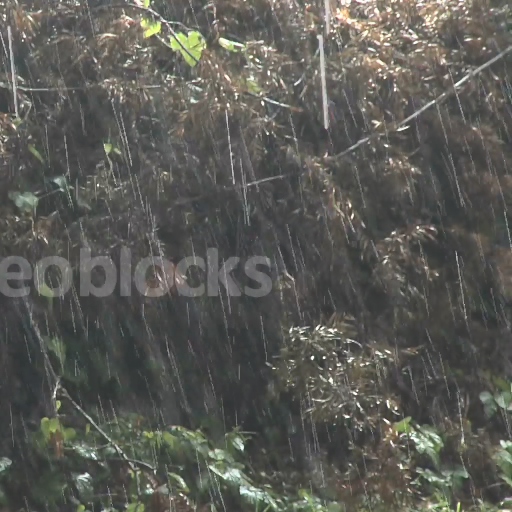}&
            \includegraphics[width=0.120\linewidth]{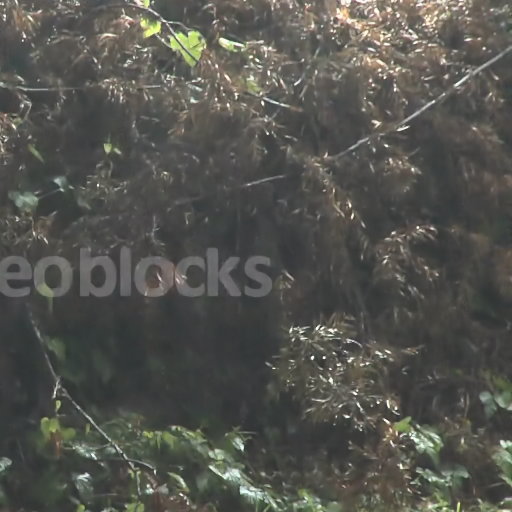}\\
21.96/0.729& 24.19/0.758&26.04/0.811&28.83/0.859& 32.14/0.928&32.52/0.931&28.24/0.873& \textbf{33.16/0.939}\\        
		Rainy&NOT&OTUR&MPRNet&Restormer&IR-SDE&PromptIR&RCOT
		
	\end{tabular}
	\caption{Visual comparison of deraining results on synthetic Rain100L \cite{fan2019general} (row 1) and  real-world SPANet \cite{Wang_2019_CVPR} (row 2). The RCOT reproduces rain-free images with realistic structural details.}
	\label{vrain}
\end{figure*} 

\begin{table*}[!h]
	\centering
	\caption{Dehazing results on synetheic dataset SOTS~\cite{li2018benchmarking} and real-world dataset O-HAZE  \cite{ancuti2018haze}.  ($*$) indicates the method in an unpaired setting.}
	\label{haze}
	\setlength{\tabcolsep}{15pt}
	\renewcommand{\arraystretch}{1.2}
	\resizebox{\textwidth}{!}{
		\begin{tabular}{@{}ccccccccc@{}}
			\toprule
			\multirow{2}{*}{Method}&\multicolumn{4}{c}{Synthetic SOTS (outside)~\cite{fan2019general}} & \multicolumn{4}{c}{Real-world O-HAZE~\cite{ancuti2018haze}}\\ \cmidrule(lr){2-5}  \cmidrule(lr){6-9} 
			& PSNR ($\uparrow$) & SSIM ($\uparrow$) & LPIPS ($\downarrow$) & FID ($\downarrow$) & PSNR ($\uparrow$) & SSIM ($\uparrow$) & LPIPS ($\downarrow$) & FID ($\downarrow$) \\
			\midrule
			NOT$^*$ \cite{korotin2023neural}& 24.21 & 0.900& 0.046 & 21.35 & 15.13 & 0.673& 0.271 &258.68 \\
			OTUR$^*$  \cite{wang2022optimal}& 26.36 & 0.953 & 0.024 & 18.96 & 16.43& 0.719& 0.209 &248.74 \\
			MPRNet \cite{Zamir2021MPRNet}& 28.31 & 0.954 & 0.029 &17.79 & 21.55 & 0.778& 0.256 &223.86\\
			Restormer \cite{Zamir2021Restormer}& 30.87 & 0.969 & 0.026 & 13.29 & 25.20& 0.804& 0.221 &198.07\\
               Dehazeformer \cite{song2023vision}& \underline{31.45} & \textbf{0.978} & 0.021 & 15.54 & \underline{25.56} & 0.812& 0.209 &199.35\\
			IR-SDE \cite{luo2023image}& 30.55& 0.968& \underline{0.018} & \underline{12.76} & 22.13 & 0.776& \underline{0.160} &\underline{179.23}\\
			PromptIR \cite{potlapalli2023promptir}& 31.31 & 0.973 & 0.021 &  16.28 & 25.27 & \underline{0.813} & 0.216 &  217.66\\
			\midrule
			RCOT$^*$ & 30.34 & 0.965 & 0.020 & 12.97 & 21.01 & 0.773& 0.186 &203.72\\
			
			RCOT& \textbf{31.66} & \underline{0.976} & \textbf{0.015} & \textbf{10.21} & \textbf{27.16} & \textbf{0.839}& \textbf{0.145} &\textbf{169.38} \\ 
			\bottomrule
	\end{tabular}}
 \vspace{-0.2cm}
\end{table*}
\begin{figure*}[!h]
	\setlength\tabcolsep{1pt}
	\renewcommand{\arraystretch}{0.5} 
	\centering
	\begin{tabular}{cccccccc}
		\includegraphics[width=0.120\linewidth]{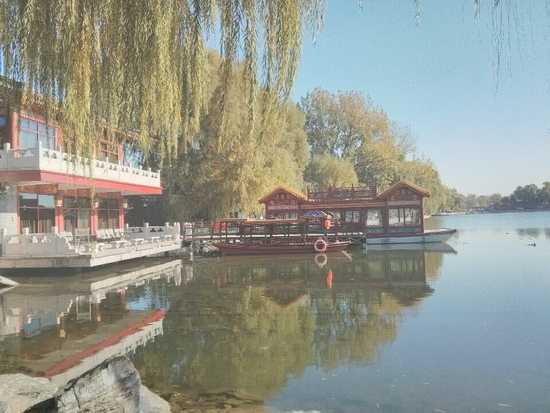}&
		\includegraphics[width=0.120\linewidth]{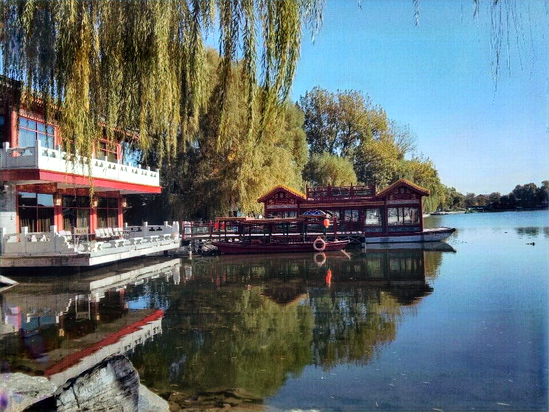}&
		\includegraphics[width=0.120\linewidth]{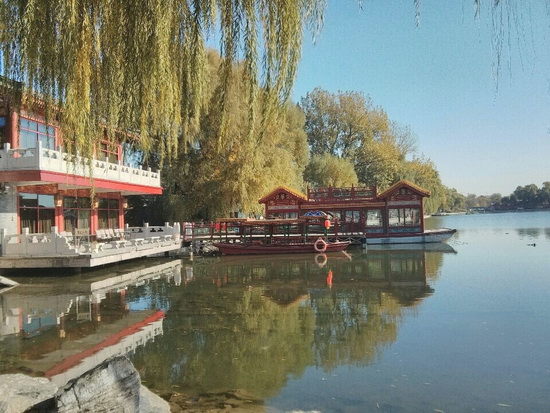}&
		\includegraphics[width=0.120\linewidth]{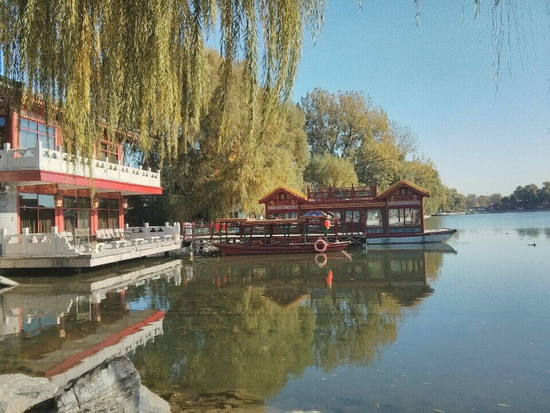}&
		\includegraphics[width=0.120\linewidth]{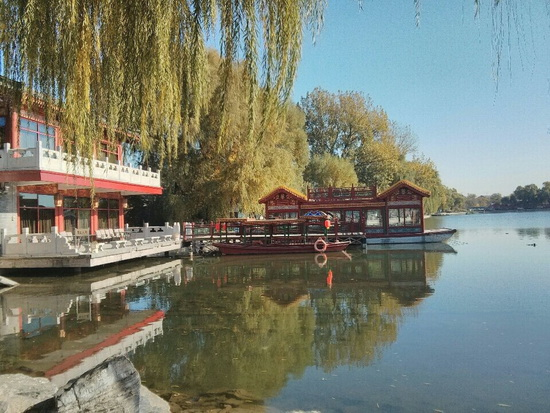}&
		\includegraphics[width=0.120\linewidth]{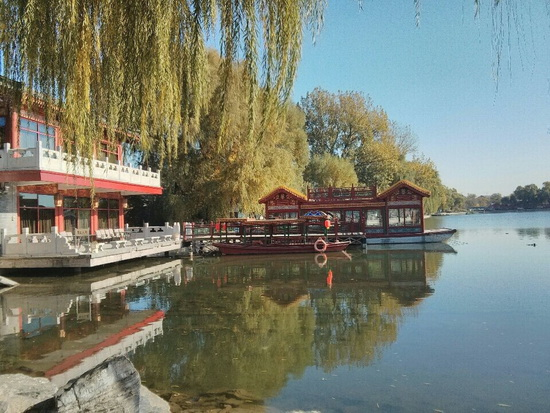}&
		\includegraphics[width=0.120\linewidth]{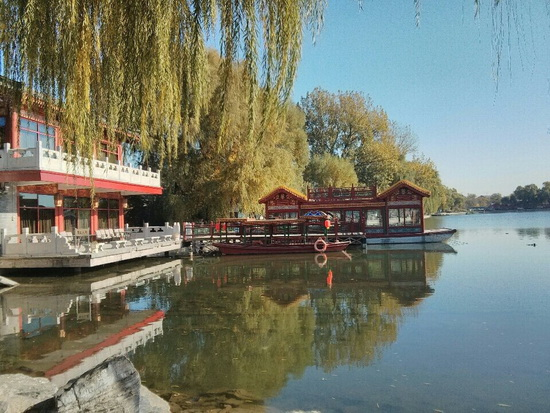}&
		\includegraphics[width=0.120\linewidth]{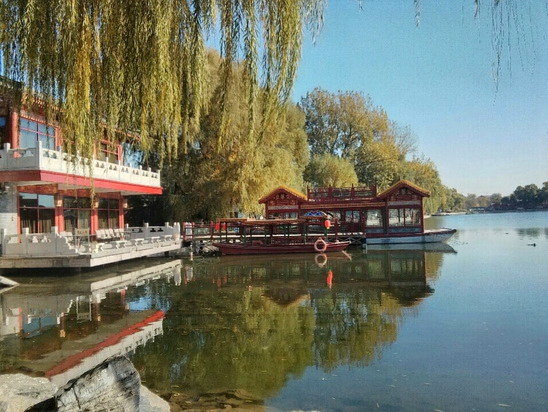}\\
		17.28/0.897&26.65/0.949&27.70/0.985&28.52/0.986&31.25/0.989&32.76/0.992&33.01/0.991&\textbf{33.21/0.993}\\

  	\includegraphics[width=0.120\linewidth]{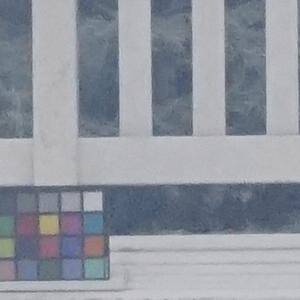}&
		\includegraphics[width=0.120\linewidth]{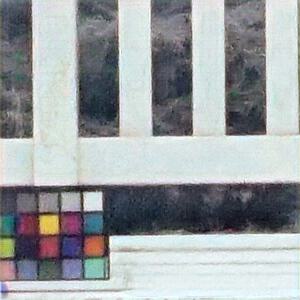}&
		\includegraphics[width=0.120\linewidth]{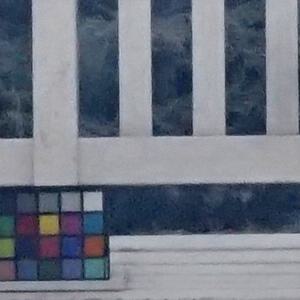}&
		\includegraphics[width=0.120\linewidth]{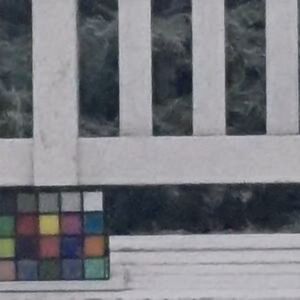}&
		\includegraphics[width=0.120\linewidth]{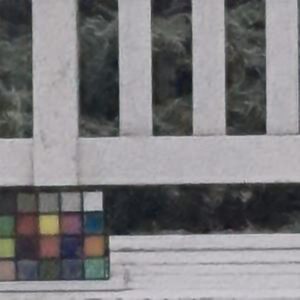}&
		\includegraphics[width=0.120\linewidth]{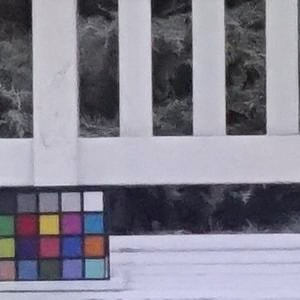}&
		\includegraphics[width=0.120\linewidth]{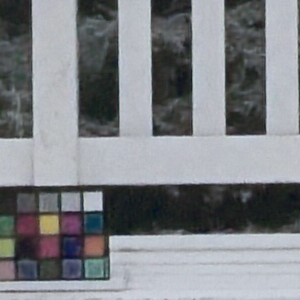}&
		\includegraphics[width=0.120\linewidth]{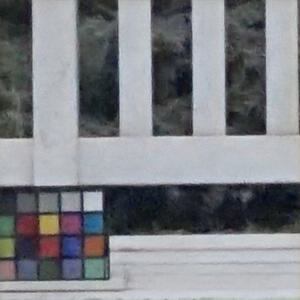}\\
        16.78/0.704&18.11/0.724&20.31/0.743&23.61/0.810&25.51/0.781&23.14/0.775&27.80/0.827&\textbf{28.50/0.841}\\
		Hazy&NOT&OTUR&MPRNet&Restormer&IR-SDE&PromptIR&RCOT
	\end{tabular}
	\caption{Visual comparison of dehazing results on synethtic SOTS~\cite{li2018benchmarking} (row 1) and real-world O-HAZE \cite{ancuti2018haze} (row 2).  RCOT produces haze-free images with faithful color.}
	\label{vhaze}
 \vspace{-0.2cm}
\end{figure*}
\paragraph{Image Deraining.}
We evaluate RCOT on both \textbf{synthetic} dataset Rain100L \cite{yang2017deep} and \textbf{real-world} dataset SPANet \cite{Wang_2019_CVPR}. For Rain100L \cite{yang2017deep}, we train the model with  13,712 paired clean-rain images collected from multiple datasets \cite{fu2017removing, li2016rain, yang2017deep, zhang2018density, zhang2019image}.
For real-world SPANet \cite{Wang_2019_CVPR}, it contains 27.5K paired rainy and rain-free images for training, and 1, 000 paired images for testing. 

Table \ref{derain} reports the performance of the evaluated methods. The proposed RCOT achieves the best performance over both distortion measures and perceptual quality measures. The underlying reason should be that the RCOT benefits from the FROT objective and residual embedding that exploits the information of rain streaks.  From the results shown in Figures \ref{vrain} and \ref{arain}, we can observe that results of NOT \cite{korotin2023neural} are discolored. Restormer \cite{Zamir2021Restormer} and PromptIR \cite{potlapalli2023promptir} effectively remove the rain streaks, but the structural details are oversmoothed. IR-SDE \cite{luo2023image} produces a restored result that still exhibits a slight presence of rain streaks. RCOT restores rain-free images with better structural content.

\paragraph{Image Dehazing.}
We evaluate RCOT on the \textbf{synthetic} SOTS \cite{li2018benchmarking} dataset, which contains 72,135 images for training and 500 images for testing, and \textbf{real hazy} O-HAZE \cite{ancuti2018haze} dataset, which contains 45 paired hazy and haze-free images, which are collected at the same scene under the same illumination conditions. 40 images are used for training and the other 5 images are used for testing. 

Table \ref{haze} reports the qualitative results. The proposed RCOT achieves the best performance overall. Notably, the unpaired RCOT also achieves a promising performance, especially in terms of the LPIPS and FID values. From the results shown in Figure \ref{vhaze} and \ref{vrealhaze}, we can observe that NOT \cite{korotin2023neural}
and IR-SDE \cite{luo2023image} produce sharp results but are not faithful in color and removing haze. Restormer, MPRNet, and PromptIR can remove the haze but produce results with distorted color. As a comparison, the RCOT can remove the haze while faithfully preserving the color.
\vspace{-0.2cm}
\paragraph{Image Super-resolution.}
\begin{table*}[!h]
	\centering
	\caption{SR results on DIV2K \cite{agustsson2017ntire}. (*) indicates the method in an unpaired setting.}
	\label{sr}
	\setlength{\tabcolsep}{15pt}
	\renewcommand{\arraystretch}{1.2}
	\resizebox{\textwidth}{!}{
		\begin{tabular}{@{}c|ccccccc|cc@{}}
			\toprule
			\multirow{2}{*}{Method}& \multirow{2}{*}{Bicubic}&NOT$^*$ & OTUR$^*$  & Restomer &IR-SDE&IDM&LINF& \multirow{2}{*}{RCOT$^*$}&\multirow{2}{*}{RCOT}\\
			&&\cite{korotin2023neural}&\cite{wang2022optimal}&\cite{Zamir2021Restormer}&\cite{luo2023image}&\cite{gao2023implicit}&\cite{yao2023local}&&\\
			\midrule
			PSNR$\uparrow$&26.70&25.73&24.88&\underline{27.90}&26.89&27.59&27.33&26.78&\textbf{28.41}\\
			SSIM$\uparrow$&0.771&0.718&0.679&\underline{0.796}&0.775&0.785&0.769&0.758&\textbf{0.804}\\
        	LPIPS$\downarrow$&0.186&0.136&0.128&0.136&0.118&0.121&\underline{0.116}&0.125&\textbf{0.114}\\
            FID$\downarrow$&15.88&19.27&24.65&6.14&8.56&5.73&\underline{3.26}&4.51&\textbf{1.22}\\
			\bottomrule	
	\end{tabular}}
\end{table*}
\begin{figure*}[!ht]
	\setlength\tabcolsep{1pt}
	\renewcommand{\arraystretch}{0.5} 
	\centering
	\begin{tabular}{ccccccc}
		\includegraphics[width=0.130\linewidth]{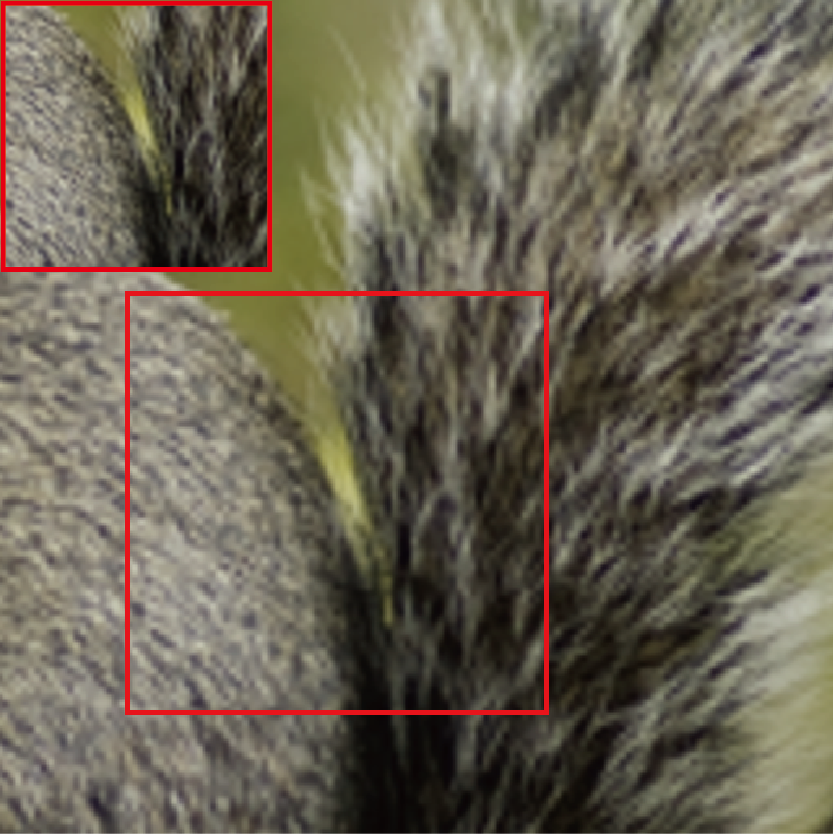}&
		\includegraphics[width=0.130\linewidth]{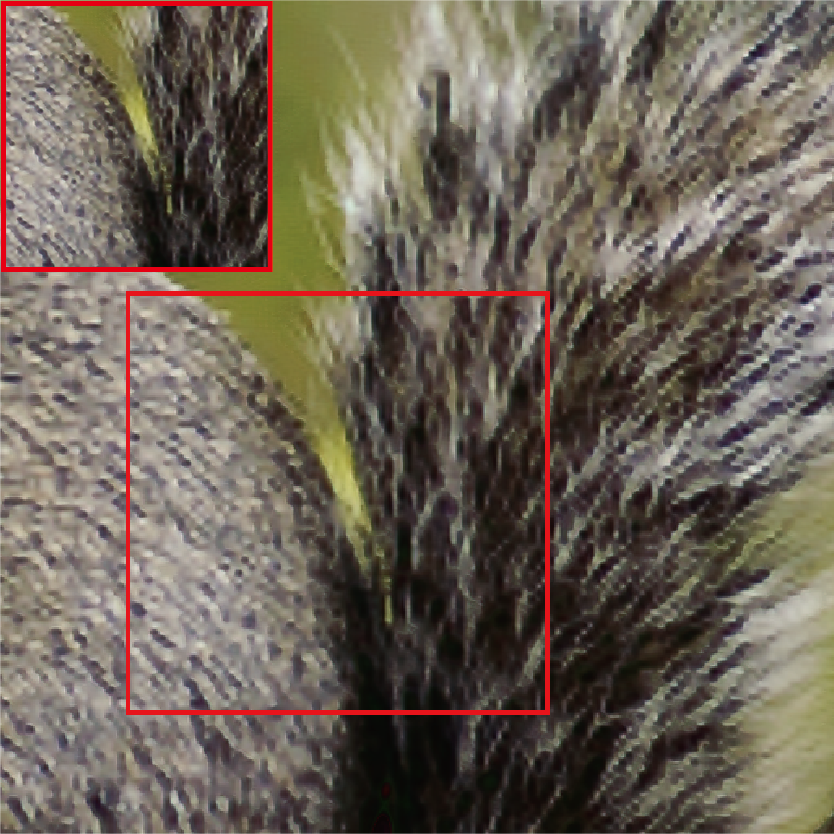}&
		\includegraphics[width=0.130\linewidth]{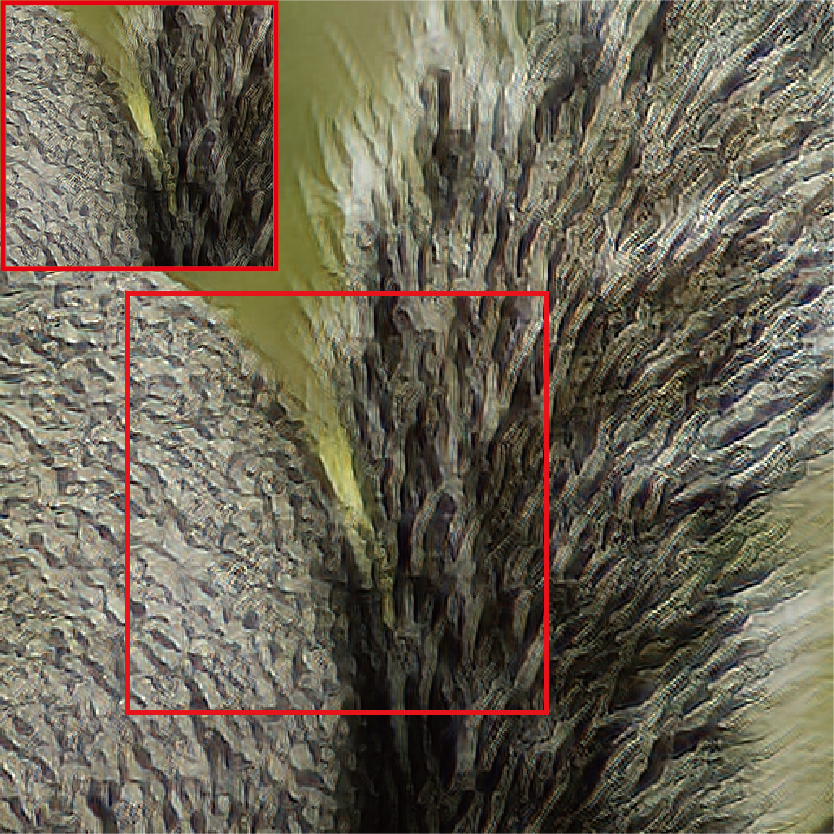}&
		\includegraphics[width=0.130\linewidth]{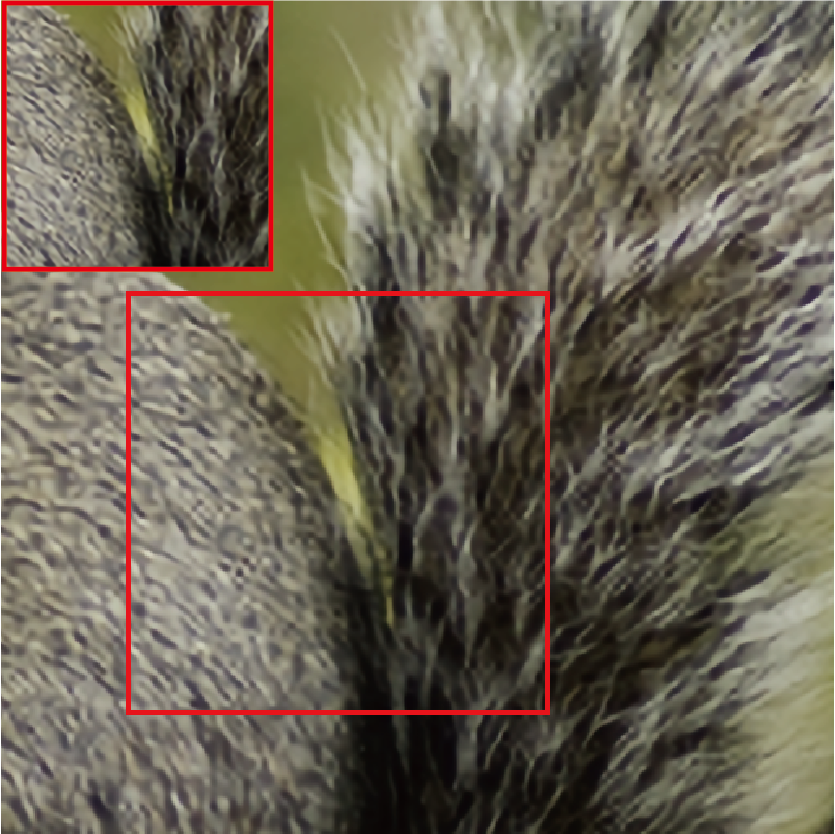}&
		\includegraphics[width=0.130\linewidth]{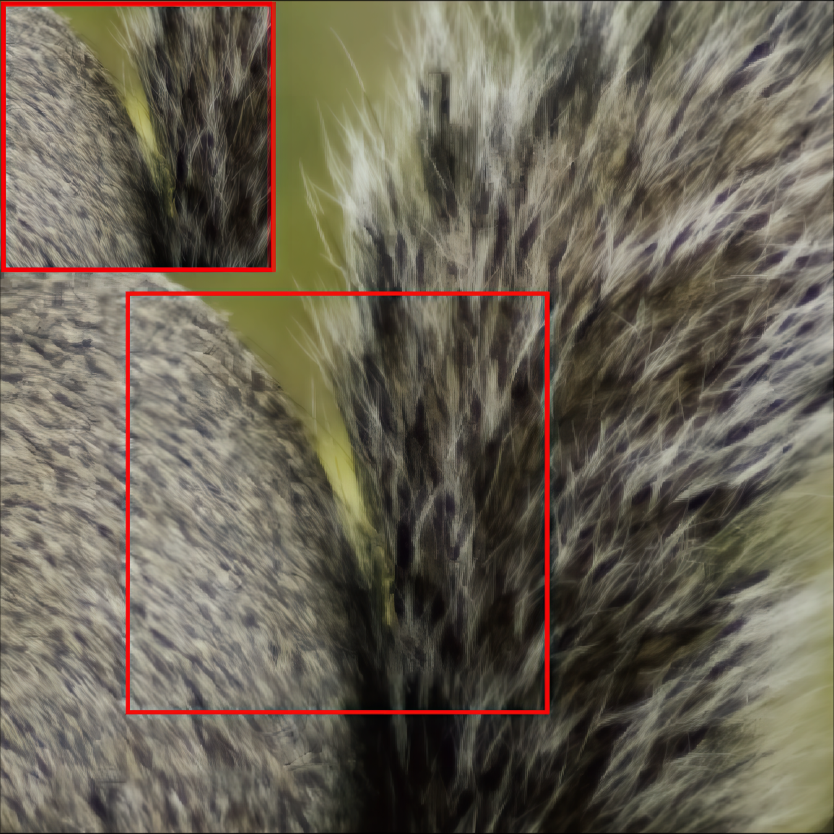}&
		\includegraphics[width=0.130\linewidth]{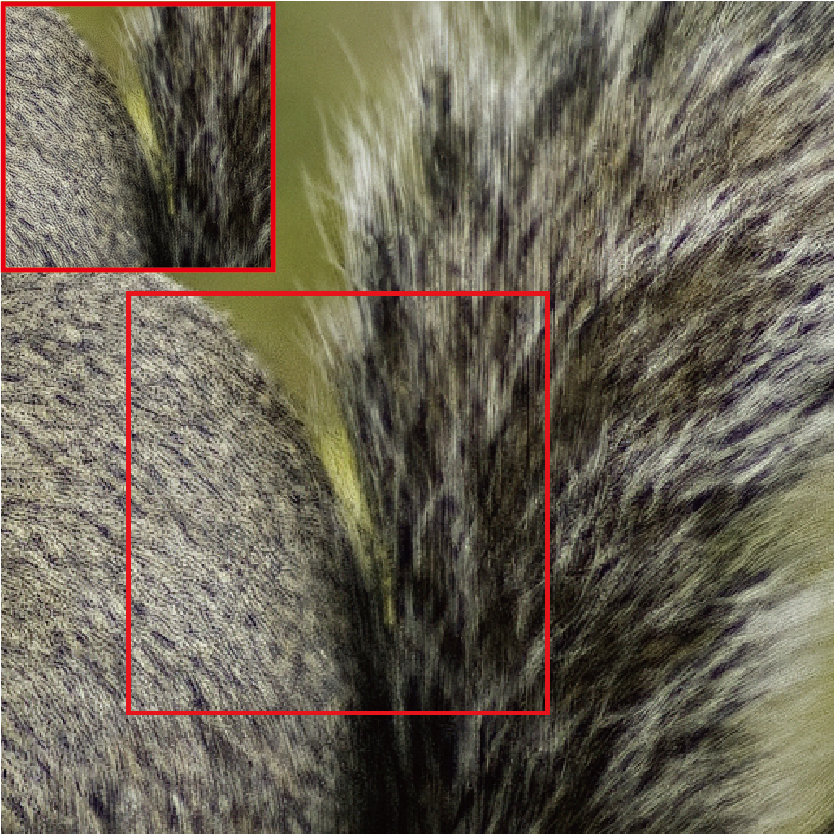}&
		\includegraphics[width=0.130\linewidth]{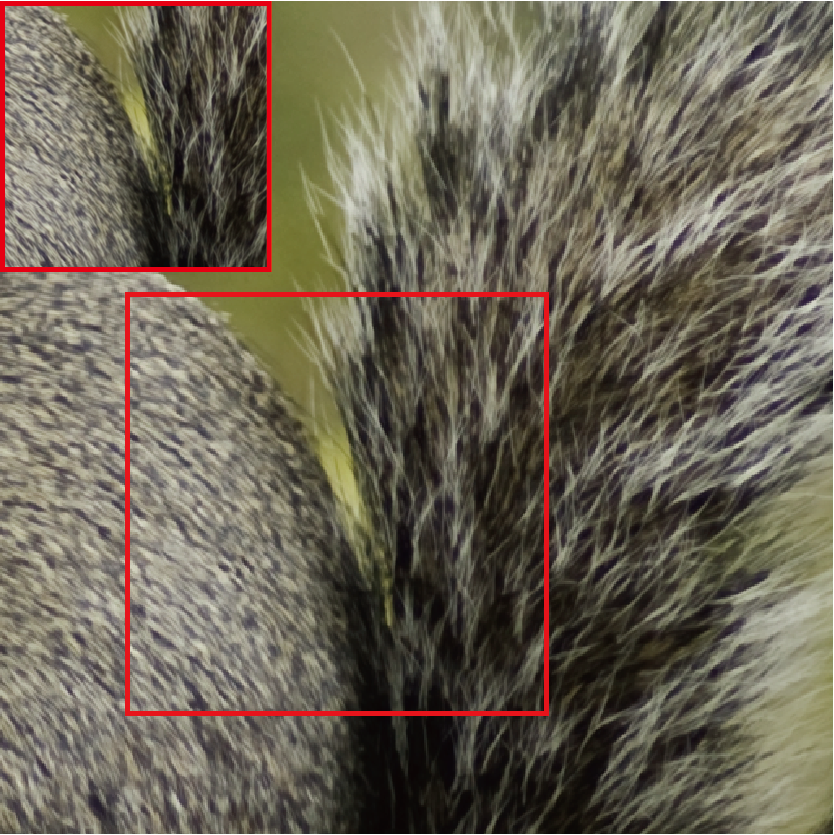}\\
		25.22/0.782&24.76/0.768&24.34/0.744&25.54/0.772&25.78/0.799&25.36/0.785&\textbf{26.12/0.814}\\
  
		Bicubic& NOT&OTUR&Restormer&IDM&IR-SDE&RCOT
	\end{tabular}
	\caption{Visual comparison of 4x SR on the DIV2K \cite{agustsson2017ntire} dataset. Please zoom in for better visualization. RCOT produces the sharpest HR image with more realistic structures.}
	\label{vsr}
\end{figure*}	

\begin{table*}[!h]
	\centering
	\caption{Results w/ and w/o TRC. TRC produces a significant boost of  PSNR/SSIM/FID values. }
	\label{treg}
	\setlength{\tabcolsep}{15pt}
	\resizebox{\textwidth}{!}{
		\begin{tabular}{c|ccc|c}
			\toprule
			Method  & SR on DIV2K & Deraining on Rain100L &  {Denoising on Kodak24}&
			Average \\
			\midrule
			w/o TRC & 27.24/0.785/6.72  & 34.20/0.955/28.63 & 28.45/0.790/64.20 & 29.97/0.843/32.47 \\
			\midrule
			w/ TRC& \textbf{28.41}/\textbf{0.804}/\textbf{1.22} & \textbf{37.27}/\textbf{0.978}/\textbf{8.67} & \textbf{29.53}/\textbf{0.828}/\textbf{33.13} & \textbf{31.72}/\textbf{0.870}/\textbf{14.34} \\
			\bottomrule
	\end{tabular}}
\end{table*}

\begin{table}[!htbp]
	\vspace{-0.3cm}
	\centering
	\caption{Effect of the loss components. The model is trained under different losses for denoising and tested on CBSD68 \cite{martin2001database} with noise level $\sigma=50$. FROT shows decent performance, working alone and with the $\ell_2$ loss. }
	\label{loss}
	\begin{tabular}{c|ccc}
		\toprule
		Loss&PSNR $\uparrow$&SSIM $\uparrow$&FID $\downarrow$\\
		\midrule
		$\mathcal L_{\rm FROT}$ &27.60&0.772&\underline{69.43}\\
		supervised $\mathcal \ell_2$ &\underline{27.69}&\underline{0.779}&78.69\\
		\midrule
		$\mathcal L_{\rm FROT}$+ $\ell_2$&\textbf{28.25}&\textbf{0.799}&\textbf{56.60}\\
		\bottomrule
	\end{tabular}
\end{table}
For image super-resolution, we evaluate the performance on the challenging DIV2K \cite{agustsson2017ntire} dataset, which consists of 800 (4x) LR and HR image pairs for training and 100 pairs for testing. 

The quantitative results in Table \ref{sr} show that the proposed RCOT outperforms the most recent generative methods. The qualitative gains of both distortion measures and perceptual quality are significant. Notably, even in unpaired setting the RCOT achieves the third-best performance in terms of the FID score. The visual examples in Figure \ref{vsr} and \ref{asr} show that the DPMs-based methods \cite{luo2023image, gao2023implicit} produce unnatural structures. The OT-based methods \cite{wang2022optimal, korotin2023neural} exhibit severe distortion. As a comparison, our RCOT restores a sharp image with more realistic details. 
\subsection{Ablation Study}
We primarily showcase the ablation studies on the TRC module and training loss components in the main body. \textbf{Four additional ablation experierments} are reported in Appendix \ref{ablation}.

\textbf{Importance of the TRC module.}
We validate the importance of the transport residual condition module on three tasks (SR on DIV2K \cite{agustsson2017ntire}, Deraining on Rain100L \cite{fan2019general}, and Denoising on Kodak24 \cite{franzen1999kodak}). We report the quantitative results with and without TRC in Table \ref{treg} and display visual comparisons in Figure \ref{freg}. The proposed TRC module yields an average gain of 1.75 dB of PSNR value over the basic model. With the TRC, the model restores images with better structures.
\begin{figure}[!h]
\vspace{-0.3cm}
	\setlength\tabcolsep{1pt}
	\renewcommand{\arraystretch}{0.5} 
	\centering
	\begin{tabular}{cccc}
		\includegraphics[width=0.225\linewidth]{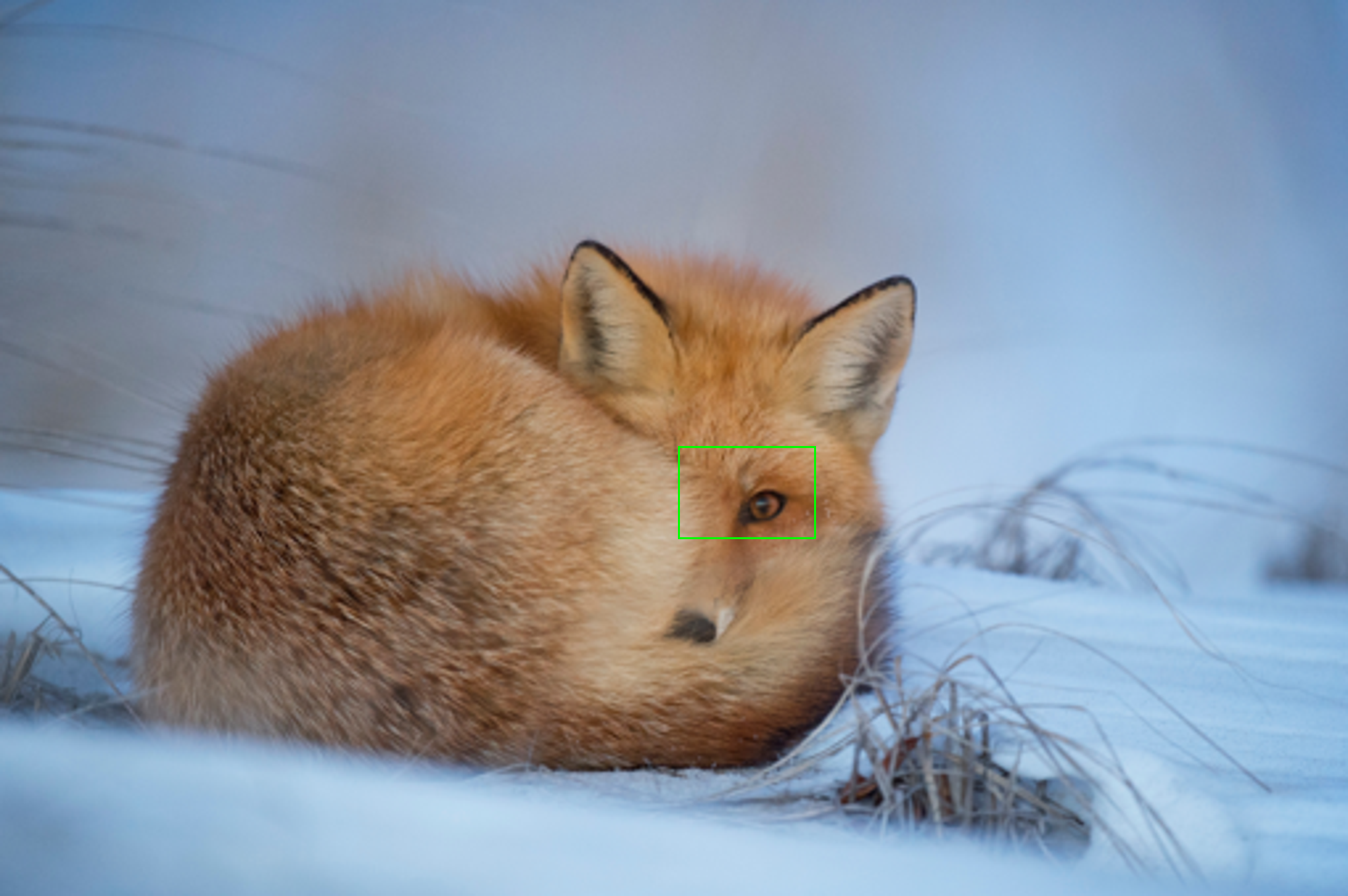}&
		\includegraphics[width=0.225\linewidth]{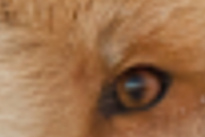}&
		\includegraphics[width=0.225\linewidth]{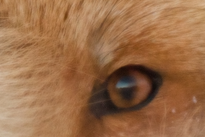}&
		\includegraphics[width=0.225\linewidth]{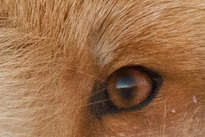}\\
		\includegraphics[width=0.225\linewidth]{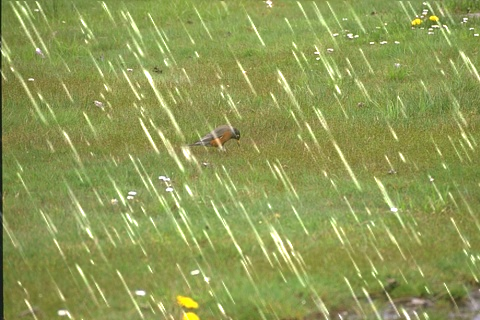}&
		\includegraphics[width=0.225\linewidth]{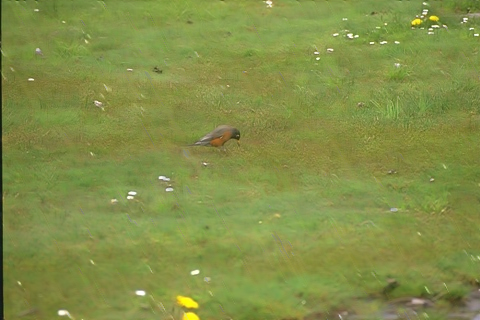}&
		\includegraphics[width=0.225\linewidth]{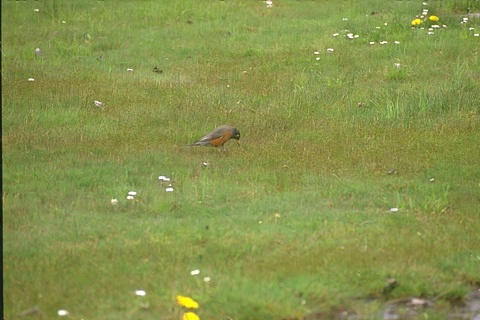}&
		\includegraphics[width=0.225\linewidth]{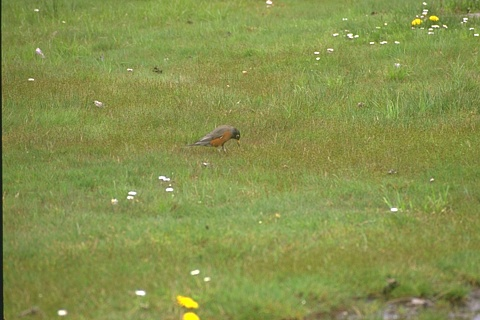}\\
		\includegraphics[width=0.225\linewidth]{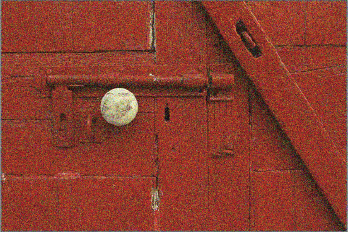}&
		\includegraphics[width=0.225\linewidth]{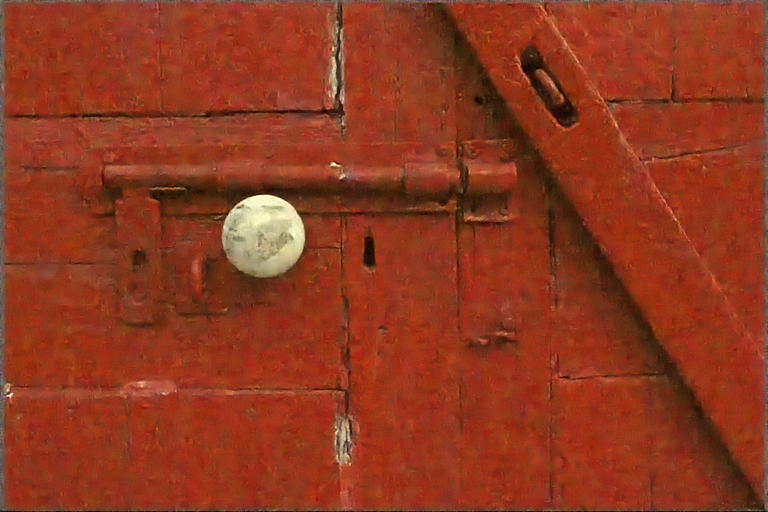}&
		\includegraphics[width=0.225\linewidth]{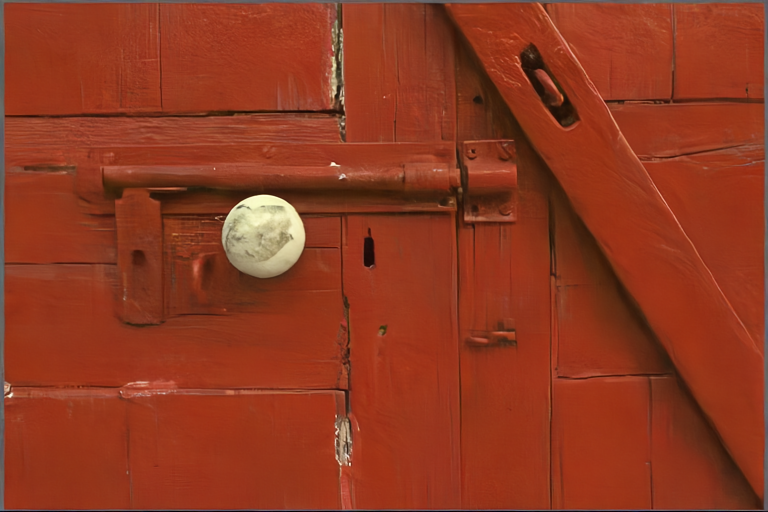}&
		\includegraphics[width=0.225\linewidth]{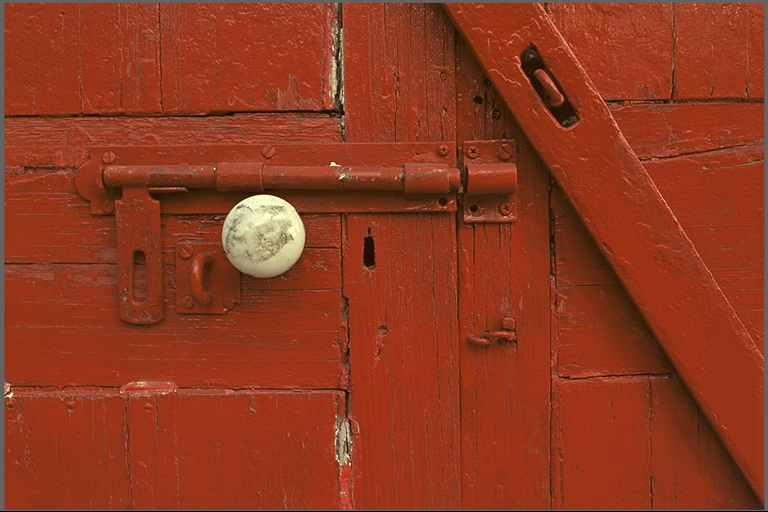}\\
		Degraded&w/o TRC&w/ TRC & Ground truth
	\end{tabular}
	\caption{Visual
		results to show the importance of TRC. The restored images with TRC in three tasks (SR, Deraining, Denoising)  contain better structures.}
  \vspace{-0.5cm}
	\label{freg}
\end{figure}

\textbf{Effect of the loss components}. To investigate the effect of different loss components (FROT and supervised $\ell_2$ loss), we train the transport map $T_\theta$ under different loss functions and evaluate its denoising performance. Table \ref{loss} reports the results on CBSD68 \cite{martin2001database} with noise level $\sigma=50$. When training under $\mathcal L_{\rm FROT}$, the performance is close to that under supervised $\ell_2$ loss. As a comparison, the training objective for the paired case, which integrates $\mathcal L_{\rm FROT}$ and supervised $\mathcal \ell_2 $, provides a significant gain to our model. These results validate the effectiveness of the proposed FROT under unpaired setting and show the importance of integrating supervised $\ell_2$ loss with $\mathcal L_{\rm FROT}$when the targets are available.
\vspace{-0.2cm}
\section{Conclusion}
\vspace{-0.2cm}
This paper proposed a novel Residual-Conditioned Optimal Transport (RCOT) approach to preserve perceptual structures while effectively removing the distortion, which treated image restoration as an OT problem and introduced the unique transport residual as a degradation-specific cue for both the transport cost and transport map. We first customized the Fourier residual-guided OT objective that exploits Fourier statistics of the degradation domain gap (represented by the residual). Based on FROT, we deduced a minimax problem that can be tackled to provide OT maps. Then we developed a two-pass RCOT map, in which the second pass generates refined results conditioned on the intermediate transport residual computed by the first pass for structure preservation. Extensive experiments demonstrated the effectiveness of RCOT for achieving competitive restoration performance, especially in terms of the ability to preserve structural content.  In the future, we are also interested in developing an all-in-one RCOT framework for restoration by properly integrating domain knowledge.

\section*{Limitation and Future Research} In this work, based on the empirical observation of the degradation domain gap in the frequency domain, we utilize handcrafted priors in the frequency domain to characterize the Fourier residual, which has been shown to be effective in the experiments. However, it may not be optimal for different degradations. Automatic, adaptive, and optimal learning for it will be explored in our future work.  We hope the method can be applied to all-in-one image restoration.
\section*{Impact Statement}
This work aims to advance the field of machine learning with applications in image restoration. It may be valuable to the research of optimal transport designed as a deep generative model working with low-quality data and has no ethical concerns as far as we know.
\section*{Acknowledgement}
This work was supported by National Key R\&D Program 2021YFA1003002, and NSFC (12125104, U20B2075, 623B2084, 12326615).

\bibliography{example_paper}
\bibliographystyle{icml2024}

\appendix
\onecolumn
\textbf{{\Large Appendix}}\medskip
\section{Proof}
\label{proof}
\begin{proposition}
	\label{OR}
	\text{(Saddle points provide OT maps).} For any optimal potential function $\varphi^*\in\arg\sup_\varphi\mathcal L(T,\varphi)$, it holds for the OT map $T^*$ (i.e., the transport map attaining the infimum of Monge's formulation \eqref{monge} under $\tilde{c}$) that
	\begin{align} T^*\in \mathop{\arg\min}_T\mathcal L(T,\varphi^*)\end{align}
\end{proposition}
\begin{proof}
	First, we give Monge's primal form of the FROT objective
	\begin{align}
		\mathbb M\text{-FROT}(\mathbb P,\mathbb Q)=\inf_{T\#\mathbb P=\mathbb Q}\int_Y\left[c(T(y),y)+g(\hat r(T))\right]d\mathbb P(y).
	\end{align}
	For the OT map  $T^*$, it holds $\hat r(T^*)=y-T^*(y)=r$, then we find
	\begin{align}
		\label{or}
		\mathcal L(T^*,\varphi^*)=	\int_X\varphi^*(x)d\mathbb Q(x)+\int_{Y}\big[c(T^*(y),y)+g( r)-\varphi^*(T^*(y))\big]d\mathbb P(y).
	\end{align}
	Using $T^*_{\#}\mathbb P=\mathbb Q$ and the change of variables $T^*(y)=x$, we can derive $$\int_Y\varphi^*(T^*(y))d\mathbb P(y)=\int_X\varphi^*(x)d\mathbb Q(x).$$
	Substituting this equality into (\ref{or}), we obtain
	\begin{align}\mathcal L(T^*,\varphi^*)=\int_{Y}\left[c(T^*(y),y)+g(r)\right]d\mathbb P(y)\nonumber=\inf_{T\#\mathbb P=\mathbb Q}\int_Y\left[c(T(y),y)+g(\hat r(T))\right]d\mathbb P(y)\nonumber=\mathbb M\text{-FROT}(\mathbb P,\mathbb Q).\nonumber \end{align}
\end{proof}
\section{Algorithm}
\vspace{-0.3cm}
\label{algorithm}
\begin{algorithm}[!h]
	\caption{RCOT Solver to compute the OT map.}  
	\textbf{Input}:  degraded  dataset $Y\sim \mathbb P$; high-quality   dataset $X\sim \mathbb Q$; transport network: $T_\theta$;\\ potential network: $\varphi_w$; the number of iterations of $\theta$ per iteration of $\omega$: $n_T$;
	
	\begin{algorithmic}[1]
		\label{algo}
		\WHILE{$\theta$ has not converged}
		\STATE Sample batches $\mathcal Y$ from $Y$, $\mathcal X$ from $X$;
		\STATE $\mathcal L(\omega)\leftarrow \frac{1}{|\mathcal Y|}\sum_{y\in \mathcal Y}\varphi_w(T_\theta(y))-\frac{1}{|\mathcal X|}\sum_{x\in \mathcal X}\varphi_w(x)$;
		\STATE	Update $w$ by using $\frac{\partial \mathcal L_\varphi}{\partial w}$;
		\FOR{$t=0,\cdots,n_T$}
		\STATE Compute $T_\theta(y)$ via (\ref{2p});
		\IF{there exist paired samples}
		\STATE $\mathcal L(\theta)\leftarrow\frac{1}{|\mathcal Y|}\sum_{y\in \mathcal Y}\left[c(y,T_\theta(y))+g(\hat r(T_\theta))-\varphi_w(T_\theta(y))\right]+\frac{\gamma}{|P|}\sum_{(y,x)\in P}\|T_\theta(y)-x\|^2$
		\ELSE
		\STATE $\mathcal L(\theta)\leftarrow\frac{1}{|\mathcal Y|}\sum_{y\in \mathcal Y}\left[c(y,T_\theta(y))+g(\hat r(T_\theta))-\varphi_w(T_\theta(y))\right]$;
		\ENDIF
		\STATE Update $\theta$ by using $\frac{\partial \mathcal L(\theta)}{\partial \theta}$;
		\ENDFOR
		\ENDWHILE
	\end{algorithmic} 
\end{algorithm} 
\section{Implementation Details and Baselines}
\label{details}
This section introduces the compared methods and detailed settings in our experiments. 

\textbf{Implementation details.} We train separate models for different tasks using the RMSProp optimizer with a learning rate of $1\times10^{-4}$ for the transport network $T_\theta$ and $0.5\times10^{-4}$ for the potential network $\varphi_w$. The inner iteration number $n_T$ is set to be 1. The learning rate is decayed by a factor of 10 after 100 epochs. In all experiments, the transport network $T_\theta$ uses the backbone in MPRNet \cite{Zamir2021MPRNet}.  The residual encoder consists of two CNN
down-sampling layers with residual channel attention block
(RCAB) \cite{zhang2018image}. In the FROT objective, $c(x,y)$ is suggested as $\ell_2$-norm. During training, we crop patches of size 256x256 as input and use a batch size of 4. All the experiments are conducted on the Pytorch framework with an NVIDIA 4090 GPU. For super-resolution, there is an extra preprocessing step. The LR images undergo bicubic rescaling to match the dimensions of their respective high-resolution counterparts. The source code will be released after the possible publication of our work. 

For the unpaired setting, although datasets that contain paired labels are utilized for training, we randomly shuffle the target $x$ and degraded input $y$ to ensure the loss are isolated from paired information, which is a common strategy \cite{wang2022optimal, korotin2023kernel, korotin2023neural} of unpaired training for restoration problems.

\textbf{Representative Compared methods.} For image denoising, deraining, and dehazing, we choose four most recent representative methods with state-of-the-art performance and two recent OT-based generative methods as competitors. They include MPRNet \cite{Zamir2021MPRNet}, Restormer \cite{Zamir2021Restormer}, IR-SDE \cite{luo2023image}, OTUR \cite{wang2022optimal}, NOT \cite{korotin2023neural}, and PromptIR \cite{potlapalli2023promptir}. 

MPRNet \cite{Zamir2021MPRNet} and Restormer \cite{Zamir2021Restormer} respectively specialize two backbones for restoration. Notably, Restormer \cite{Zamir2021Restormer} designs an efficient Transformer model by making several key designs in the building blocks (multi-head attention and feed-forward network) such that it can capture long-range pixel interactions, while still remaining applicable to large images. It achieves state-of-the-art results on several image restoration tasks.

PromptIR \cite{potlapalli2023promptir} utilizes learnable task prompts to incorporate the degradation-specific knowledge in the Transformer model to achieve task-aware restoration. We train the PromptIR model \textbf{in single task setting}.

For image super-resolution, we choose the most recent state-of-the-art generative diffusion-based  \cite{luo2023image, gao2023implicit}, sophisticated transformer-based method Restormer \cite{Zamir2021Restormer}, and OT-based methods \cite{korotin2023neural, wang2022optimal} as competitors for comparison.

\section{Additional Ablation Studies}
\label{ablation}
\subsection{Versatility and generalizability beyond specific network designs} 

The RCOT can be easily applied to different network architectures or frameworks. The two-pass TRC module is a plug-in module (Figure \ref{model}), allowing us to use any architecture as a base model to generate the restored result, and then use this result to calculate the residual. We have now included a comparison on Rain100L dataset \cite{fan2019general} between the MPRNet \cite{Zamir2021MPRNet}, NAFNet \cite{chu2022nafssr}, Restormer \cite{Zamir2021Restormer} methods and the corresponding methods with the proposed TRC module.

\begin{table*}[!h]
	\centering
	\caption{The influence of the transport residual condition (TRC) module being integrated into different network backbones. The metrics are presented as PSNR ($\uparrow$)/SSIM ($\uparrow$)/LPIPS ($\downarrow$)/FID ($\downarrow$) values.}
	\label{trc}
	\setlength{\tabcolsep}{15pt}
	\resizebox{\textwidth}{!}{
		\begin{tabular}{c|ccc}
			\toprule
			Method  & MPRNet \cite{Zamir2021MPRNet} & NAFNet \cite{chu2022nafssr} & Restormer \cite{Zamir2021Restormer} \\
			\midrule
			w/o TRC & 34.95/0.964/0.0387/21.61  & 35.58/0.969/0.0355/18.35 & 36.74/0.978/0.0226/13.29  \\
			\midrule
			w/ TRC& \textbf{36.78/0.972/0.0145/12.58}&\textbf{37.10/0.976/0.0134/11.86}&\textbf{38.22/0.984/0.0102/7.01} \\
			\bottomrule
	\end{tabular}}
\end{table*}

The results in Table \ref{trc} show that the TRC module brings a meaningful boost to three SOTA network architectures, which validates its versatility and generalizability beyond specific network designs.
\subsection{Effect of different conditions}
In this paper, we treat the transport residual as a degradation-specific condition by encoding its embedding to adaptively enhance the representation of the restoration. To better demonstrate the benefits, we also try conditioning directly on the output restored image $\hat x_0$ of the base model. The average qualitative results on (SR, Deraining, and Denoising) are reported in Table \ref{condition} which sustains our claim.

 \begin{table*}[!h]
	\centering
	\caption{The qualitative results of the model with different conditions for the second pass restoration.}
	\label{condition}
		\begin{tabular}{c|ccc}
			\toprule
			Method  & w/o condition &conditioned on $\hat x_0$&conditioned on residual $\hat r_0$\\
			\midrule
			PSNR&29.97&30.39&\textbf{31.72}\\
                SSIM&0.843&0.848&\textbf{0.870}\\
                FID&32.47&26.52&\textbf{14.34}\\
			\bottomrule
	\end{tabular}
\end{table*}
\subsection{Effect of the Fourier residual penalty in FROT objective} We investigate the effect of the proposed FROT cost on four restoration tasks (Denoising on Kodak24 \cite{franzen1999kodak} with noise level $\sigma=50$, SR on DIV2K \cite{agustsson2017ntire}, Deraining on Rain100L \cite{fan2019general}, and Dehazing on SOTS \cite{li2018benchmarking}). In Table \ref{AFROT}, we compare the performance of the models trained under the unpaired FROT cost (w/ $g(\hat r)$ ) and regular OT cost (w/o $g(\hat r)$). The results show that the Fourier residual penalty, integrating degradation-specific knowledge into the transport cost,   brings meaningful gains.
\begin{table}[!htbp]
	\centering
	\caption{Results of the OT cost w/ $g(\hat r)$ (FROT) and w/o $g(\hat r)$. The models are trained by only using unpaired FROT, in which $\ell_2$-norm is chosen as $c(x,y)$ and $g(\hat r)=\|\mathcal F(\hat r)\|_1$ for SR, draining, and dehazing while utilizing $g(\hat r)=\|\mathcal F(\hat r)\|_2$ for denoising.}
	\label{AFROT}
 \vspace{0.3cm}
	\setlength{\tabcolsep}{3pt}
	\begin{tabular}{c|cccc}
		\toprule
		Task&Transport cost&~PSNR $\uparrow$& SSIM $\downarrow$&FID $\downarrow$\\\midrule
		\multirow{2}{*}{Denoising}&w/o $g(\hat r)$&28.37&0.780&69.12\\
		&w/ $g(\hat r)$&\textbf{28.64}&\textbf{0.792}&\textbf{63.27}\\\midrule
		\multirow{2}{*}{SR}&w/o $g(\hat r)$&25.96&0.723&11.41\\
		&w/ $g(\hat r)$&\textbf{26.78}&\textbf{0.758}&\textbf{4.51}\\\midrule
		\multirow{2}{*}{Deraining}&w/o $g(\hat r)$&35.69&0.945&20.13\\
		&w/ $g(\hat r)$&\textbf{36.22}&\textbf{0.972}&\textbf{12.59}\\\midrule
		\multirow{2}{*}{Dehazing}&w/o $g(\hat r)$&29.72&0.953&15.56\\
		&w/ $g(\hat r)$&\textbf{30.34}&\textbf{0.965}&\textbf{12.57}\\\bottomrule
	\end{tabular}
	\vspace{-0.3cm}
\end{table}
\subsection{Effect of different regularizers for the Fourier residual}
We conduct an extra ablation study on the regularizer for the Fourier residual. The results are reported in Table \ref{regularizer}.

\begin{table}[H]
\centering
\caption{Comparison of different regularizers for the Fourier residual.}
\label{regularizer}
\begin{tabular}{c|ccc}
\toprule
Regularizer & $L_{0.5}$ & $\ell_2$ & $L_{1}$ \\ \midrule
SR          & \textbf{26.82}/0.756 & 26.49/0.747 & 26.78/\textbf{0.758} \\
Denoising   & 28.52/0.782 & \textbf{28.79/0.795} & 28.64/0.792 \\
Deraining   & 36.01/0.968 & 35.80/0.948 & \textbf{36.22/0.972} \\
Dehazing    & 30.18/0.958 & 29.80/0.954 & \textbf{30.34/0.965} \\ \bottomrule
\end{tabular}
\end{table}
For the denoising task, the $\ell_2$ regularizer is more suitable for the Fourier residual since the Gaussian noise is basically from a Gaussian distribution, which is equivalent to the  $\ell_2$ regularization. The $\ell_1$ sparsity regularizer applies to other degradations.

\section{Sensitivity to the Supervised Trade-off Parameter $\gamma$}
We test the sensitivity of RCOT to the trade-off parameter $\gamma$ on SR tasks. Figure \ref{sen} shows that when $\gamma=1\times 10^4$, RCOT's performance might approach its limit. Interestingly, the distortion measures, i.e., PSNR and SSIM, go steady when $\gamma$ increases after $1\times10^3$. However, the FID score seems to turn bad if $\gamma$ keeps increasing. Based on this result, we fix $\gamma=10^4$ in our experiments.
\begin{figure}[!h]
	\centering
	\includegraphics[width=0.8\linewidth]{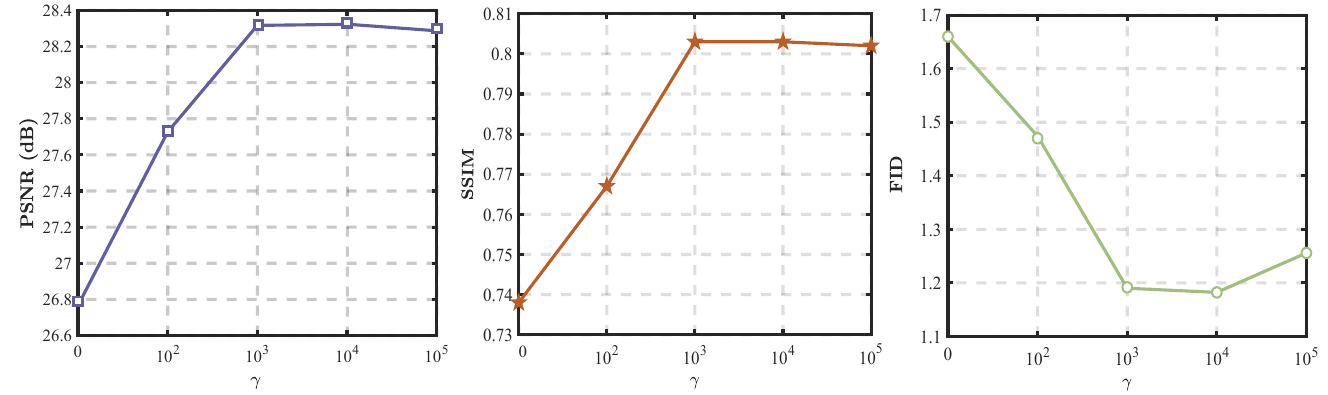}
	\caption{Test of the sensitivity of RCOT to the parameter $\gamma$ on the DIV2K \cite{agustsson2017ntire} dataset for SR task.}
	\label{sen}
\end{figure}

\section{Training Cost Curves for Three tasks}
In Figure \ref{curve}, we display the cost curves over three tasks (i.e., SR, Deraining, and Denoising) of $T_\theta$ and $\varphi_w$ in the training process. The $T_\theta$ cost curve is normalized to $[0,1]$. $\varphi_w$ cost is scaled to $[0,1]$ and then take the negative.
\begin{figure}[!h]
	\centering
	\includegraphics[width=0.8\linewidth]{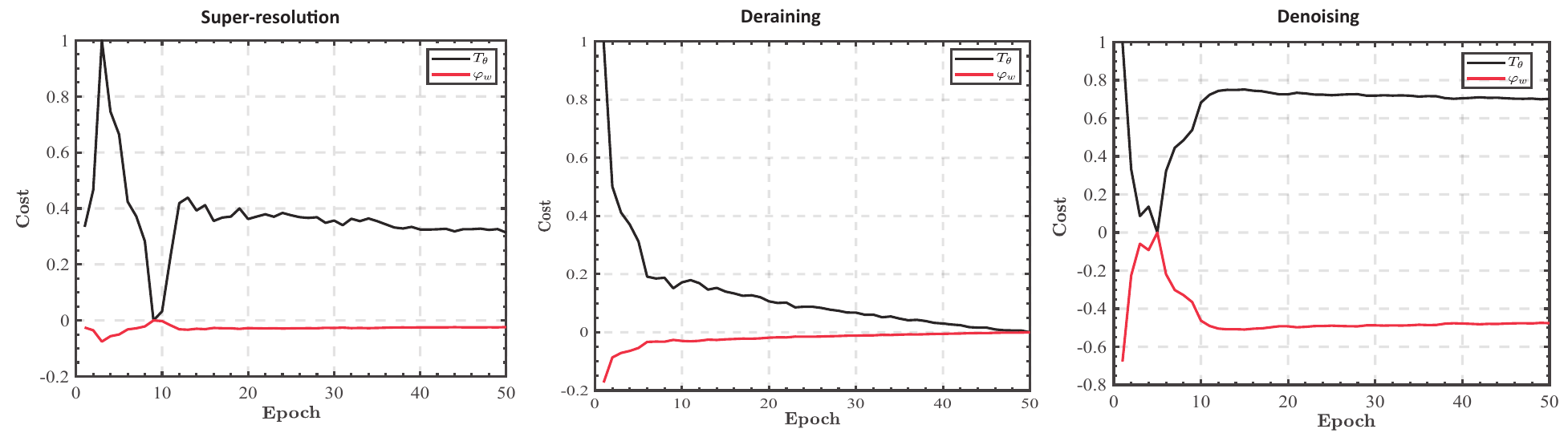}
	\caption{The cost curves for three tasks. The cost of $T_\theta$ is scaled to $[0,1]$. The cost of $\varphi_w$ is scaled to $[0,1]$ and then takes the negative.}
	\label{curve}
 \vspace{-0.3cm}
\end{figure}
\section{Additional Evaluations and Comparisons}
\label{aexp}

\textbf{Parameter quantity and computational complexity.}
We report the parameter quantity and computational complexity in Table \ref{para}. We compare with the two most recent SOTA methods IRSDE and PromptIR. The results show that our approach is based on a relatively lightweight model, which sustains its effectiveness.
\begin{SCtable*}[][!h]
	\centering
	\caption{Comparison of the number of parameters, model computational efficiency, and inference time. The metrics computed on Rain100L \cite{fan2019general} with a single NVIDIA 4090 GPU.}
	\label{para}
		\begin{tabular}{c|ccc}
			\toprule
			Method  & IR-SDE  & PromptIR & RCOT \\
			\midrule
			\#Param&36.2M&33M&14.2M\\
                Flops&117G&158G&142G\\
                Inference time&9.09s&3.25s&1.06s\\
			\bottomrule
	\end{tabular}
	\vspace{-0.3cm}
\end{SCtable*}
\begin{table*}[!h]
	\centering
	\caption{Cross-domain comparisons. The models are trained on synthetic datasets and tested on real-world datasets (real rainy SPANet \cite{Wang_2019_CVPR} and real hazy O-HAZE \cite{ancuti2018haze}).  ($*$) indicates an unpaired setting.}
	\label{cross}
	\setlength{\tabcolsep}{15pt}
	\renewcommand{\arraystretch}{1.2}
	\resizebox{\textwidth}{!}{
		\begin{tabular}{@{}ccccccccc@{}}
			\toprule
			\multirow{2}{*}{Method}&\multicolumn{4}{c}{Real rainy SPANet~\cite{Wang_2019_CVPR}} & \multicolumn{4}{c}{Real hazy O-HAZE~\cite{ancuti2018haze}}\\ \cmidrule(lr){2-5}  \cmidrule(lr){6-9} 
			& PSNR ($\uparrow$) & SSIM ($\uparrow$) & LPIPS ($\downarrow$) & FID ($\downarrow$) & PSNR ($\uparrow$) & SSIM ($\uparrow$) & LPIPS ($\downarrow$) & FID ($\downarrow$) \\
			\midrule 
			NOT$^*$ \cite{korotin2023neural}& 31.22 &0.882&0.040&50.23& 17.28 & 0.698& 0.215 &198.56 \\
			OTUR$^*$  \cite{wang2022optimal}& 36.52& 0.932 & 0.025 &33.21 & 17.45& 0.712& 0.202 &199.78\\
			MPRNet \cite{Zamir2021MPRNet}& 37.38 & 0.943 & 0.031 & 36.37 & 21.28 & 0.765& 0.244 &228.67\\
			Restormer \cite{Zamir2021Restormer}& 37.86 & 0.951 & 0.022 & 33.65 & \underline{24.56}& 0.788& 0.233 &196.33\\
			IR-SDE \cite{luo2023image}& \underline{38.29} & \underline{0.966} & \underline{0.013} & \underline{19.52} & 24.53 & \underline{0.796}& \underline{0.169} &\underline{186.44}\\
			PromptIR \cite{potlapalli2023promptir}& 37.23 & 0.947 & 0.019 &  34.29 & 24.32& 0.776 & 0.240&  205.66\\
			\midrule
			RCOT$^*$ & 37.37 & 0.948 & 0.015 & 21.20 & 22.35 & 0.775& 0.178 &189.72\\
			
			RCOT& \textbf{40.02} & \textbf{0.972} & \textbf{0.009} & \textbf{15.66} & \textbf{26.59} & \textbf{0.827}& \textbf{0.148} &\textbf{165.12} \\ 
			\bottomrule
	\end{tabular}}
 \vspace{-0.6cm}
\end{table*}
\textbf{Cross-domain evaluations.}
We conduct two representative cross-domain evaluations to compare the generalizability of RCOT and other methods. Specifically, we train the deraining models on the synthetic Rain13K dataset \cite{fu2017removing, li2016rain, yang2017deep, zhang2018density, zhang2019image} and then test their performances on the real rainy dataset SPANet \cite{Wang_2019_CVPR}. Then we train the dehazing models on the synthetic SOTS dataset and then test their performances on the real hazy dataset O-HAZE \cite{ancuti2018haze}. Table \ref{cross} reports the quantitative results. The results show that the compared methods generally exhibit a declined performance in the presence of a domain gap. However, RCOT still achieves notable performance under all metrics, which sustains the generalizability of RCOT over real-world datasets.

\textbf{Evaluation of a single model for multiple degradations.}
We are interested in extending RCOT to realize all-in-one restoration in the future, which is to train a single model for multiple degradations. Since the Fourier residual regularizer in RCOT is task-specific, there is still a gap towards the all-in-one target. Nevertheless, the regularizers for deraining and dehazing are set to be $\ell_1$ sparsity. Hence we train our model with a combination of rainy data (Rain100L training set) and noisy data (BSD400 and WED with $\sigma=25$). Then we test on Kodak24 with noise level $\sigma=25$ and Rain100L test set. The quantitative results are reported in Table \ref{all}, in which  PromptIR \cite{potlapalli2023promptir}, the most recent representative all-in-one method is chosen as a competitor. 
The results show that our RCOT exhibits a better capability of handling multiple degradations. The reason behind the reason should be our residual embedding contains richer degradation-specific knowledge (e.g., degradation type and level) as compared to the learnable visual prompt in PromptIR \cite{potlapalli2023promptir}.
\begin{table}[!h]
\vspace{-0.7cm}
\centering
\caption{Evaluation of a single model for both denoising and deraining.}
\label{all}
\begin{tabular}{c|cc}
\toprule
Method & Rain100L & Kodak24\\ \midrule
PromptIR \cite{potlapalli2023promptir} &36.79/0.974&31.25/0.872\\
RCOT (ours)   & \textbf{38.02/0.984}&\textbf{31.82/0.879} \\
\bottomrule
\end{tabular}
\vspace{-0.7cm}
\end{table}
\subsection{Additional comparison with GAN-based methods}
\vspace{-0.2cm}
Given that the potential network $\varphi$ can be considered as one discriminator, a comparison with GAN-based approaches could provide valuable insights into the effectiveness and uniqueness of the proposed method in terms of image restoration. We compare RCOT with the most recent GAN-based restoration methods for deraining (on synthetic Rain100L and real SPANet) and SR (on DIV2K). 
\begin{table*}[!h]
\centering
\caption{Comparison with GAN-based methods.}
\label{all}
\begin{tabular}{c|cc}
\toprule
Method & Rain100L & Kodak24\\ \midrule
CycleGAN \cite{zhu2017unpaired} &32.03/0.889/0.052/37.26&28.79/0.923/0.036/35.25\\
DeCyGAN \cite{wei2021deraincyclegan}  & 32.75/0.893/0.047/32.15&34.78/0.929/0.042/32.67\\
RCOT &\textbf{37.27/0.980/0.012/7.97}&\textbf{43.77/0.993/0.008/9.52}\\
\bottomrule
\end{tabular}
\quad\\
\quad\\
\begin{tabular}{c|c}
\toprule
Method & DIV2K\\ \midrule
ESRGAN \cite{wang2018esrgan} &26.22/0.752/0.112/6.59\\
RankSRGAN \cite{zhang2019ranksrgan}  & 26.55/0.75/0.117/9.52\\
RCOT &\textbf{28.41/0.804/0.114/1.22}\\
\bottomrule
\end{tabular}
\end{table*}

 \section{Additional Visual Results}
\subsection{Additional Comparison on Real Rainy Dataset}
\label{com}
We evaluate the performance of the compared methods on the real rainy dataset SPANet \cite{Wang_2019_CVPR}. This dataset contains 27.5K paired rainy and rain-free images for training, and 1, 000 paired images for testing.
\begin{figure*}[!h]
	\setlength\tabcolsep{1pt}
	\renewcommand{\arraystretch}{0.5} 
	\centering
	\begin{tabular}{cccccccc}
		\includegraphics[width=0.120\linewidth]{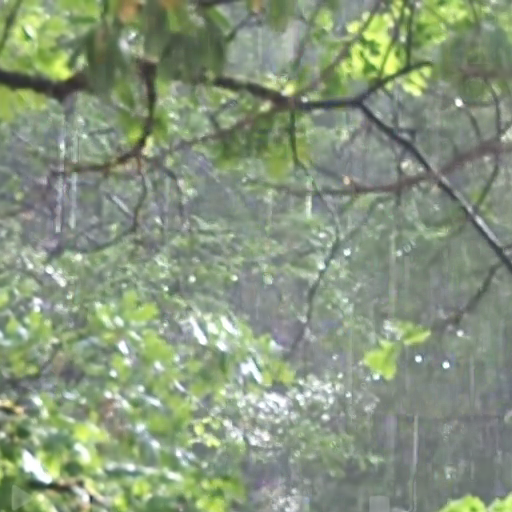}&
		\includegraphics[width=0.120\linewidth]{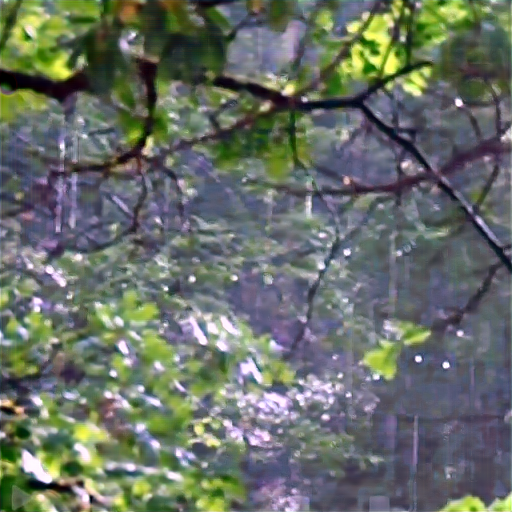}&
		\includegraphics[width=0.120\linewidth]{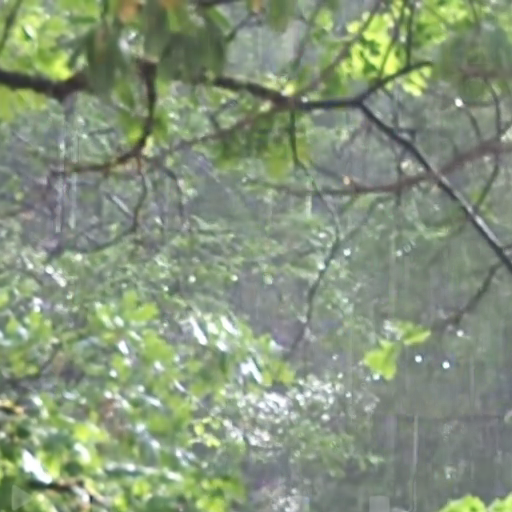}&
		\includegraphics[width=0.120\linewidth]{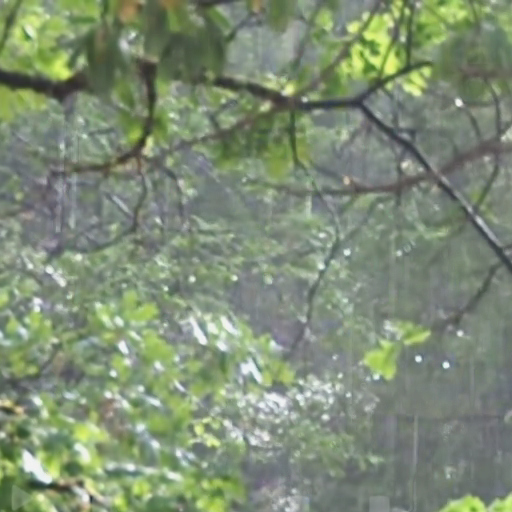}&
		\includegraphics[width=0.120\linewidth]{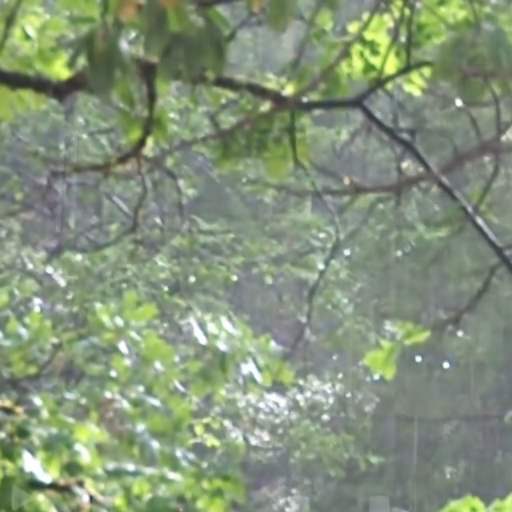}&
		\includegraphics[width=0.120\linewidth]{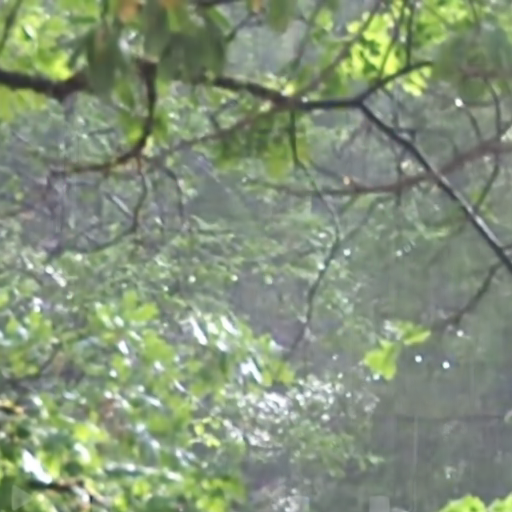}&
		\includegraphics[width=0.120\linewidth]{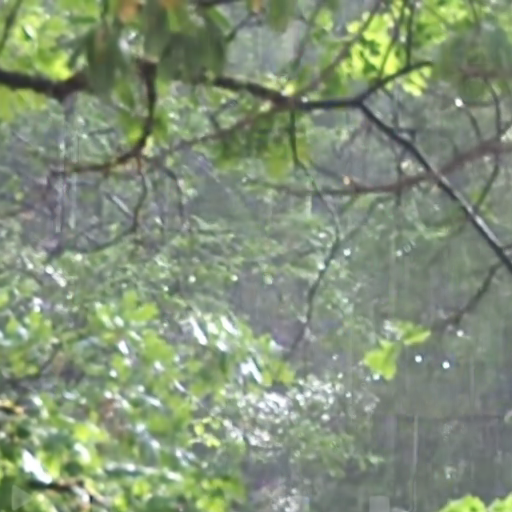}&
		\includegraphics[width=0.120\linewidth]{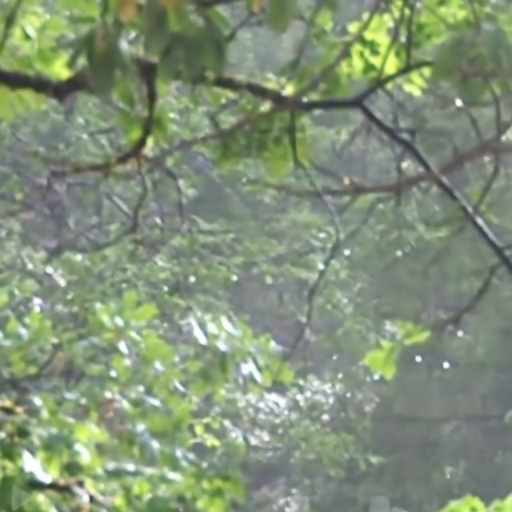}\\
		\includegraphics[width=0.120\linewidth]{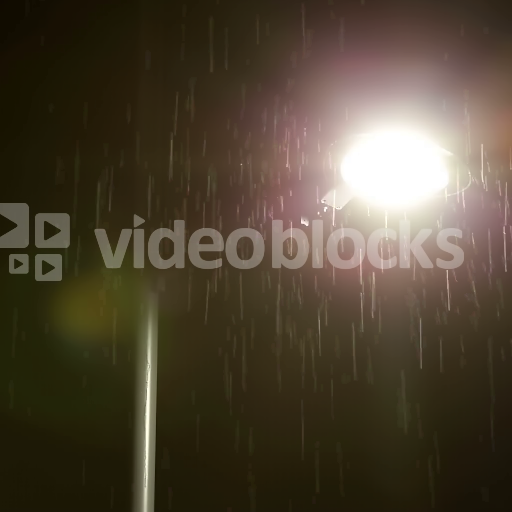}&
		\includegraphics[width=0.120\linewidth]{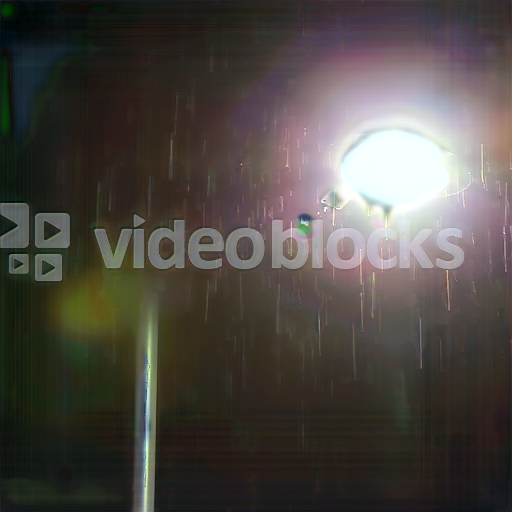}&
		\includegraphics[width=0.120\linewidth]{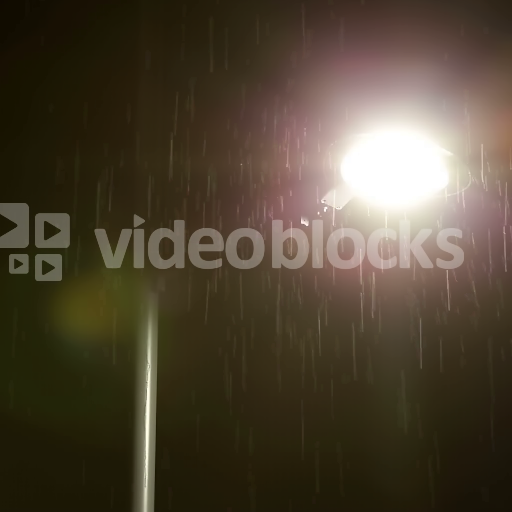}&
		\includegraphics[width=0.120\linewidth]{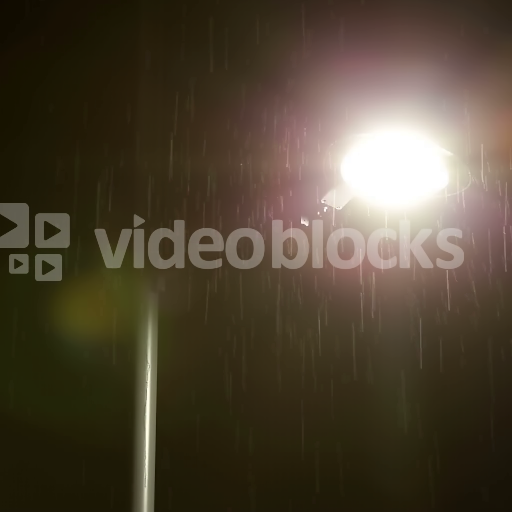}&
		\includegraphics[width=0.120\linewidth]{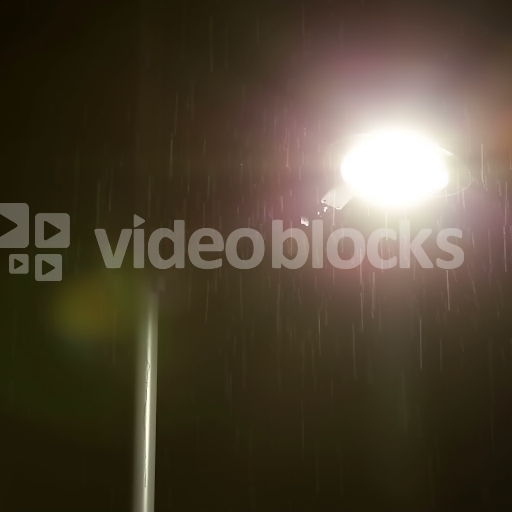}&
		\includegraphics[width=0.120\linewidth]{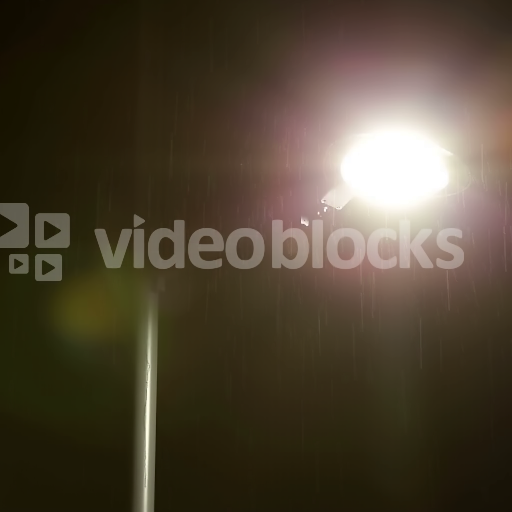}&
		\includegraphics[width=0.120\linewidth]{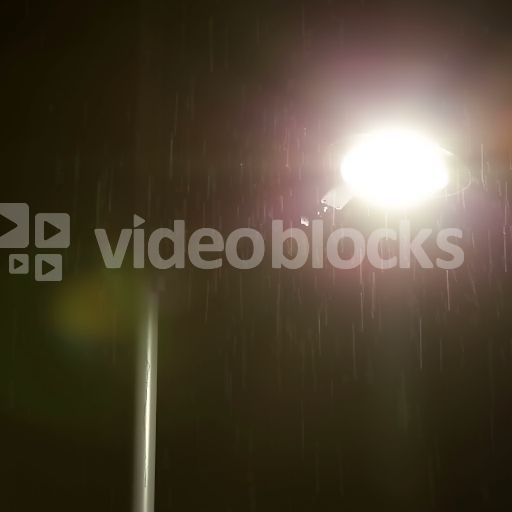}&
		\includegraphics[width=0.120\linewidth]{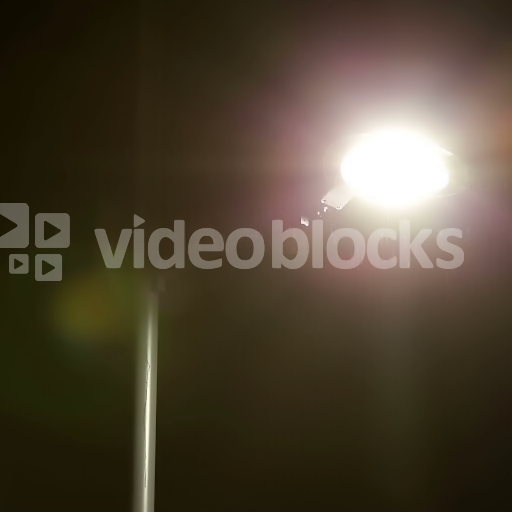}\\
		\includegraphics[width=0.120\linewidth]{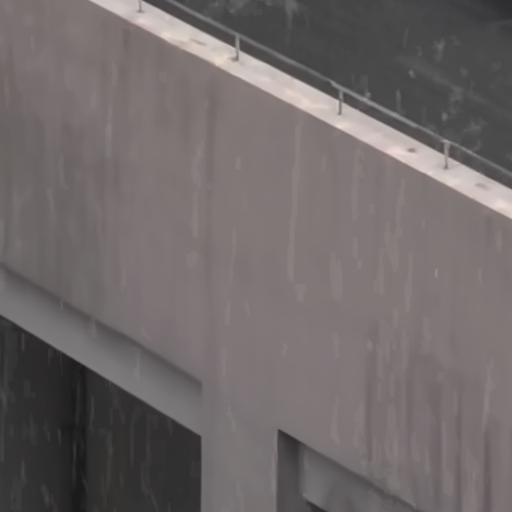}&
		\includegraphics[width=0.120\linewidth]{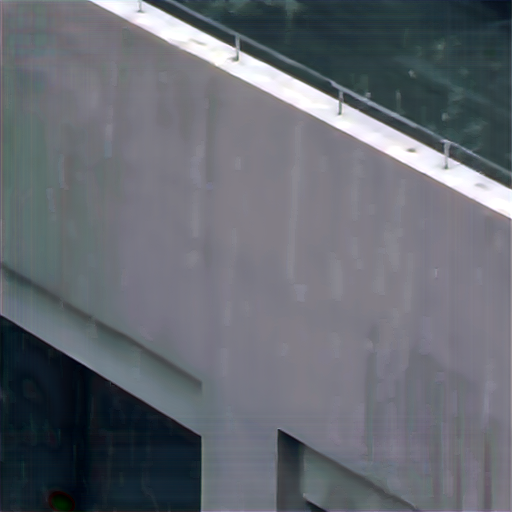}&
		\includegraphics[width=0.120\linewidth]{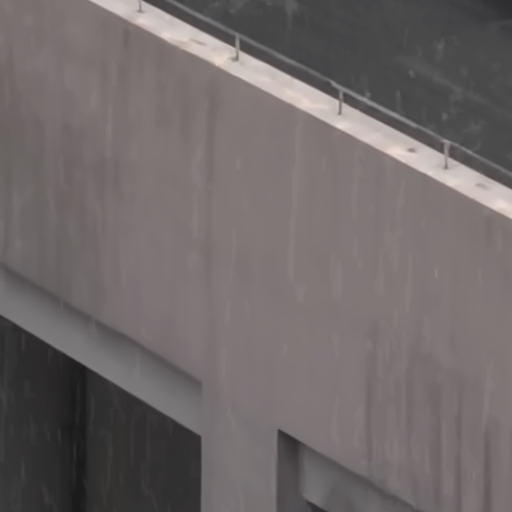}&
		\includegraphics[width=0.120\linewidth]{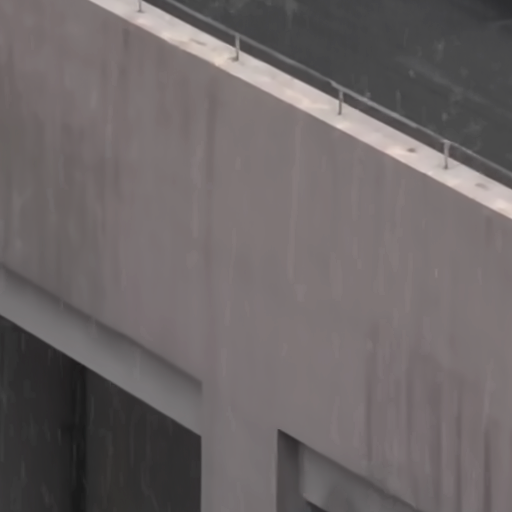}&
		\includegraphics[width=0.120\linewidth]{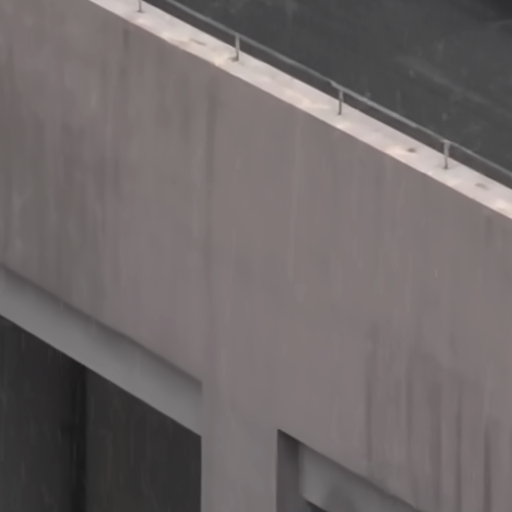}&
		\includegraphics[width=0.120\linewidth]{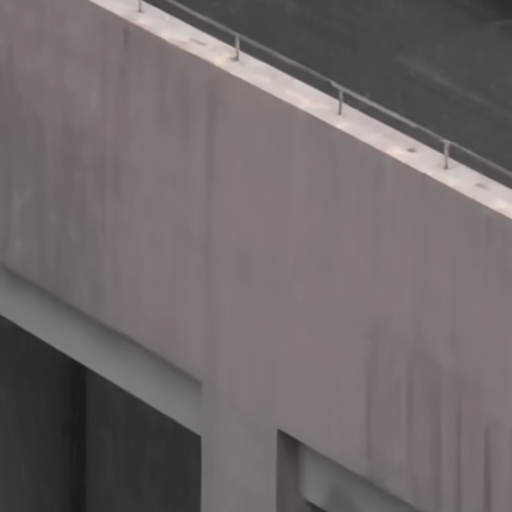}&
		\includegraphics[width=0.120\linewidth]{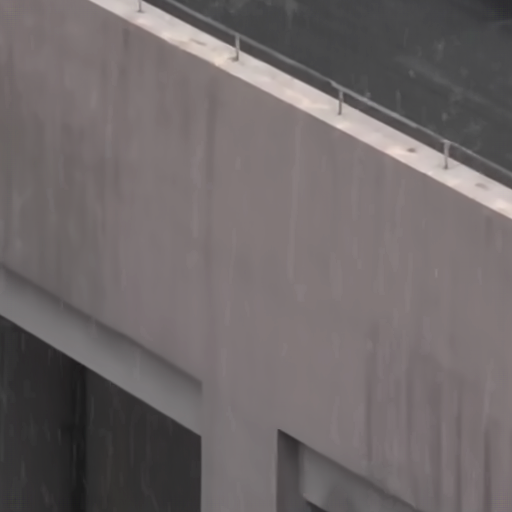}&
		\includegraphics[width=0.120\linewidth]{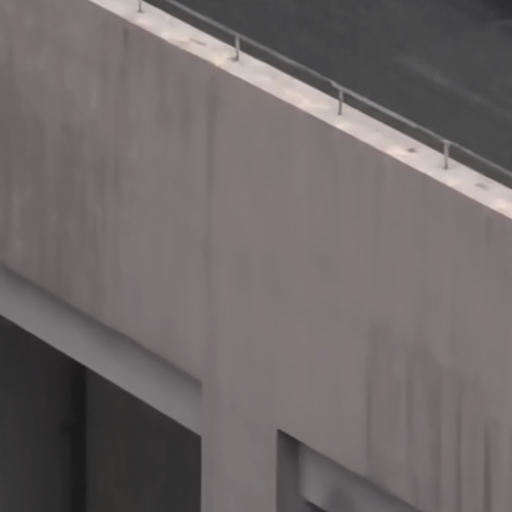}\\
		\includegraphics[width=0.120\linewidth]{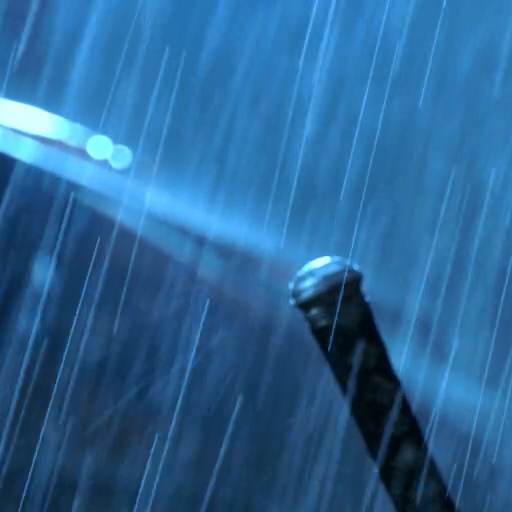}&
		\includegraphics[width=0.120\linewidth]{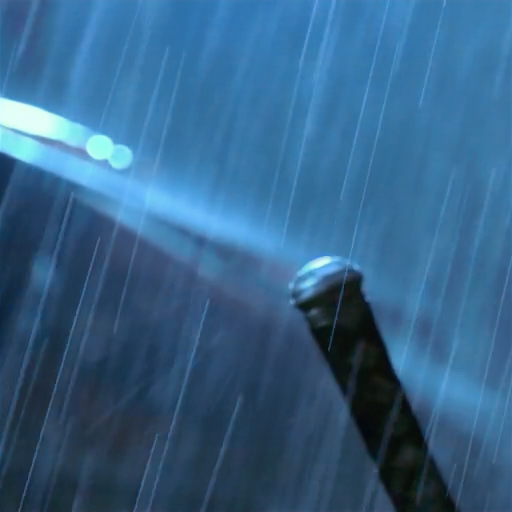}&
		\includegraphics[width=0.120\linewidth]{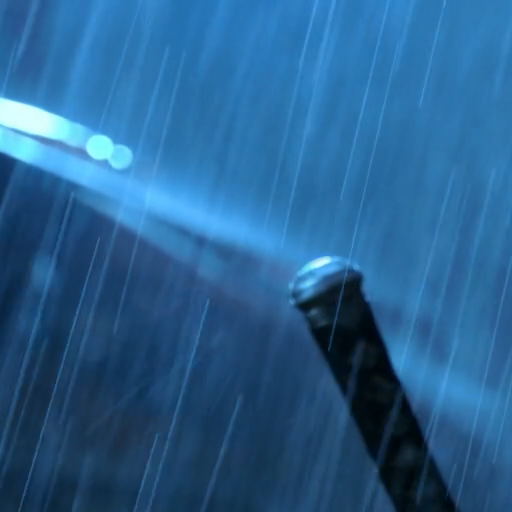}&
		\includegraphics[width=0.120\linewidth]{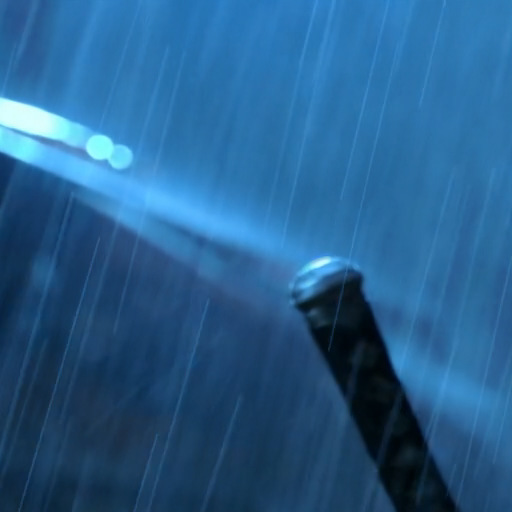}&
		\includegraphics[width=0.120\linewidth]{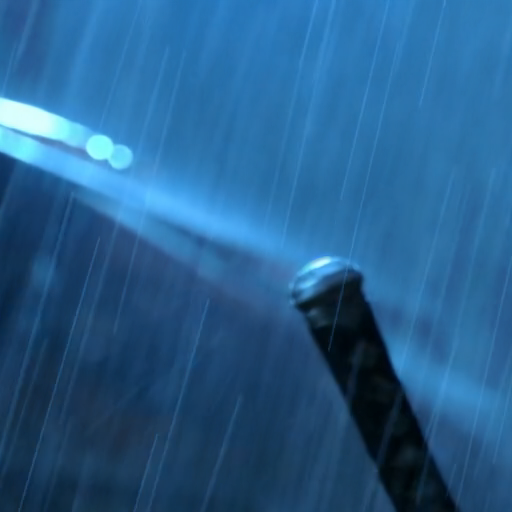}&
		\includegraphics[width=0.120\linewidth]{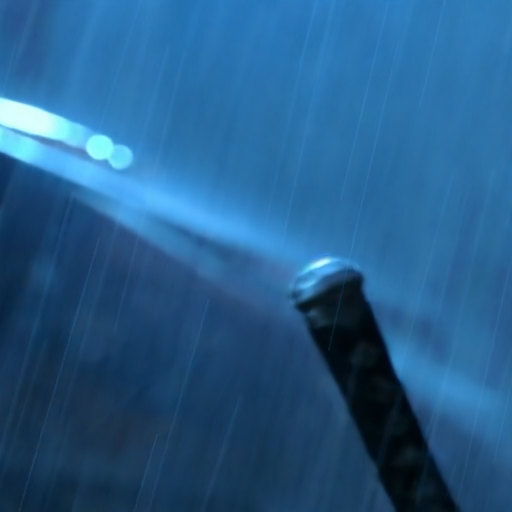}&
		\includegraphics[width=0.120\linewidth]{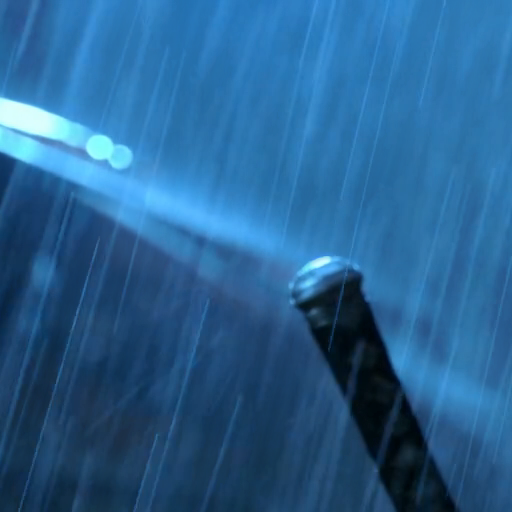}&
		\includegraphics[width=0.120\linewidth]{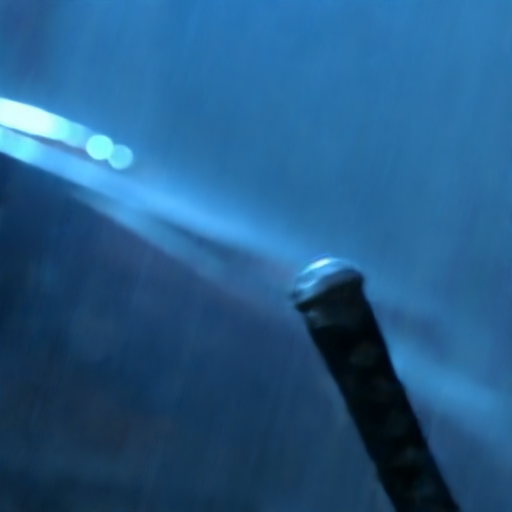}\\
		Rainy&NOT&OTUR&MPRNet&Restormer&IR-SDE&PromptIR&RCOT\\
	\end{tabular}
	\caption{Visual comparison of deraining on the real dataset SPANet \cite{Wang_2019_CVPR}. The RCOT reproduces rain-free images with better structural details.}
	\label{vrealrain}
 \vspace{-0.3cm}
\end{figure*} 
\subsection{Additional Visual Comparison on Real Hazy Dataset}
\begin{figure*}[!h]
	\setlength\tabcolsep{1pt}
	\renewcommand{\arraystretch}{0.5} 
	\centering
	\begin{tabular}{cccccccc}
		\includegraphics[width=0.120\linewidth]{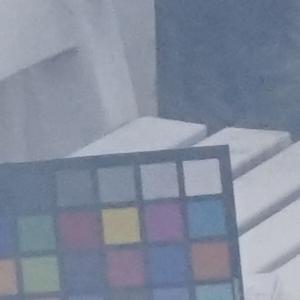}&
		\includegraphics[width=0.120\linewidth]{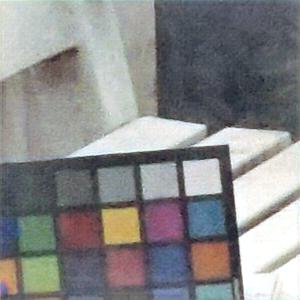}&
		\includegraphics[width=0.120\linewidth]{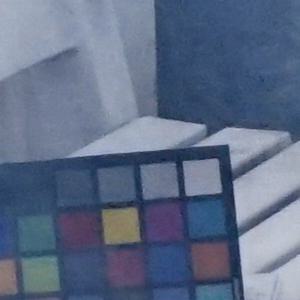}&
		\includegraphics[width=0.120\linewidth]{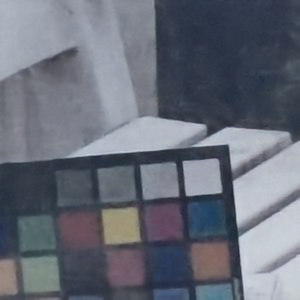}&
		\includegraphics[width=0.120\linewidth]{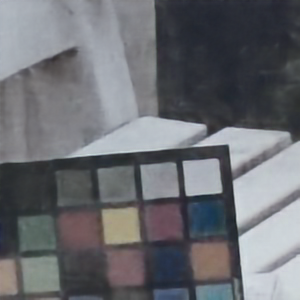}&
		\includegraphics[width=0.120\linewidth]{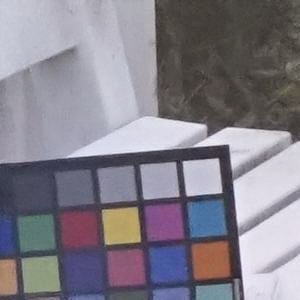}&
		\includegraphics[width=0.120\linewidth]{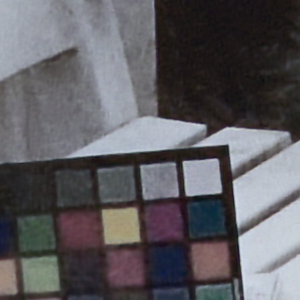}&
		\includegraphics[width=0.120\linewidth]{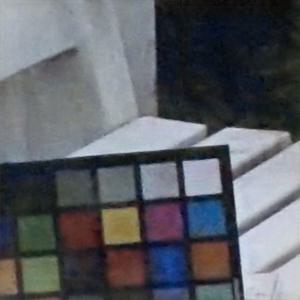}\\
		\includegraphics[width=0.120\linewidth]{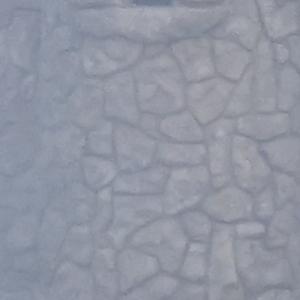}&
		\includegraphics[width=0.120\linewidth]{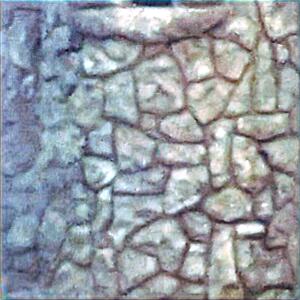}&
		\includegraphics[width=0.120\linewidth]{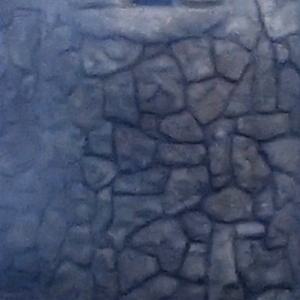}&
		\includegraphics[width=0.120\linewidth]{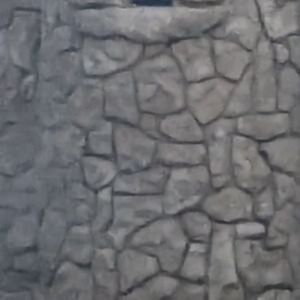}&
		\includegraphics[width=0.120\linewidth]{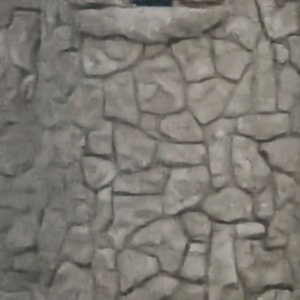}&
		\includegraphics[width=0.120\linewidth]{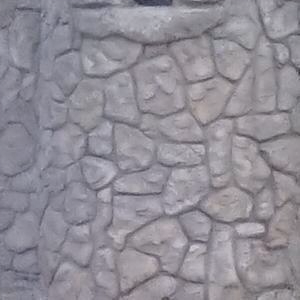}&
		\includegraphics[width=0.120\linewidth]{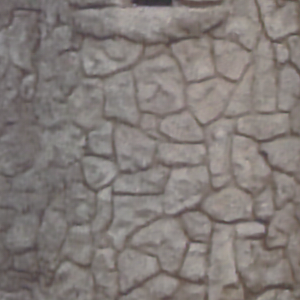}&
		\includegraphics[width=0.120\linewidth]{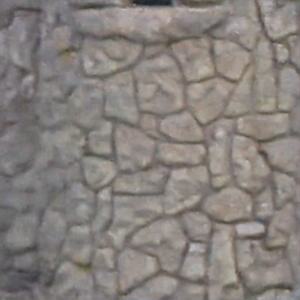}\\
		\includegraphics[width=0.120\linewidth]{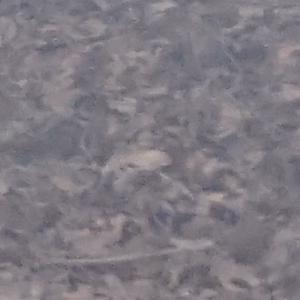}&
		\includegraphics[width=0.120\linewidth]{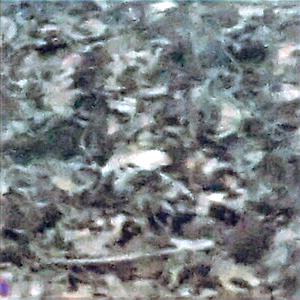}&
		\includegraphics[width=0.120\linewidth]{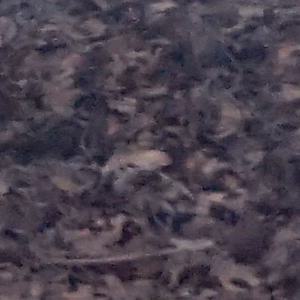}&
		\includegraphics[width=0.120\linewidth]{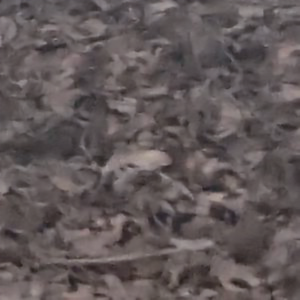}&
		\includegraphics[width=0.120\linewidth]{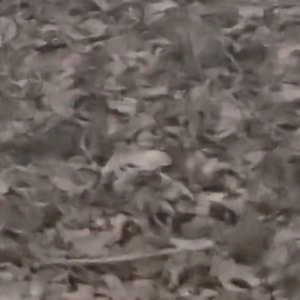}&
		\includegraphics[width=0.120\linewidth]{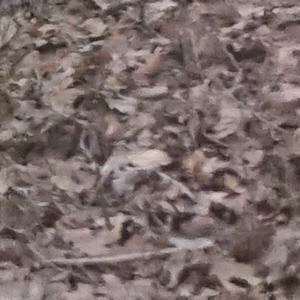}&
		\includegraphics[width=0.120\linewidth]{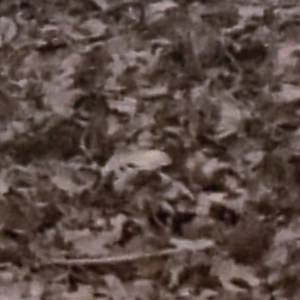}&
		\includegraphics[width=0.120\linewidth]{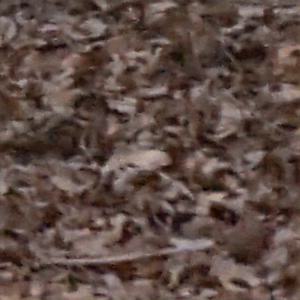}\\
		Hazy&NOT&OTUR&MPRNet&Restormer&IR-SDE&PromptIR&RCOT\\
	\end{tabular}
	\caption{Visual comparison of dehazing on the real dataset O-HAZE \cite{ancuti2018haze}. The RCOT reproduces haze-free images with more faithful color.}
	\label{vrealhaze}
\end{figure*} 

\label{comh}

\subsection{More visual examples}
\label{more}
Figures \ref{anoisy}, \ref{arain}, and \ref{asr} display more visual results of the compared methods. 
\begin{figure*}[!h]
	\setlength\tabcolsep{1pt}
	\renewcommand{\arraystretch}{0.5} 
	\centering
	\begin{tabular}{cccccccc}
		\includegraphics[width=0.120\linewidth]{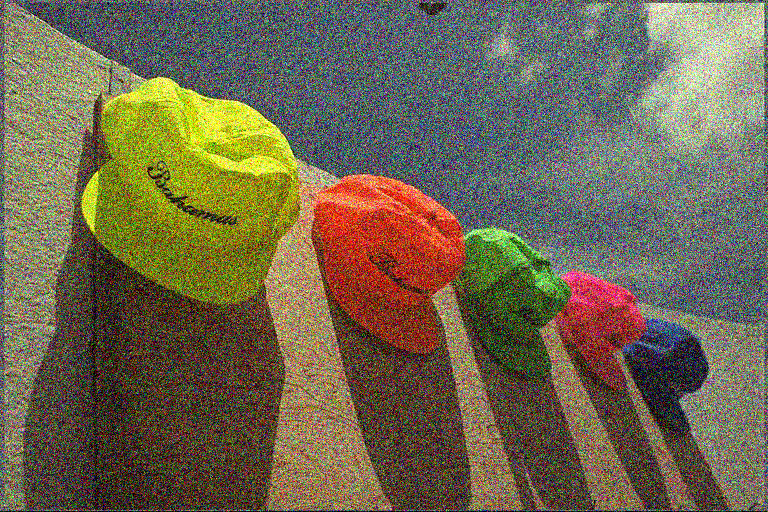}&
		\includegraphics[width=0.120\linewidth]{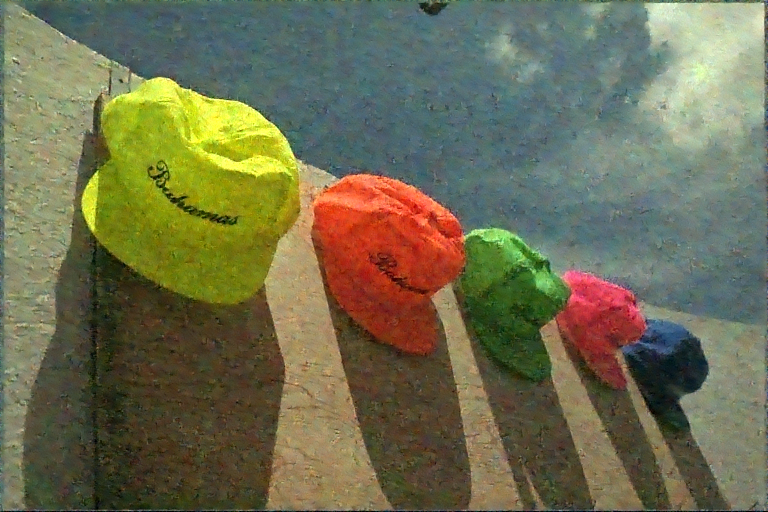}&
		\includegraphics[width=0.120\linewidth]{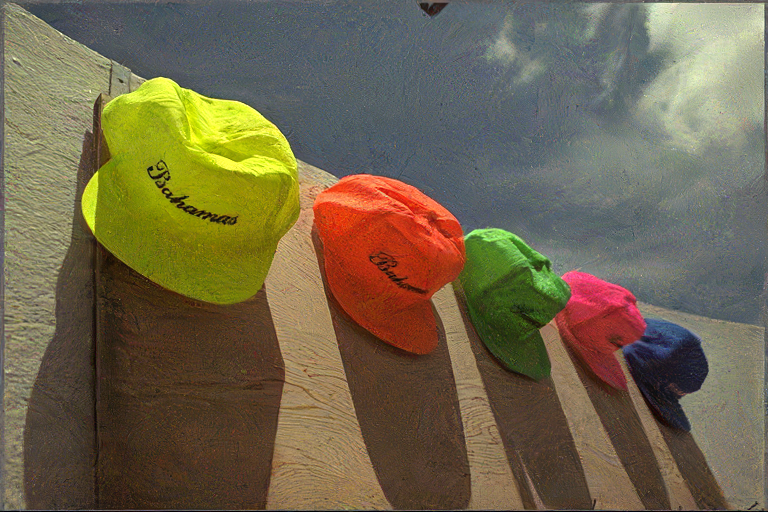}&
		\includegraphics[width=0.120\linewidth]{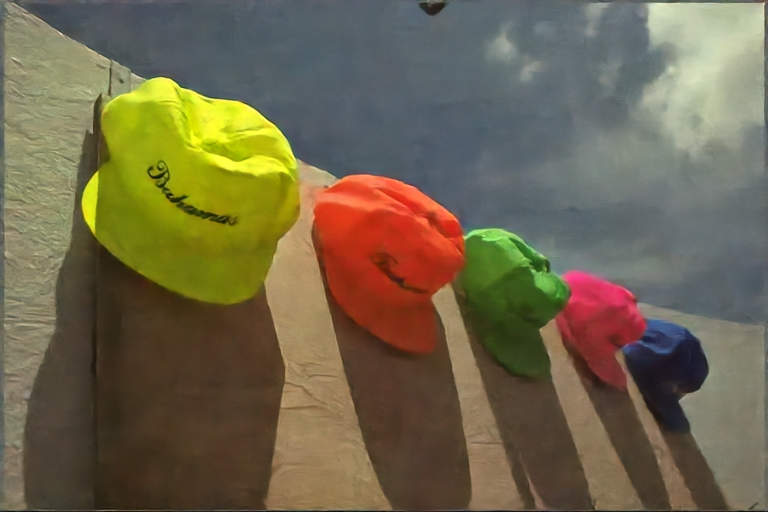}&
		\includegraphics[width=0.120\linewidth]{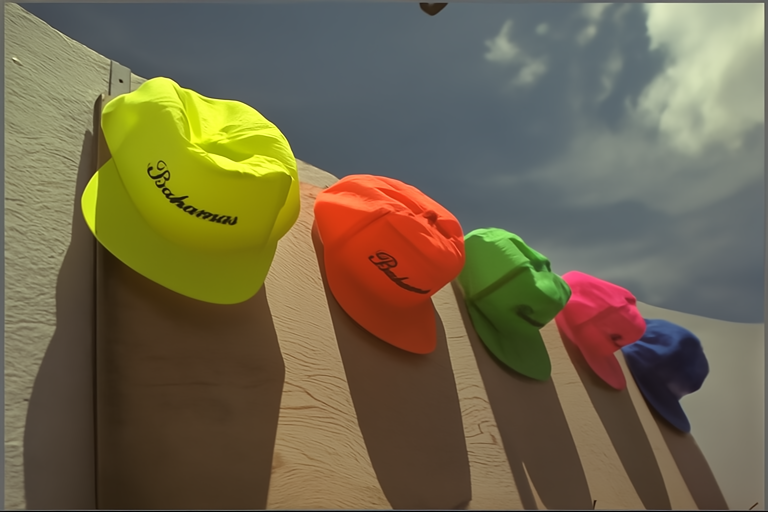}&
		\includegraphics[width=0.120\linewidth]{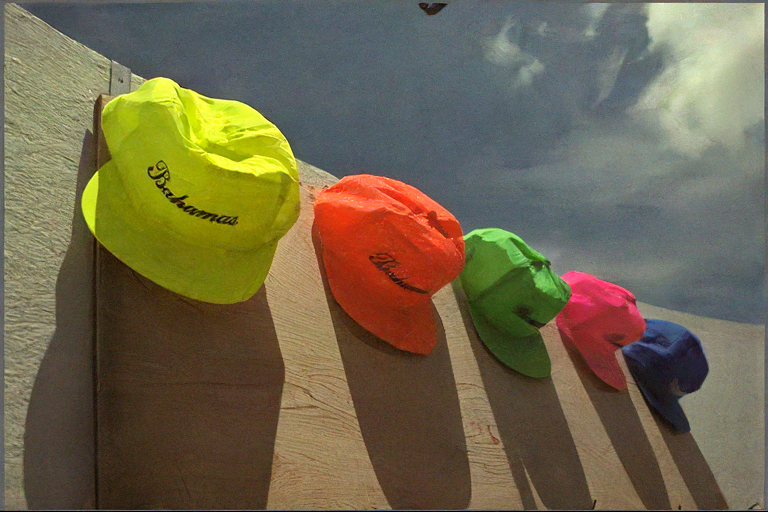}&
		\includegraphics[width=0.120\linewidth]{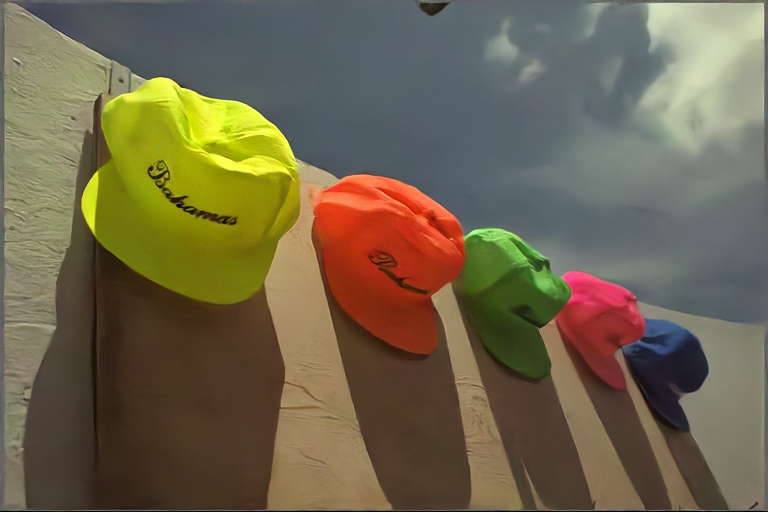}&
		\includegraphics[width=0.120\linewidth]{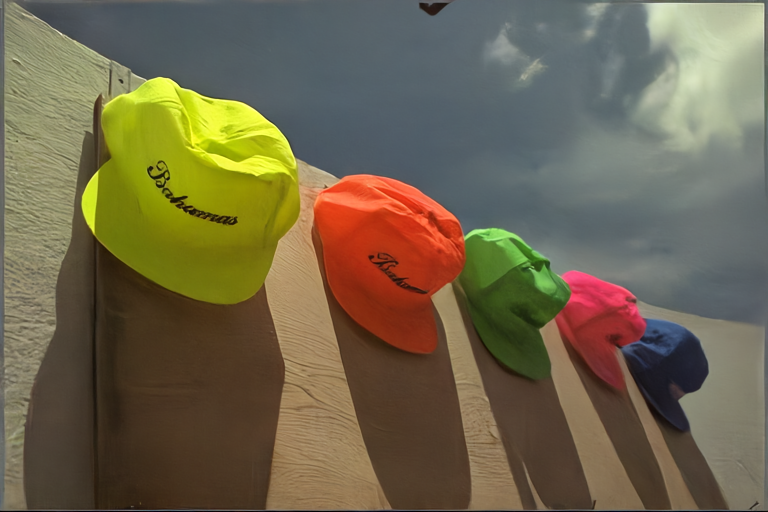}\\
14.75/0.271&26.82/0.836&28.88/0.846&30.70/0.946&32.04/0.960&30.18/0.903&31.89/0.959&32.41/0.963\\
		\includegraphics[width=0.120\linewidth]{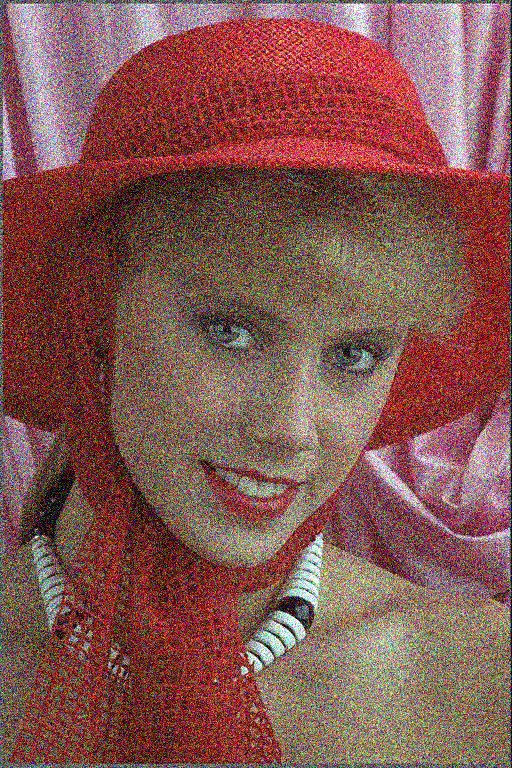}&
		\includegraphics[width=0.120\linewidth]{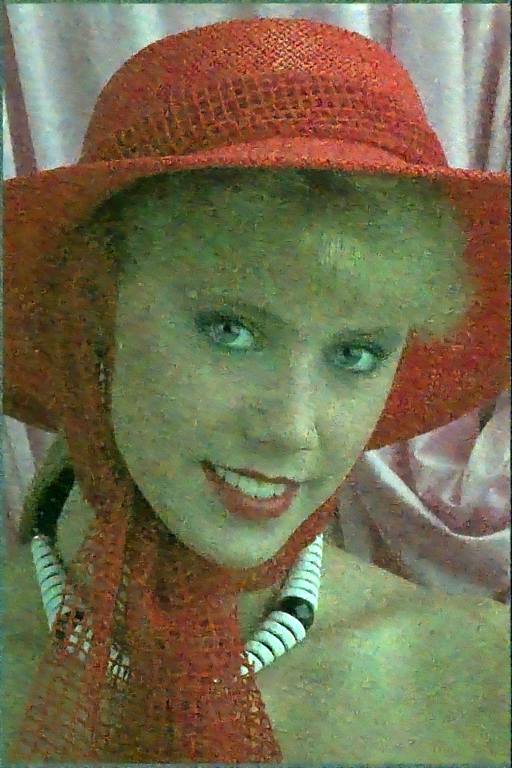}&
		\includegraphics[width=0.120\linewidth]{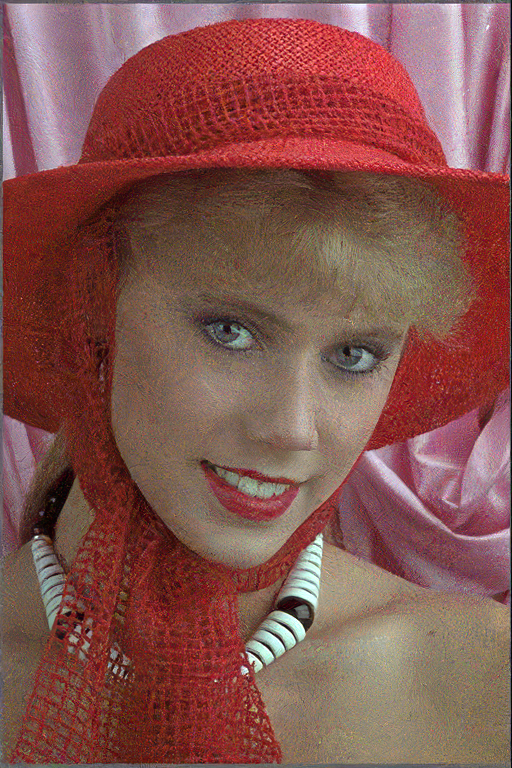}&
		\includegraphics[width=0.120\linewidth]{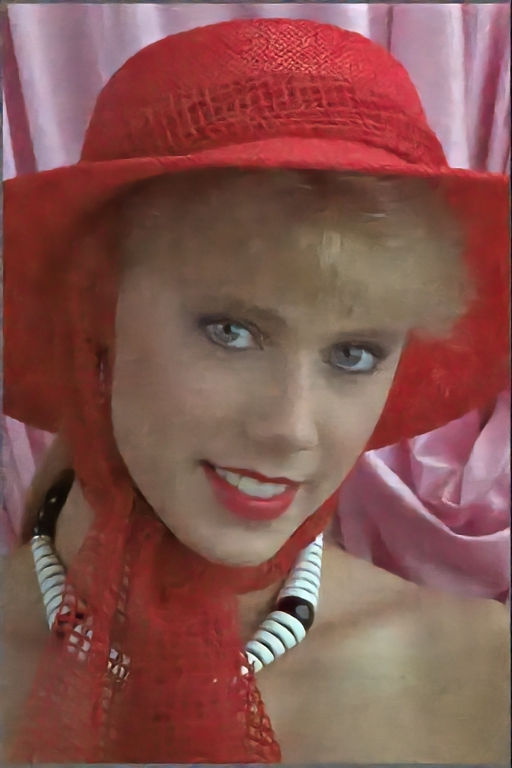}&
		\includegraphics[width=0.120\linewidth]{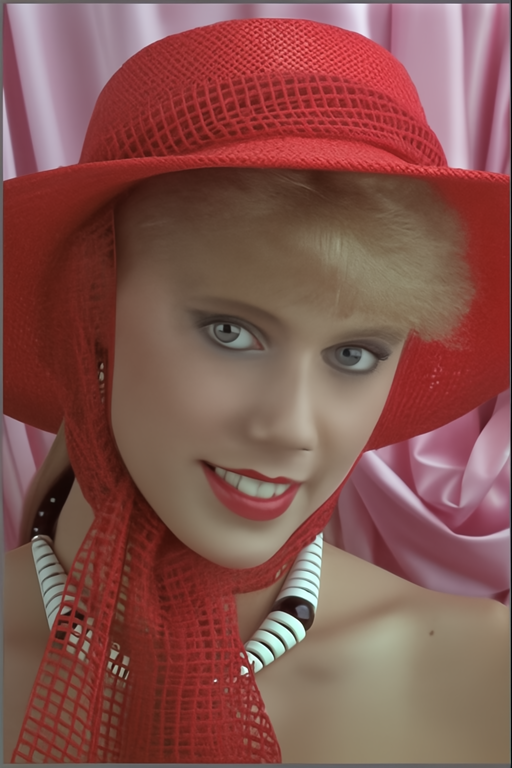}&
		\includegraphics[width=0.120\linewidth]{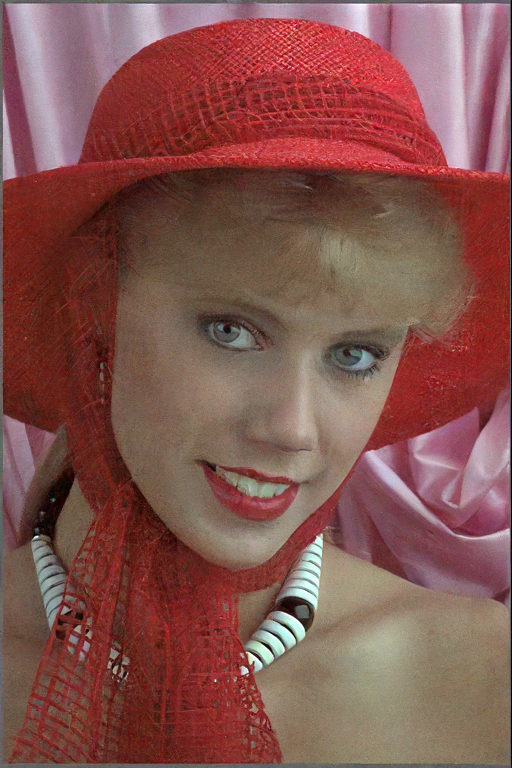}&
		\includegraphics[width=0.120\linewidth]{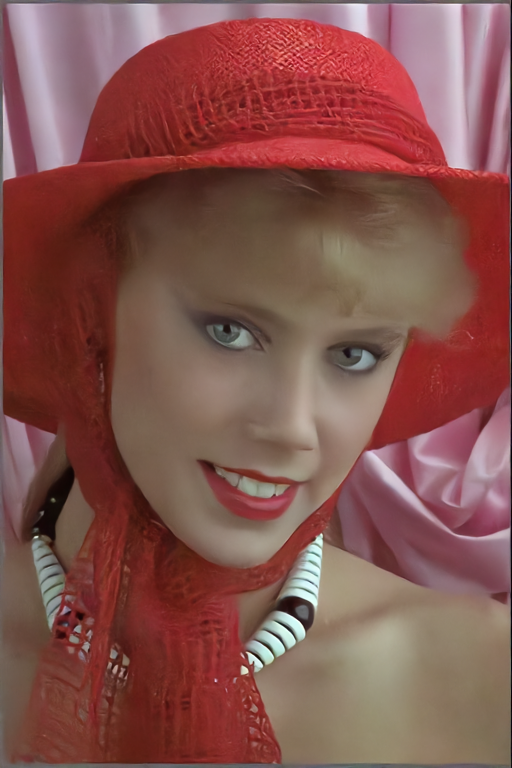}&
		\includegraphics[width=0.120\linewidth]{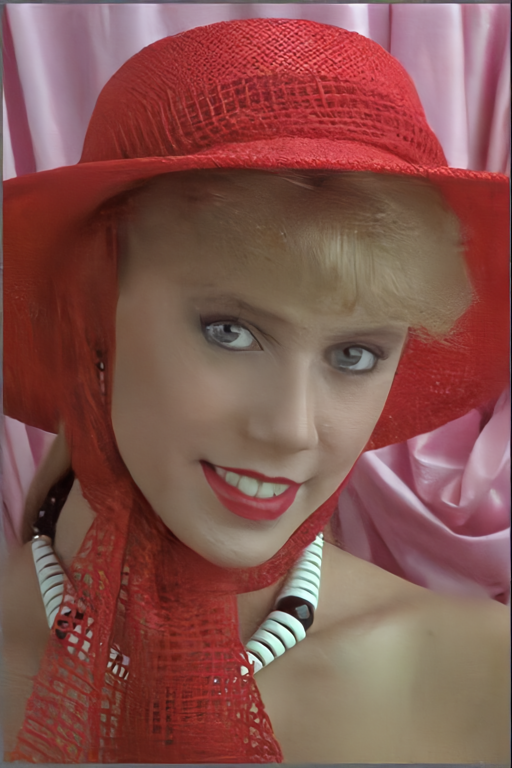}\\
14.77/0.364&21.97/0.608&25.29/0.840&29.26/0.955&30.22/0.964&28.73/0.938&30.18/0.965&30.61/0.967\\
		Noisy&NOT&OTUR&MPRNet&Restormer&IR-SDE&PromptIR&RCOT\\
	\end{tabular}
	\caption{Visual comparison of denoising on noise level of $\sigma=50$. Our RCOT produces results with better textural details.}
	\label{anoisy}
\end{figure*} 

\begin{figure*}[!h]
	\setlength\tabcolsep{1pt}
	\renewcommand{\arraystretch}{0.5} 
	\centering
	\begin{tabular}{cccccccc}
		\includegraphics[width=0.120\linewidth]{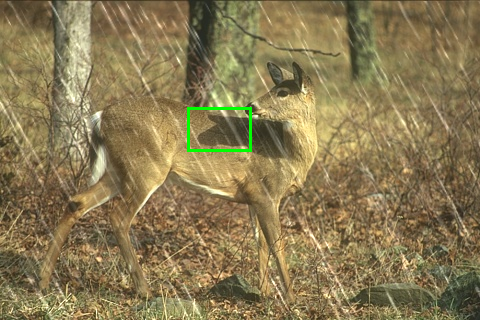}&
		\includegraphics[width=0.120\linewidth]{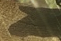}&
		\includegraphics[width=0.120\linewidth]{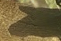}&
		\includegraphics[width=0.120\linewidth]{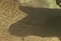}&
		\includegraphics[width=0.120\linewidth]{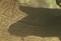}&
		\includegraphics[width=0.120\linewidth]{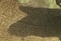}&
		\includegraphics[width=0.120\linewidth]{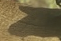}&
		\includegraphics[width=0.120\linewidth]{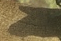}\\
            25.56/0.949&28.56/0.897&29.70/0.968&32.91/0.987&33.19/0.988&33.71/0.989&33.22/0.988&\textbf{34.20/0.990}\\
		\includegraphics[width=0.120\linewidth]{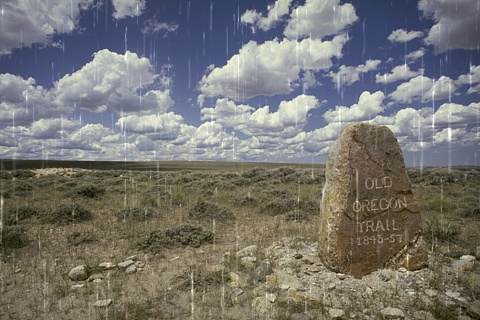}&
		\includegraphics[width=0.120\linewidth]{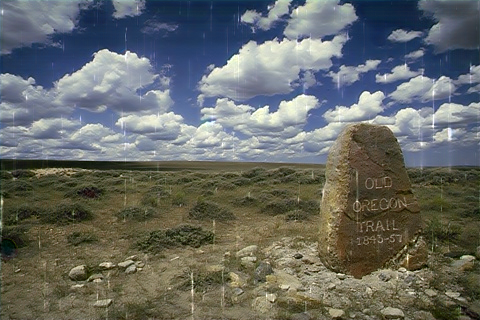}&
		\includegraphics[width=0.120\linewidth]{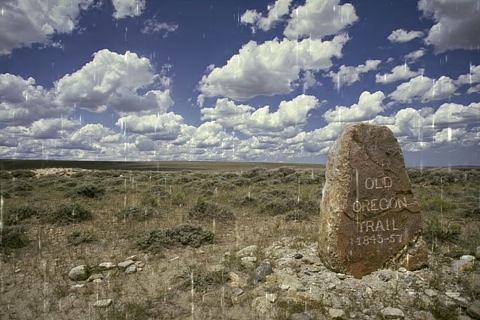}&
		\includegraphics[width=0.120\linewidth]{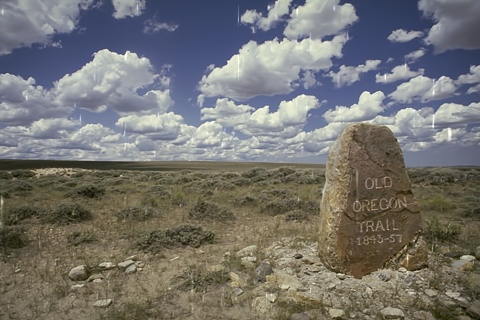}&
		\includegraphics[width=0.120\linewidth]{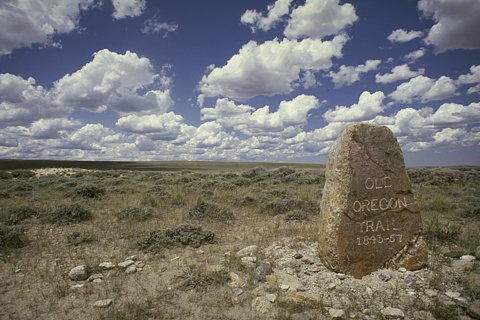}&
		\includegraphics[width=0.120\linewidth]{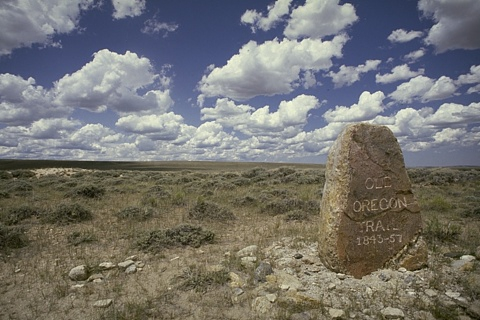}&
		\includegraphics[width=0.120\linewidth]{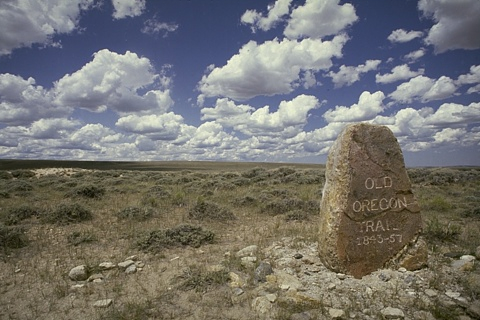}&
		\includegraphics[width=0.120\linewidth]{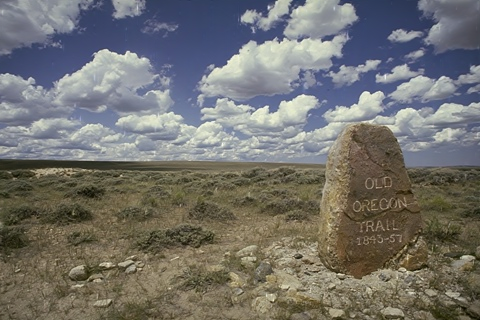}\\
		31.23/0.967&27.18/0.939&31.71/0.971&34.96/0.981&35.20/0.986&36.12/0.988&35.56/0.985&37.12/0.988\\
		Rainy&NOT&OTUR&MPRNet&Restormer&IR-SDE&PromptIR&RCOT\\
	\end{tabular}
	\caption{Visual comparison of deraining.  Our RCOT produces results with better structural details.}
	\label{arain}
\end{figure*} 
\begin{figure*}[!h]
	\setlength\tabcolsep{1pt}
	\renewcommand{\arraystretch}{0.5} 
	\centering
	\begin{tabular}{ccccccc}
            \includegraphics[width=0.135\linewidth]{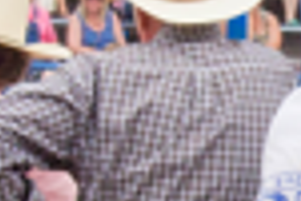}&
			\includegraphics[width=0.135\linewidth]{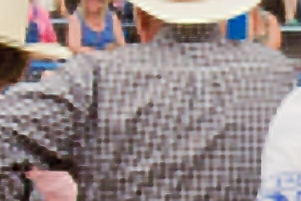}&
			\includegraphics[width=0.135\linewidth]{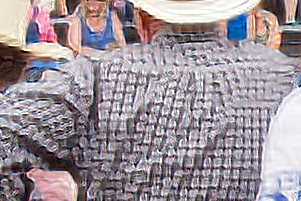}&			\includegraphics[width=0.135\linewidth]{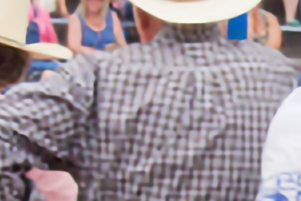}&
			\includegraphics[width=0.135\linewidth]{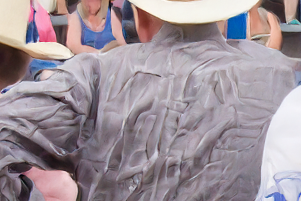}&
			\includegraphics[width=0.135\linewidth]{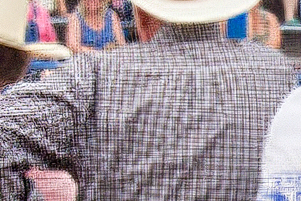}&
			\includegraphics[width=0.135\linewidth]{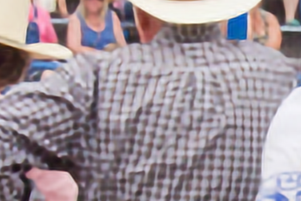}\\
			20.20/0.543&19.29/0.550&19.08/0.493&19.86/0.588&19.57/0.538&19.80/0.570&20.75/0.621\\
		\includegraphics[width=0.135\linewidth]{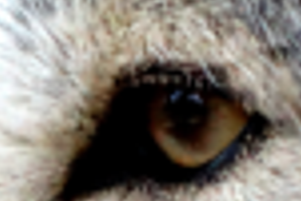}&
		\includegraphics[width=0.135\linewidth]{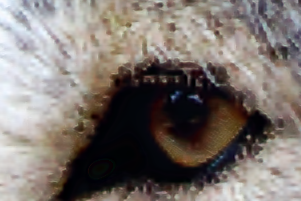}&
		\includegraphics[width=0.135\linewidth]{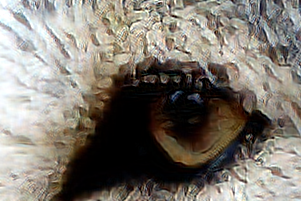}&
		\includegraphics[width=0.135\linewidth]{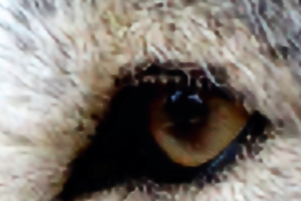}&
		\includegraphics[width=0.135\linewidth]{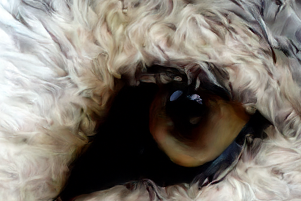}&
		\includegraphics[width=0.135\linewidth]{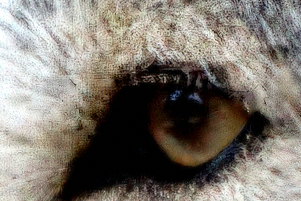}&
		\includegraphics[width=0.135\linewidth]{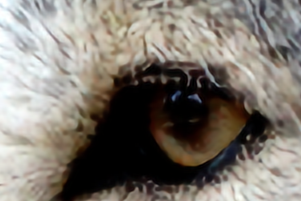}\\
		24.86/0.743&23.12/0.663&23.06/0.595&24.99/0.749&23.96/0.711&23.82/0.727&25.62/0.793\\
		LR& NOT&OTUR&Restormer&IDM&IR-SDE&RCOT\\		
	\end{tabular}
	\caption{Visual comparison of 4x SR on the DIV2K \cite{agustsson2017ntire} dataset. RCOT produces the sharpest HR images with more realistic structures.}
	\label{asr}
\end{figure*}
\newpage
\section{Visual examples of Transport Residuals}
\label{res}
This section exhibits some visual examples of the transport residual $\hat r(T_\theta)$. We can observe that different degradations have unique degradation-specific transport residuals, which coincides with our motivation to utilize the residual embedding to encode the degradation-specific knowledge (e.g., degradation type and level).
\begin{figure}[!h]
	\centering
	\includegraphics[width=0.8\linewidth]{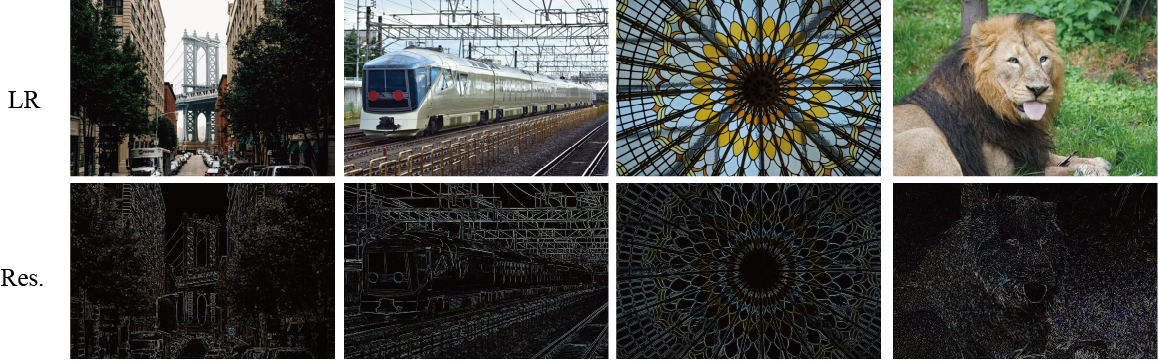}
	\caption{ Visual examples of the estimated transport residuals $\hat r(T_\theta)$ for \textbf{SR} task.  }
\end{figure}
\begin{figure}[!h]
	\centering
	\includegraphics[width=0.8\linewidth]{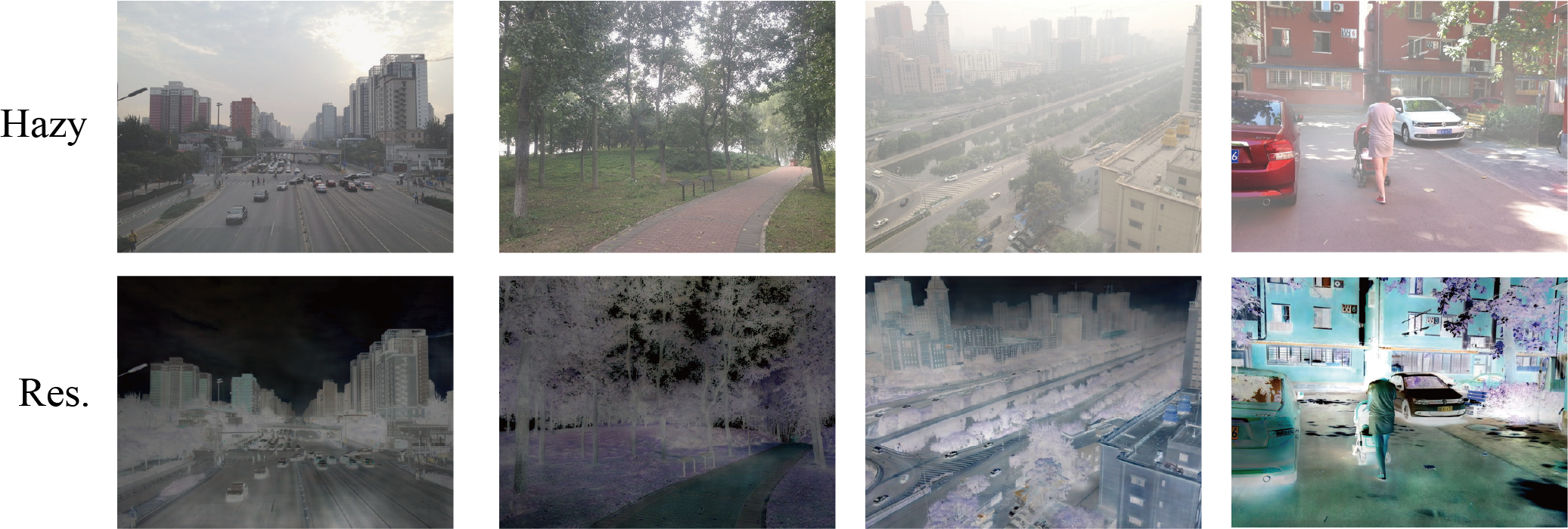}
	\caption{ Visual examples of the estimated transport residuals $\hat r(T_\theta)$ for \textbf{dehazing} task.}
\end{figure}

\begin{figure}[!h]
	\centering
	\includegraphics[width=0.8\linewidth]{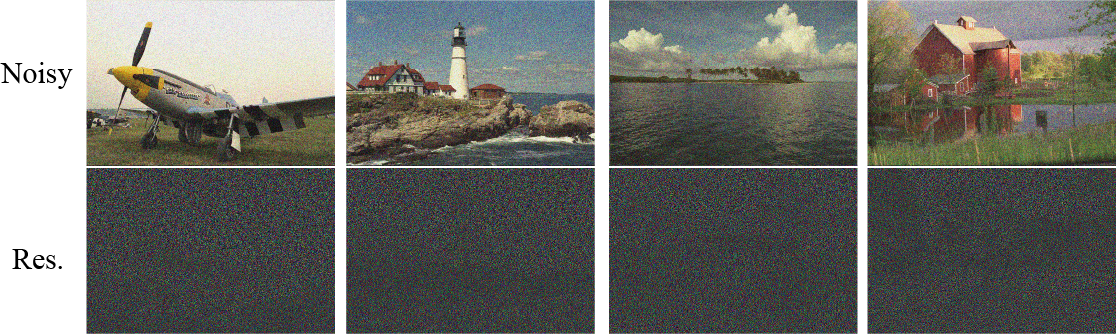}
	\caption{ Visual examples of the estimated transport residuals $\hat r(T_\theta)$ for \textbf{denoising} task.}
\end{figure}

\begin{figure}[!h]
	\centering
	\includegraphics[width=0.8\linewidth]{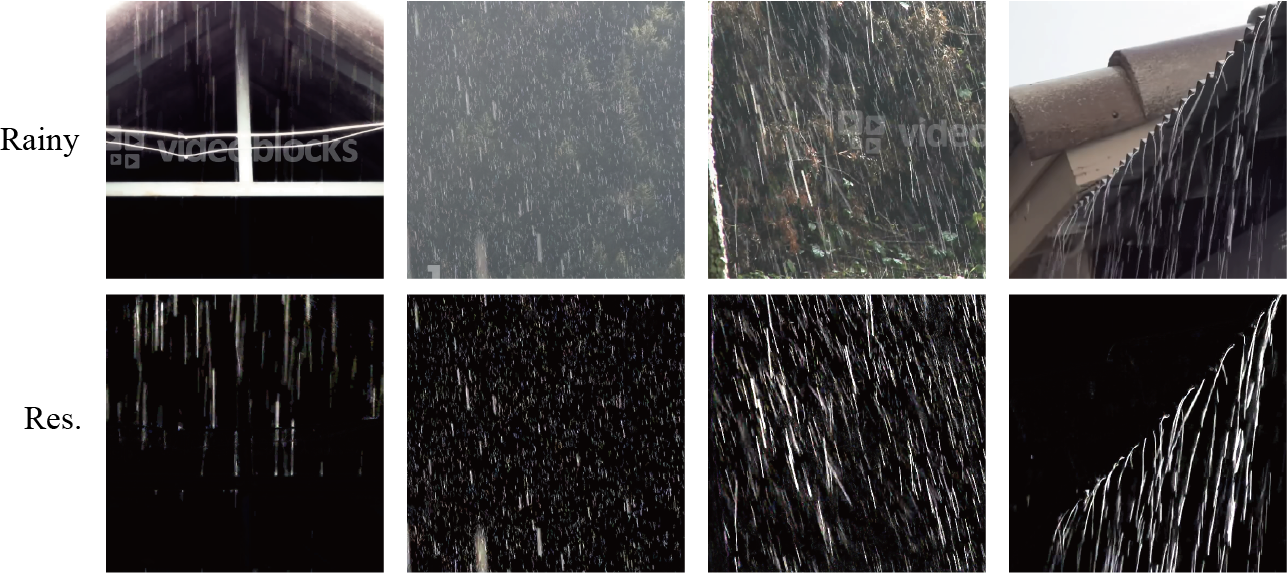}
	\caption{ Visual examples of the estimated transport residuals $\hat r(T_\theta)$  for \textbf{draining} task.}
\end{figure}

\end{document}